\documentclass[nohyperref]{article}

\usepackage{microtype}
\usepackage{graphicx}
\usepackage{subcaption}     
\usepackage{booktabs} 

\usepackage{hyperref}

\usepackage[accepted]{icml2022}

\usepackage{amsmath}
\usepackage{amssymb}
\usepackage{mathtools}
\usepackage{amsthm}

\usepackage[capitalize,noabbrev]{cleveref}

\usepackage{colortbl}
\usepackage{longtable}
\usepackage{tikz}
\usetikzlibrary{positioning}
\usetikzlibrary{shapes, shapes.misc, backgrounds}
\usetikzlibrary{calc}

\input{macros_math.tex}
\input{macros.tex}

\theoremstyle{plain}
\newtheorem{theorem}{Theorem}[section]

\newtheorem{lemma}[theorem]{Lemma}
\newtheorem{corollary}[theorem]{Corollary}
\theoremstyle{definition}
\newtheorem{assumption}{Assumption}[section]
\newtheorem{definition}{Definition}[section]
\newtheorem{illustration}{Illustration}[section]
\theoremstyle{remark}
\newtheorem{remark}{Remark}[section]

\usepackage[disable,textsize=tiny]{todonotes}

\newcommand{\commentrg}[2][noinline]{{\todo[author=Rom, color=blue!40, fancyline, #1]{#2}}}
\newcommand{\commentcsg}[2][noinline]{{\todo[author=Cam, color=green!40, fancyline, #1]{#2}}}

\icmltitlerunning{UniRank: Unimodal Bandit Algorithm for Online Ranking}

\begin{document}

\twocolumn[
\icmltitle{UniRank: Unimodal Bandit Algorithm for Online Ranking}



\icmlsetsymbol{equal}{*}

\begin{icmlauthorlist}
\icmlauthor{Camille-Sovanneary Gauthier}{equal,Lv,irisa}
\icmlauthor{Romaric Gaudel}{equal,ecc}
\icmlauthor{Elisa Fromont}{univR,IUF,irisa}
\end{icmlauthorlist}

\icmlaffiliation{univR}{Univ. Rennes 1, F-35000 Rennes, France}
\icmlaffiliation{ecc}{Univ Rennes, Ensai, CNRS, CREST - UMR 9194, F-35000 Rennes, France}
\icmlaffiliation{Lv}{Louis Vuitton, F-75001 Paris, France}
\icmlaffiliation{irisa}{IRISA UMR 6074 / INRIA rba, F-35000 Rennes, France}
\icmlaffiliation{IUF}{ Institut Universitaire de France, M.E.S.R.I., F-75231 Paris}
\icmlcorrespondingauthor{Camille-Sovanneary Gauthier}{camille-sovanneary.gauthier@louisvuitton.com}

\icmlkeywords{Online Recommender System, Online Learning to Rank, Multiple-Play Bandit, Unimodal Bandit}

\vskip 0.3in
]



\printAffiliationsAndNotice{\icmlEqualContribution} 


\begin{abstract}
We tackle, in the multiple-play bandit setting, the online ranking problem of assigning $L$ items to $K$ predefined positions on a web page in order to maximize the number of user clicks. We propose a generic algorithm, UniRank, that tackles state-of-the-art click models. The regret bound of this algorithm is a direct consequence of the unimodality-like property of the bandit setting with respect to
a graph where nodes are ordered sets of indistinguishable items.
The main contribution of UniRank is its $\OO\left(L/\Delta \log T\right)$ regret for $T$ consecutive assignments, where $\Delta$ relates to the reward-gap between two items.
This regret bound is based on the usually implicit condition that two items may not have the same attractiveness.
Experiments against state-of-the-art learning algorithms specialized or not for different click models, show that our method has better regret performance than other generic algorithms on real life and synthetic datasets.
\end{abstract}

\section{Introduction} 

We consider \emph{Online Recommendation Systems} (ORS) which choose $K$ relevant items among $L$ potential ones ($L \geq K$), such as songs, ads or movies to be displayed on a website. The user feedbacks, such as listening time, clicks, rates, etc., reflecting the user's appreciation with respect to each displayed item, are collected after each recommendation. As these feedbacks are only available for the items which were actually presented to the user, this setting corresponds to an instance of the \emph{multi-armed bandit problem} with \emph{semi-bandit feedback} \citep{Gai2012, Chen2013}.
Besides, some  displayed items are not looked at and lead to a negative feedback while they would be appreciated by the user.
It raises a specific challenge related to ranking: the attention toward a displayed item is impacted by its position. Numerous approaches have been proposed to handle this partial attention \citep{Radlinski2008, Combes2015, Lagree2016} referred to as \emph{multiple-play bandit} or \emph{online learning to rank}. Several models of partial attention, a.k.a. click models, are considered in the state of the art \citep{Richardson2007,Craswell2008} and have been transposed to the bandit framework \citep{Kveton2015a,Komiyama2017}.
In current paper, in the same line as
\citep{Zoghi2017, TopRank}
we propose an algorithm which handles multiple state-of-the-art click models. 

The main contribution of our work is a new bandit algorithm, \ouralgo{}, dedicated to a generic online learning to rank setting.
\ouralgo{} takes inspiration from unimodal bandit algorithms \citep{Combes2014, GRAB}:
we implicitly consider a graph $\Gc$ on the partitions of the item-set such that the considered bandit setting is unimodal w.r.t. $\Gc$, and  \ouralgo{} chooses each recommendation in the $\Gc$-neighborhood of an elicited partition. Thanks to this restricted exploration, \ouralgo{} is the first algorithm dedicated to a generic setting with a $O(L/\Delta\log T)$  regret upper-bound, while previous state-of-the-art algorithms were suffering
a $O(LK/\Delta\log T)$ regret.
Note that this $O(L/\Delta\log T)$ upper-bound requires all items' attractiveness to be different, which is a usual assumption satisfied by real world applications.
Otherwise, \ouralgo{} recovers the $O(LK/\Delta\log T)$ bound.
From an application point of view, \ouralgo{} has several interesting features: it handles multiple state-of-the-art click models altogether; it is simple to implement and efficient in terms of computation time; it does not require the knowledge of the time horizon $T$; and it exhibits a smaller empirical regret than other generic algorithms by leaning on the \emph{different attractiveness} property when this property is satisfied.

\begin{table*}[t!]
 \caption{Required click model and upper-bound on cumulative regret for $T$ consecutive recommendations for some well-known recommender algorithms that chose $K$ items among $L$. The exact definition of $\Delta$ is specific to each algorithm. The symbol ${}^*$ means that Assumption 3.1$^*$, defined in \cref{sec:model}, is satisfied. $\kappav$ denotes the vector of observation-probabilities of PBM, and $\gamma$ is the degree of the graph explored by the unimodal bandit algorithm.}
    \label{tab:compare_reg}
    \vskip 0.15in  
    \centering
    \begin{tabular}{lll}
    \toprule
        \textbf{Algorithm} &\textbf{Click model}& \textbf{Regret}\\
        \midrule
     
        \ouralgo{} (our algorithm)  &CM$^*$ & $\OO\left((L-K)/\Delta\log T\right)$\\
                                    &PBM$^*$, \dots & $\OO\left(L/\Delta\log T\right)$\\
                                    
                                    &PBM, CM, \dots & $\OO\left(LK/\Delta\log T\right)$\\

        TopRank \citep{TopRank} &PBM, CM, \dots & $\OO\left(LK/\Delta\log T\right)$ \\

        \midrule

        CascadeKL-UCB \citep{Kveton2015a} &CM & $\OO\left((L-K)/\Delta\log T\right)$ \\

        GRAB \citep{GRAB} &PBM$^*$& $\OO\left(L/\Delta\log T\right)$\\
        PB-MHB \citep{PBMHB} &PBM ($\kappav_1=0$)& unknown\\
        PBM-PIE \citep{Lagree2016} &PBM ($\kappav$ known)& $\OO\left((L-K)/\Delta\log T\right)$ \\
        
        SAM \citep{Sentenac2021} &Matching$^*$& $\OO\left(L\log L/\Delta\log T\right)$ \\

        OSUB \citep{Combes2014} &Unimodal& $\OO\left(\gamma/\Delta\log T\right)$\\
\bottomrule
\end{tabular}
\vskip -0.1in  
 \end{table*}

As an indirect contribution, \ouralgo{} demonstrates that unimodality is a key tool to analyze the intrinsic complexity  of some combinatorial semi-bandit problems. We also demonstrate the flexibility of unimodal bandit algorithms and of the  proof of their regret upper-bound.
In particular, we extend \citep{Combes2014}'s analysis to a graph which is unimodal in a weaker sense: (i) \ouralgo{} takes its decisions given an optimistic index which is not based on the expected reward but on the probability for an item to be more attractive than another one and (ii) some sub-optimal nodes in the graph have no better node in their neighborhood.

The paper is organized as follows: \cref{sec:related} presents the related work and \cref{sec:olr} defines our target setting. We then introduce \ouralgo{} in \cref{sec:unirank}, and theoretical guarantees and empirical performance are presented respectively in \cref{sec:unirank-th} and \cref{sec:exp}. We conclude in \cref{sec:conclusion}.

\section{Related Work}\label{sec:related}

\Cref{tab:compare_reg} shows a comparison of the assumptions and the regret upper-bounds of the most related algorithms.

Several bandit algorithms are designed to handle the online learning to rank setting while the user follows one of the currently defined click models, namely the \emph{position based model} (PBM) \citep{Komiyama2015, Lagree2016, Komiyama2017, PBMHB, GRAB} or the \emph{cascading model} (CM) \citep{Kveton2015a,Kveton2015b,Combes2015,Zong2016,Katariya2016,Li2016,Cheung2019}. To the best of our knowledge, only the algorithms BatchRank \citep{Zoghi2017}, TopRank \citep{TopRank}, and BubbleRank \citep{BubbleRank} handle users following a general model covering both behaviors. These three algorithms exhibit a regret upper-bound for $T$ consecutive recommendations of at least $\OO(LK/\Delta\log T)$, where $\Delta$ depends on the \emph{attraction probability} $\thetav$ of items.

One ingredient of TopRank and BubbleRank is a statistic to compare two items independently of the position at which they are displayed. The algorithm we propose also makes use of this statistic.
However, we define an exploration strategy which does not require the knowledge of the time-horizon $T$ and which induces a $\OO(L/\Delta\log T)$ regret upper-bound
 when items have strictly different attractiveness.

\ouralgo{} also builds upon an extension of the \emph{unimodal bandit} setting \citep{Combes2014,GRAB}. This setting assumes the knowledge of a graph $\Gc$ on the set $\Ac$ of bandit arms, such that the expected reward $\mu_\av$ associated to each arm $\av$ satisfies the following assumption:

\begin{assumption}[Unimodality\footnote{Usually, the unimodality is defined as the existence of a strictly increasing path from any sub-optimal arm to $\av^*$. \Cref{asp:unimodality} is equivalent and we use it in this paper as it directly relates to the shape of the theoretical analysis.}]\label{asp:unimodality}
There exists a unique arm $\av^*\in\Ac$ with highest expected reward, and
for any arm $\av\in\Ac$,
%
    either (i) $\av=\av^*$,
    or (ii) there exists $\av^+$ in the neighborhood $\Nc_\Gc\left(\av\right)$ of $\av$ given $\Gc$ such that $\mu_{\av^+} > \mu_\av$.
\end{assumption}
The unimodal bandit algorithms are aware of $\Gc$, but ignore the weak order induced by the edges of $\Gc$. However, they rely on $\Gc$ to efficiently browse the arms up to the best one. Typically, the algorithm OSUB \citep{Combes2014} selects at each iteration $t$, an arm $\av(t)$ in the neighborhood $\Nc_\Gc\left(\tilde\av(t)\right)$ given $\Gc$ of the current best arm $\tilde\av(t)$ (a.k.a. the \emph{leader}). By restricting the exploration to this neighborhood, the regret suffered by OSUB scales in $\OO(\gamma/\Delta\log T)$, where $\gamma$ is the maximum degree of $\Gc$, to be compared with $\OO(|\Ac|/\Delta\log T)$ if the arms were independent. OSUB is designed for the standard bandit setting and makes use of estimators of the expected reward of arms to select the leader and chose the arm to play. In comparison, \ouralgo{} extends OSUB's idea to the semi-bandit setting, relies on a new variant of the unimodality property (see \cref{lem:unimodality}), selects the leader and the recommended arm based on other statistics, and does not require the ‘forced exploitation’  step which consists in recommending the leader each $\gamma$-th iteration.

Finally, \citep{GRAB} also builds upon the unimodality framework to solve a learning to rank problem in the bandit setting. However, the corresponding algorithm (GRAB) is dedicated to the PBM click model. 
In this model, there is a natural statistic to look at to measure the quality of an item and a position: the probability of click when presenting item $i$ in position $k$. This statistic is independent of the items at other positions. Within a CM model, such a statistic does not exist. Instead, we refer to a statistic related to the relative attractiveness of items $i$ and $j$ (which we denote $\hat{s}_{i,j}$).
Secondly, in the PBM model, only weak assumptions are needed to guarantee a unique optimal recommendation, which is required to get the unimodality property. With a CM model, any recommendation including the $K$ best items leads to the optimal reward. To recover the unicity of the best arm, our algorithm is not targeting the best reward, but the best ranking of items (which implies the best reward).
However, when facing PBM model, our algorithm requires an assumption which is omitted by GRAB: the position are indexed from the most look-at position to the least one.

\section{Learning to Rank in a Semi-Bandit Setting}\label{sec:olr}
We consider the following \emph{online learning to rank (OLR) problem with clicks feedback}. 
For any integer $n$, let $[n]$ denote the set $\{1,\dots, n\}$. A recommendation $\av=(a_{1},\dots, a_{K})$ is a permutation  of $K$ distinct items among $L$, where $a_{k}$ is the item displayed at position $k$ and $\av([K]) \defeq \left\{a_k : k \in [K]\right\}$ is the set of all displayed items. We denote $\perm_K^L$ the set of such permutations. Throughout the paper, we will use the terms \emph{permutation} and \emph{recommendation} interchangeably to denote an element of $\perm_K^L$.

An instance of our OLR problem is a tuple $(L, K, \rho)$, where $L$ is the number of available items, $K\leqslant L$ is the number of positions to display the items, and $\rho$ is a function from $\perm_K^L\times[K]$ to $(0,1]$ such that 
for any recommendation $\av$ and position $k$, $\rho(\av,k)$ is the probability for a user to click on the item displayed at position $k$ when recommending $\av$.

A recommendation algorithm is only aware of $L$ and $K$ and has to deliver $T$ consecutive recommendations. At each iteration $t\in [T]$, the algorithm recommends a permutation $\av(t)$ and observes the values $c_{a_1(t)}(t), \dots, c_{a_K(t)}(t)$, where for any position $k$, $c_{a_k(t)}(t)$ equals $1$ whenever the user clicks on the item $a_k(t)$, and $0$ otherwise.
To keep notations simple, we also define $c_i(t)=0$ for each undisplayed item \scalebox{0.9}{$i\in[L]\setminus\av(t)([K])$}.
Note that the recommendation at time $t$ is only based on previous recommendations and observations.

While the individual clicks are observed, the reward of the algorithm is their sum $r(t)\defeq\sum_{k=1}^Kc_{a_k(t)}(t)=\sum_{i=1}^Lc_i(t)$. Let $\mu_\av$ denote the expectation of $r(t)$ when the recommendation is $\av(t)=\av$, and $\mu^*\defeq\max_{\av\in \perm_K^L}\mu_\av$ the highest expected reward. The aim of the algorithm is to minimize the \emph{cumulative regret}

\begin{equation}\label{eq:regret}
R(T) = \EE\left[T\mu^* - \sum_{t = 1}^T \mu_{\av(t)}\right],
\end{equation}
where the expectation is taken w.r.t. the recommendations from the algorithm and the clicks.

\forlater{revoir formule de regret. pour être plus clair sur le fait qu'ici on parle du pseudo-regret (comme dans preuve et algo).}

\begin{illustration}[Click model PBM]
With the click model PBM, at each iteration $t$, the user looks at the position $k$ with probability $\kappa_k$, independently of the displayed items $\av(t)$. Moreover, whenever she observes the position $k$, she clicks on the corresponding item $a_k(t)$ with probability $\theta_{a_k(t)}$, independently of her other actions. Overall, the clicks $c_{a_k(t)}(t)$ are  independent and $\rho(\av(t),k)=\EE\left[c_{a_k(t)}(t)\right] = \kappa_k\theta_{a_k(t)}$.
Therefore, the optimal recommendation consists in displaying the item $i$ with the $\ell$-th highest value $\theta_i$ at the position $k$ with the $\ell$-th highest value $\theta_k$. Hence, if $\theta_1>\theta_2>\dots>\theta_K>\max_{K<k\leqslant L}\theta_k$ and $\kappa_1>\kappa_2>\dots>\kappa_K$, $\mu^*=\sum_{k=1}^K\kappa_k\theta_k$.
\end{illustration}

\subsection{Modeling Assumption}\label{sec:model}
Up to now, an OLR problem assumes two main properties: (i) a click at a position is a random variable only conditioned by the recommendation and the position, and (ii) the expectation of the corresponding distribution is fixed. We now introduce the three assumptions required by \ouralgo{}, which are fulfilled by PBM and CM click models.

\forlater{changer articulation du discours:\\
pref att. est def.\\
rem: existe avec PBM, CM et DCM.\\
hypothes 1 : existe et relier à optimaliter\\
hypothes 1* : hypothèse 1 avec strict\\
}

We first assume an order on items. 
Note that the existence of an order on item is a weak assumption by itself (we may chose any random order). The strength of this assumption derives from Assumptions \ref{asp:acyclicity} and \ref{asp:identifiability} which enforces this order to relate with expected reward.

\begin{assumption}[Strict weak order]\label{asp:strict_weak_order}
There exists a \emph{preferential attachment} function $g:[L] \to \RR$ on items, and for any pair of items $(i,j)$,
\begin{itemize}
    \item if $g(i) > g(j)$, item $i$ is said \emph{more attractive} than item $j$, which we denote $i\succ j$;
    \item if $g(i) = g(j)$, item $i$ is said \emph{equivalent} to item $j$, which we denote $i\sim j$.
\end{itemize}
\end{assumption}

\begin{illustration}[Strict weak order with PBM]
With PBM, a typical choice for the function $g$ is $g:i\mapsto\theta_i$.
\end{illustration}

\Cref{asp:strict_weak_order} ensures the existence of a strict weak order $\succ$ on items: the items may be ranked by attractiveness, some items being equivalent. A typical example with $L=4$ would be $1 \succ 2 \sim 3 \succ 4$, meaning item 1 is more attractive than any other item, and items 2 and 3 are equivalent and more attractive than item 4. Such situation may also be represented with an ordered partition: $\left(\{1\}, \{2,3\}, \{4\}\right)$, where if the subset $E$ is listed before the subset $F$, then for any item $i\in E$ and any item $j \in F$, $i\succ j$. In the rest of the paper we will use either the preferential attachment function, or its associated strict weak order, or the corresponding ordered partition depending on the most appropriate representation.

The strongest results of the theoretical analysis require the slightly stronger assumption which ensures that the $K$ best items are uniquely defined. This assumption is equivalent to any of both hypothesis: (i) the order $\succ$ is total on the $K$ best items and the $K$-ith item is strictly more attractive than remaining $L-K$ items, and (ii) each of the $K$ first subsets of the ordered partition is composed of only one item.

\begin{assumption0s}[Strict total order on top-$K$ items]\label{asp:strict_total_order}
There exists a \emph{preferential attachment} function $g:[L] \to \RR$ and a permutation $\av\in\perm_K^L$ s.t. $g(a_1) > g(a_2) > \dots > g(a_K)$ and for any item $j\in[L]\setminus \av([K])$, $g(a_K) > g(j)$.
\end{assumption0s}

Our next assumption states that recommending the items according to the order $\succ$ associated to the preferential attachment leads to an optimal recommendation. 

\begin{definition}[Compatibility with a strict weak order]\label{def:compatibility}
Let $\succ$ be a strict weak order on items, and $\av$ be a recommendation. 
The recommendation $\av$ is \emph{compatible} with $\succ$ if
\begin{enumerate}
    \item for any position $k\in[K-1]$, either $a_k \succ a_{k+1}$ or $a_k \sim a_{k+1}$;
    \item for any item \scalebox{0.9}{$j\in[L]\setminus \av([K])$}, either $a_K \succ j$ or $a_K \sim j$.
\end{enumerate}
\end{definition}

\begin{assumption}[Optimal reward]\label{asp:acyclicity}
Any recommendation $\av$ compatible with $\succ$ is optimal, meaning $\mu_\av = \mu^*$.
\end{assumption}

\begin{illustration}[Optimal reward with PBM]
With PBM, if the positions are ranked by decreasing observation probabilities and $g(i)=\theta_i$, this assumption means that the recommendation placing the $k$-th most attractive item at the $k$-th most observed position is optimal, which indeed is true.
\end{illustration}

\Cref{asp:acyclicity} is of utmost importance for \ouralgo{} as it means that identifying a partition of the items coherent with $\succ$ is sufficient to ensure optimal recommendations.

\commentrg{version "variable aléatoire Bernoulli" $S_{i,j}(t) = \frac{1+c_i(t)-c_j(t)}{2}$ pour simplifier discours et notations par la suite ?}
Let us now consider the last assumption which regards the expectation of the random variable $c_i(t)-c_j(t)$.

\begin{definition}[Expected click difference]\label{def:statistic}
Let $i$ and $j$ be two items, and $\av$ a recommendation.
The \emph{probability of difference} and the \emph{expected click difference} between items $i$ and $j$ w.r.t. the recommendation $\av$ are respectively:
\begin{align*}
\tilde\delta_{i,j}(\av) &=
\PP_{\av(t)\sim \Uc(\{\av, (i,j)\circ\av\}) }\left[c_i(t) \neq c_j(t)\right] \text{ and}
\\
\tilde\Delta_{i,j}(\av) &=
\EE_{\av(t)\sim \Uc(\{\av, (i,j)\circ\av\}) }\left[c_i(t)-c_j(t) \mid c_i(t) \neq c_j(t)\right],
\end{align*}
where $(i,j)\circ\av$ is the permutation $\av$ such that items $i$ and $j$ have been swapped, and $\Uc(S)$ is the uniform distribution on the set $S$. If only $i$ (respectively $j$) belongs to $\av$, $(i,j)\circ\av$ is the permutation $\av$ where item $i$ is replaced by item $j$ (resp. $j$ by $i$). If neither $i$ nor $j$ belongs to $\av$, $(i,j)\circ\av$ is $\av$.
\end{definition}

\begin{assumption}[Order identifiability]\label{asp:identifiability}
The strict weak order $\succ$ on items is \emph{identifiable}, meaning that for any couple of items $(i, j)$ in $[L]^2$ s.t. $i\succ j$,
and for any recommendation $\av\in\perm_K^L$ s.t. at least one of both items is displayed,
$
\tilde\delta_{i,j}(\av) \neq 0
$
and
$\tilde\Delta_{i,j}(\av) > 0$ . 
\end{assumption}

\begin{illustration}[Expected click difference with PBM]
With the click model PBM, if the positions are ranked by decreasing observation probabilities, for any recommendation $\av$,
any position $k\in[K]$ and any position $\ell\in[L]\setminus\{k\}$,
denoting $i$ and $j$ the items at respective positions $k$ and $\ell$,  
$\tilde\delta_{i,j}(\av) = \frac{1}{2}\left(\theta_{i}+\theta_{j}\right)\left(\kappa_{k}+\kappa_{\ell}\right) - 2 \theta_i\theta_j\kappa_{k}\kappa_\ell$
and $\tilde\Delta_{i,j}(\av) =\frac{\theta_{i}-\theta_j}{\theta_i+\theta_j}d_{i,j}(\av)$, where $d_{i,j}(\av)>1$.
Therefore, if $g(i)=\theta_i$, \cref{asp:identifiability} is fulfilled.
\end{illustration}

The expected click difference reflects the fact that an item leads to more clicks than another independently of the position of both items (other items being unchanged).
Hence, \cref{asp:identifiability} points out that when an item is more attractive than another one, it has a higher probability to be clicked upon, all other things being equal. This assumption is natural and ensures that the order on items may be recovered from the expected click difference, which is observed.

Finally, the following lemma, proven in \cref{app:pbm_and_cm_are_unimodal}, states that both  CM and  PBM models fulfill our assumptions. 

\begin{lemma}\label{lem:pbm_and_cm_are_unimodal}
Let $(L, K, \rho)$ be an online learning to rank problem with users following CM or PBM model with positions ranked by decreasing observation probabilities. Then
 Assumptions~\ref{asp:strict_weak_order}, \ref{asp:acyclicity}, and \ref{asp:identifiability} are fulfilled. Furthermore, Assumption  3.1$^*$ is fulfilled if, for any top-$K$ item $i$ and any item $j$ in $[L]\setminus\{i\}$, either $i\succ j$ or $j \succ i$. 
\end{lemma}

\begin{algorithm}[t!]
\caption{\ouralgo{}: \ouralgolong{}}\label{alg:unicrank}
\begin{algorithmic}[1]
\REQUIRE number of items $L$, number of positions $K$
\FOR{$t =  1, 2, \dots$}
    \STATE compute the leader partition $\tilde{\Pm}(t)$
    \STATE $\Pm(t) \gets \argmax_{\Pm\in \left\{\tilde{\Pm}(t)\right\} \cup \Nc\left(\tilde{\Pm}(t)\right)} b_{\Pm}(t)$
    \STATE draw the recommendation $\av(t)$ uniformly at random in $\Ac\left(\Pm(t)\right)$  
    \STATE observe the clicks vector $\cv(t)$
\ENDFOR
\end{algorithmic}
\end{algorithm}

\section{\ouralgo{} Algorithm}\label{sec:unirank}

Our algorithm, \ouralgo{}, is detailed in \cref{alg:unicrank}, and  \cref{fig:one_iteration} unfolds one of its iterations.
 This algorithm takes inspiration from the unimodal bandit algorithm OSUB \citep{Combes2014} by selecting at each iteration $t$ an \emph{arm to play} $\Pm(t)$ in the neighborhood of the current best one $\tilde\Pm(t)$ (a.k.a. the \emph{leader}).
However, \ouralgo{}'s arms are not recommendations but sets of recommendations represented by ordered partitions. Hence, the recommendation $\av(t)$ is drawn uniformly at random in the subset $\Ac(\Pm(t))$ of recommendations compatible with $\Pm(t)$.


Let us now first define the notations used by \ouralgo{} and then present its concrete behaviour.

\begin{figure}[t]
\centering
\pgfmathsetmacro{\nodeDistance}{0.2}
\pgfmathsetmacro{\nodeDistancex}{1}
\begin{tikzpicture}[align=center,node distance=\nodeDistance and \nodeDistancex, scale=0.8]
    \node (1) {\footnotesize 1};
    \node (2) [below=of 1] {\footnotesize 2};
    \node (3) [right=of 1] {\footnotesize 3};
    \node (4) [right=of 3] {\footnotesize 4};
    \node (5) [below=of 4] {\footnotesize 5};
    \node (6) [right=of 4] {\footnotesize 6};
    \node (7) [below=of 6] {\footnotesize 7};
    
    \draw[thick, ->] (1) -- (3);
    \draw[thick, ->] (3) -- (4);
    \draw[thick, ->] (3) -- (5);
    \draw[thick, ->] (4) -- (6);
    \draw[thick, ->] (5) -- (7);
    \draw[thick, ->] (7) -- (6);
    
    \draw[dashed] ($(1)!0.5!(2)$) ellipse (1*\nodeDistance cm and 3.2*\nodeDistance cm);
    \draw[dashed] (3) ellipse (1*\nodeDistance cm and 1.2*\nodeDistance cm);
    \draw[dashed] ($(4)!0.5!(5)$) ellipse (1*\nodeDistance cm and 3.2*\nodeDistance cm);
    \draw[dashed] ($(6)!0.5!(7)$) ellipse (1*\nodeDistance cm and 3.2*\nodeDistance cm);
    

    



\end{tikzpicture}

    {%
    \scriptsize
    \begin{align*}
    \tilde{\Pm} &= \left(\tilde{P}_1, \tilde{P}_2, \tilde{P}_3, \tilde{P}_4\right) =
    \left(\{1,2\}, \{3\}, \{4,5\}, \{6,7\}\right)\\
        \Nc(\tilde{\Pm}) &
        = \{\left(\{1,2, 3\}, \{4,5\}, \{6,7\}\right), \qquad\text{\% merge of }\tilde{P}_1 \text{ and } \tilde{P}_2\\
        &\qquad\left(\{1,2\}, \{3, 4,5\}, \{6,7\}\right),\qquad\text{\% merge of }\tilde{P}_2 \text{ and } \tilde{P}_3\\
        &\qquad\left(\{1,2\}, \{3\}, \{4,5, 6\}, \{7\}\right), \quad\text{\% try } 6\\
        &\qquad\left(\{1,2\}, \{3\}, \{4,5, 7\}, \{6\}\right)\} \quad\text{\% try } 7\\
    \Pm &= \left(\{1,2\}, \{3,4,5\}, \{6,7\}\right)\qquad\text{\% the 2$^{nd}$ neighbor wins}\\
    \av &= (2,1,3,5)
    \end{align*}%
    }%
    \caption{One iteration of \ouralgo{} with $L=7$ items and $K=4$ positions ($t$ is omitted for clarity). Each arrow $i\to j$ in the top graph on items means the statistic $\hat{s}_{i,j}$ is non-negative. With these values, the leader partition (represented with dashed ellipses) is $\tilde{\Pm} = \left(\{1,2\}, \{3\}, \{4,5\}, \{6,7\}\right)$, where $\tilde{P}_4=\{6,7\}$ gathers remaining items as the 3 first partitions contain more than $K$ items. 
    Then, we assume that $\max(\bar{\bar s}_{3,1},\bar{\bar s}_{3,2}) > \max(
\max(\bar{\bar s}_{4,3},\bar{\bar s}_{5,3}),\quad
\max(\bar{\bar s}_{6,4},\bar{\bar s}_{6,5}),\quad
\max(\bar{\bar s}_{7,4},\bar{\bar s}_{7,5}),\quad
    0)$.
    Therefore, \ouralgo{} plays the optimistic partition $\Pm=\left(\{1,2\}, \{3,4,5\}, \{6,7\}\right)$. Finally the recommendation $\av$ is obtained by concatenating a random permutation of $P_1=\{1,2\}$ with a random permutation of 2 items from $P_2=\{3,4,5\}$.
    }
    \label{fig:one_iteration}
\end{figure}
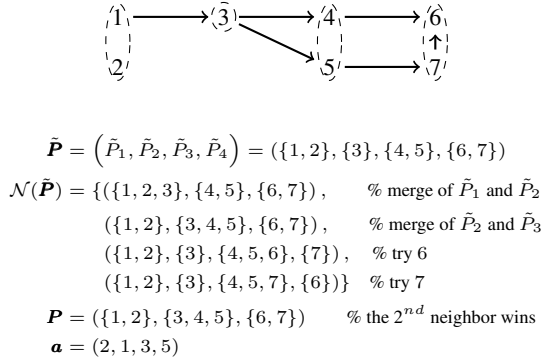

\paragraph{Statistic $\hat{s}_{i,j}(t)$}
\ouralgo{}'s choices are based on the statistic $\hat{s}_{i,j}(t)$ and the optimistic estimator of its expected value: the Kullback-Leibler-based one denoted $\bar{\bar s}_{i,j}(t)$. $\hat{s}_{i,j}(t)$ is the average value of $c_i(s)-c_j(s)$ for $s$ in $[t-1]$, where we restrict ourselves to iterations at which items $i$ and $j$ are in the same subset of the played partition \scalebox{0.9}{$\Pm(s)= \left(P_1(s), \dots, P_{d(s)}(s)\right)$}, and \scalebox{0.9}{$c_i(s)\neq c_j(s)$}. Specifically, 
$$\hat{s}_{i,j}(t) \defeq \frac{1}{T_{i,j}(t)}\sum_{s = 1}^{t-1} O_{i,j}(s)\left(c_i(s) - c_j(s)\right),$$
where \scalebox{0.9}{$O_{i,j}(s)\defeq\ind\left\{\exists c, (i,j)\in P_c(s)^2\right\}\ind\{c_i(s) \neq c_j(s)\}$} denotes that the difference between items $i$ and $j$ is observable at iteration $s$, $T_{i,j}(t) \defeq \sum_{s = 1}^{t-1} O_{i,j}(s)$,
and $\hat{s}_{i,j}(t)\defeq0$  when $T_{i,j}(t)=0$.
Note that $\hat{s}_{i,j}(t)$ is antisymmetric ($\hat{s}_{i,j}(t)=-\hat{s}_{j,i}(t)$) and $\hat{s}_{i,j}(t)>0$ (equivalent to $\hat{s}_{j,i}(t)<0$) indicates that $i$ is probably more attractive than $j$.

The statistics $\hat{s}_{i,j}(t)$ are paired with their respective optimistic \emph{indices} $$\bar{\bar s}_{i,j}(t) \defeq 2*f\left(\frac{1+\hat{s}_{i,j}(t)}{2}, T_{i,j}(t), \tilde{t}_{\tilde{\Pm}(t)}(t)\right)-1,$$
where
$f$ is a function from $[0,1]\times\NN\times\NN$ to $[0,1]$ and 
$f(\hat\mu, N, t) \defeq \sup \{\mu\in[\hat\mu, 1]: N\times\KL(\hat\mu, \mu) \leq \log (t) + 3 \log(\log (t))\},$
with 
$\KL(p,q) \defeq p \log \frac{p}{q} + (1-p)\log \frac{1 - p}{1 - q}$
the \emph{Kullback-Leibler divergence} (KL) from a Bernoulli distribution of mean $p$ to a Bernoulli distribution of mean $q$;
$f(\hat\mu, N, t) \defeq 0$ when $\hat\mu=1$, $N=0$, or $t=0$;
and $\tilde{t}_{\tilde{\Pm}(t)}(t)$ is the number of iterations the partition at which $\tilde{\Pm}$ has previously been the leader.
This optimistic index is the one used for KL-based bandit algorithms, after a rescaling of $\hat{s}_{i,j}(t)$ to the interval $[0,1]$.
Note that, unlike $\hat{s}_{i,j}(t)$, $\bar{\bar s}_{i,j}(t)$ is not antisymmetric, and $\bar{\bar s}_{j,i}(t)\geqslant0$ while $\hat{s}_{i,j}(t)>0$ indicates that it is unclear whether $i$ is more attractive than $j$ or not.

\paragraph{Leader Elicitation}
At each iteration, \ouralgo{} first builds a partition $\tilde{\Pm}(t) = (\tilde{P}_1(t), \dots, \tilde{P}_{\tilde{d}}(t))$ using \cref{alg:leader_elicitation} (see. \cref{app:algo_elicitation}). This partition is \emph{coherent} with $\hat{s}_{i,j}(t)$, meaning that for any couple of items $(i, j)$ in $[L]^2$, if $\hat{s}_{i,j}(t) > 0$ then either $i$ belongs to a subset $\tilde{P}_c(t)$ ranked before the subset of $j$, or there exists a cycle $(i_1, i_2, \dots, i_N)$ such that $i_1=i_N=i$, $i_2=j$, and for any $n\in[N-1]$, $\hat{s}_{i_n,i_{n+1}}(t) > 0$.
We also ensure that the $\tilde{d}-1$ first subsets of $\tilde{\Pm}(t)$ gather at least $K$ items:  
$\sum_{c=1}^{\tilde{d}-2} |\tilde{P}_c(t)| < K \leqslant \sum_{c=1}^{\tilde{d}-1}|\tilde{P}_c(t)|$.
This means that the items in $\tilde{P}_{\tilde{d}}(t)$ are the ones which are never displayed by the recommendations in $\Ac (\tilde{\Pm}(t))$. Note that the subset $\tilde{P}_{\tilde{d}}(t)$ may be empty.

The partition $\tilde{\Pm}(t)$ is built by repeating the process of (i) identifying the smallest subset of items dominating all other items (meaning the items $i$ for which $\hat{s}_{i,j}(t) > 0$ for any remaining item $j$), and (ii) removing this subset. A special care is taken to  gather in the same subset remaining items as soon as the first subsets contain more than $K$ items.

\paragraph{Optimistic Partition Elicitation}
The partition $\tilde{\Pm}(t)$ plays the role of leader, meaning that at each iteration, \ouralgo{} solves an exploration-exploitation dilemma and picks either $\tilde{\Pm}(t)$ or a permutation $\Pm(t)$ in the neighborhood $\Nc(\tilde{\Pm}(t))$ of $\tilde{\Pm}(t)$, where $\Nc(\tilde{\Pm})\defeq$
{\scriptsize
\begin{align*}
&\left\{\left(\tilde{P}_1, \dots,\tilde{P}_{c-1}, \tilde{P}_c \cup \tilde{P}_{c+1},\tilde{P}_{c+2},\dots \tilde{P}_{\tilde{d}}\right) : c \in [\tilde{d}-2]\right\}\\
&\cup
\left\{\left(\tilde{P}_1, \dots,\tilde{P}_{\tilde{d}-2},\tilde{P}_{\tilde{d}-1}\cup\{j\}, \tilde{P}_{\tilde{d}}\setminus\{j\}\right) : j \in \tilde{P}_{\tilde{d}}\right\}.
\end{align*}%
}%
This neighborhood results either (i) from the merge of two consecutive subsets $\tilde{P}_c(t)$ and $\tilde{P}_{c+1}(t)$ of the partition $\tilde{\Pm}(t)$, or (ii) from the addition to $\tilde{P}_{\tilde{d}-1}(t)$ of an item $j$ from the last subset.
For each neighbor $\Pm$ of type (i), the optimistic index $b_{\Pm}(t)$ is $\max_{(i,j)\in\tilde{P}_c(t) \times \tilde{P}_{c+1}(t)}\bar{\bar s}_{j,i}(t)$ to reflect whether or not at least one of the items in $\tilde{P}_{c+1}(t)$ may potentially be more attractive than one of the items in $\tilde{P}_{c}(t)$.
Similarly, for each neighbor $\Pm$ of type (ii), the optimistic index $b_{\Pm}(t)$ is $\max_{i\in\tilde{P}_{\tilde{d}-1}(t)}\bar{\bar s}_{j,i}(t)$. \commentcsg{comme j est utiliser entre temps pour la definition de $b_P$ pourn une partition (i) ca se comprend mais je ne sais s'il ne vaut pas mieux séparer directement en (i) def du voisin par permutation et def de $\bar{\bar{s}}$ puis idem avec le voisin de type ii}

\begin{remark}[Recommendation chosen at random]
Taking a random permutation is required to control the statistic $\hat{s}_{i,j}(t)$. Indeed, the theoretical analysis requires the probability for $i$ to be ranked before $j$ in the recommendation to be even. Overall, the aim is to identify a partition $\Pm^*$ such that any permutation in $\Ac\left(\Pm^*\right)$ is compatible with the unknown strict weak order on items.
\end{remark}

\section{Theoretical Analysis}\label{sec:unirank-th}
The proof of the upper-bound on the regret of \ouralgo{} follows a similar path as the proof of OSUB \citep{Combes2014}: (i) apply a standard bandit analysis to control the regret under the condition that the leader $\tilde\Pm(t)$ is an optimal partition, and (ii) upper-bound by $\OO(\log\log T)$ the expected number of iterations such that $\tilde\Pm(t)$ is not an optimal partition.
However, both steps differ from \citep{Combes2014}.
First, \ouralgo{} handles partitions instead of recommendations.
Secondly, it builds upon $\hat{s}_{i,j}(t)$ instead of estimators of the expected reward.
While $\hat{s}_{i,j}(t)$ is the average of dependent random variables with different expected values, these expected values are greater than some non-negative constant $\tilde\Delta_{i,j}$ when $i\succ j$, which is sufficient to lower-bound $\hat{s}_{i,j}(t)$ away from $0$ as required by the proof of the regret upper-bound (see Appendices \ref{app:min_s}, \ref{app:concentration}, and \ref{app:sufficient_optimism} for details). 
Finally, the proof is adapted to handle the fact that $T_{i,j}(t)$ randomly increases when we play items $i$ and $j$ due to the exploration-exploitation rule, which is unusual in the  bandit literature.
Up to our knowledge, this exploration-exploitation strategy and its analysis are new in the bandit community. We believe that it opens new perspectives for other semi-bandit settings.

Note that, as in \citep{Sentenac2021} and \citep{GRAB}, we restrict the theoretical analysis to the setting where the order on top-items is total, meaning we use Assumption 3.1$^*$.
Without loss of generality, we also assume that $1 \succ 2 \succ \dots \succ K \succ [L]\setminus[K]$ to shorten the notations. 
Hence the only partition $\Pm^*$ which is such that, any permutation $\av$ in $\Ac\left(\Pm^*\right)$ is compatible with the unknown strict order on items, is $\left(\{1\}, \dots,\{K\}, [L]\setminus[K]\right)$.

We now propose the main theorem that upper-bounds the regret of \ouralgo{}. 

\begin{theorem}[Upper-bound on the regret of \ouralgo{} assuming a total order on top-$K$ items]\label{theo:unirank}
Let $(L, K, \rho)$ be an OLR problem satisfying Assumptions 3.1$^*$,  \ref{asp:acyclicity}, and \ref{asp:identifiability} and such that $1 \succ 2 \succ \dots \succ K \succ [L]\setminus[K]$. Denoting $\Pm^* = (\{1\},\dots,\{K\}, [L]\setminus[K])$ the optimal partition associated to this order, 
when facing this problem, \ouralgo{} fulfills
\begin{multline}
\forall k\in[L]\setminus\{1\},~\label{eq:good_leader}
\EE\left[ \sum_{t=1}^T\ind\left\{\substack{\tilde\Pm(t)=\Pm^*,\\~\exists c, P_{c}(t)=\{{\min(k-1, K)}, k\}}\right\}\right]
    \\\leqslant \frac{16}{\tilde\delta_k^*\tilde\Delta_k^2}\log T + \OO\left(\log\log T\right)
\end{multline}
\begin{equation}
\label{eq:wrong_leader}
\text{and }
\EE\left[\sum_{t=1}^T\ind\{\tilde\Pm(t)\neq\Pm^*\}\right]
    = \OO\left(\log\log T\right),
\end{equation}
and hence
$$R(T)
\leqslant \sum_{k=2}^{L}\frac{8\Delta_k}{\tilde\delta_k^*\tilde\Delta_k^2}\log T
+ \OO\left(\log\log T\right)
=\OO\left(\frac{L}{\Delta}\log T\right),
$$
where
for any position $k>1$, denoting $\ell\defeq\min(k-1, K)$,\\
\scalebox{0.95}{$
\tilde{\delta}_k^* \defeq
\min_{\substack{\Pm\in\Nc(\Pm^*): \exists c, (\ell,k)\in P_c^2}}
\PP_{\av(t) \sim \Uc(\Ac\left(\Pm\right)) }\left[c_{\ell}(t) \neq c_k(t)\right],
$}\\
$
\tilde\Delta_k \defeq \min_{\av\in\perm_K^L:\{\ell,k\}\cap\av([K])\neq\varnothing}\tilde\Delta_{\ell,k}(\av),
$\\
$\Delta_k \defeq \mu_{(1,\dots,K)} - \mu_{(\ell,k)\circ(1,\dots,K)},
$\\
and $\Delta \defeq \min_{k\in\{2, \dots, L\}} \tilde\delta_k^*\tilde\Delta^2_k / \Delta_k$.
\end{theorem}

The first upper-bound (Equation \eqref{eq:good_leader}) controls the expected number of iterations at which \ouralgo{} explores while the leader is the optimal partition. Both types of exploration are covered: the merging of two consecutive subsets of $\tilde\Pm(t)$, and the addition of a sub-optimal arm to the last subset of the chosen partition $\Pm(t)$.
The second upper-bound (Equation \eqref{eq:wrong_leader}) deals with the expected number of iterations at which the leader is not the optimal partition.
Let us now express the same bounds while assuming one of the state-of-the-art click models. 

\begin{corollary}[Facing CM$^*$]\label{theo:unirank_CM}
Under the hypotheses of  \cref{theo:unirank}, with the clik-model CM with  probability $\theta_i$ to click on item $i$ when it is observed, \ouralgo{} fulfills
\begin{align*}
R(T)
&\leqslant \sum_{k=K+1}^L16\frac{\theta_K+\theta_k}{\theta_K-\theta_k}\log T
+ \OO\left(\log\log T\right)
\\&
= \OO\left((L-K)\frac{\theta_K+\theta_{K+1}}{\theta_K-\theta_{K+1}}\log T\right).
\end{align*}
\end{corollary}

\begin{corollary}[Facing PBM$^*$]\label{theo:unirank_PBM}
Under the hypotheses of  \cref{theo:unirank}, if the user follows PBM with  the probability $\theta_i$ of clicking on item $i$ when it is observed and  the probability $\kappa_k$ of observing the position $k$, then \ouralgo{} fulfills
\begin{align*}
R(T)
= \OO\left(\frac{L}{\Delta}\log T\right),
\end{align*}
where
$\Delta \defeq \min \{
\frac{\theta_K-\theta_{K+1}}{\theta_K+\theta_{K+1}},$\\
\scalebox{0.95}{$
\min_{k \in\{2,\dots,K\}} \frac{((\kappa_{k-1}+\kappa_k)(\theta_{k-1}+\theta_k)-4\kappa_{k-1}\kappa_k\theta_{k-1}\theta_k)(\theta_{k-1}-\theta_k)}{(\kappa_{k-1}-\kappa_k)(\theta_{k-1}+\theta_k)^2}
\}$}.
\end{corollary}

Note that the regret upper-bound reduces to $\OO((L-K)/\Delta\log T)$ with CM since, with this model, the recommendation is optimal as soon as optimal items are displayed. 

A more detailed version of these corollaries is given in the appendix, together with their proofs and  \cref{theo:unirank}'s  proof. These proofs builds upon the following pseudo-unimodality property.

\begin{lemma}[Pseudo-unimodality assuming a total order on top-$K$ items]\label{lem:unimodality}
Under the hypotheses of  \cref{theo:unirank}, for any ordered partition of the items $\tilde\Pm = \left(\tilde{P}_1,\dots, \tilde{P}_{\tilde{d}}\right) \neq \Pm^*$,
\begin{itemize}
    \item either  $\exists c \in [\tilde{d}]$, such that $|P_c|>1$ and  $i^*\succ \argmax_{j\in P_c\setminus\{i^*\}} g(j)$, where $i^*=\argmax_{i\in P_c} g(i)$;
    \item or $\exists c \in [\tilde{d}-1]$, $\exists (i,j)\in\tilde{P}_c\times\tilde{P}_{c+1}$, such that $j \succ i$.
\end{itemize}
\end{lemma}

The first alternative implies that the subset $\tilde{P}_c$ should be split, which will be discovered by recommending permutations compatible with either $\tilde{\Pm}$ or one of its neighbors. The second alternative implies that $j$ should be in a subset ranked before the subset containing $i$, which will be discovered by recommending the permutation in the neighborhood of $\tilde{\Pm}$ which puts $i$ and $j$ in the same subset.

\subsection{Discussion}\label{sec:discussion}

We gather here some remarks regarding the optimality of the theoretical results and their extension to a weak order.

\begin{remark}[Optimality of UniRank's upper-bound]
While deriving the exact lower-bound on the expected regret in this setting is out of the scope of our paper, we believe that this bound takes the form $\OO\left(\sum_{k=2}^K \frac{\mu^*-\mu_k}{kl(\nu_k^*,\nu_k)}\log(T) + \sum_{k=K+1}^L \frac{\mu^*-\mu_k}{kl(\nu_k^*,\nu_k)}\log(T)\right)$, where for $k\in\{2,\dots,K\}$ (respectively $k\in\{K+1,\dots,L\}$), $\mu_k$, $\nu_k^*$, and $\nu_k$ result from the best partition and the best random variables to compare item $k$ to item $k-1$ (resp. to item $K$). 

In \cite{Combes2015} (Propositions 1 and 2) and in \cite{Lagree2016} (Theorem 6) a bound with only the second sum is proven. The first sum is missing as both papers consider more restricting settings where the comparison between items $k\in\{2,\dots,K\}$ and $k-1$ is free in terms of regret: CM for \cite{Combes2015}, and PBM with $\kappav $ known for \cite{Lagree2016}.

We also believe that TopRank's upper-bound on the regret with our additional hypothesis either remains $\OO(KL/\Delta\log(T))$ or reduces to $\OO(L\log L/\Delta\log(T))$.\footnote{This second bound is proven in \citep{Sentenac2021} for a matching problem handled with an algorithm similar to TopRank.}
Indeed, while UniRank reduces the exploration by only comparing each item $k$ to the item $\min(k-1,L)$, in the worst case scenario TopRank compares each item $k$ to each item $k'\in\{1,\dots, \min(k-1,L)\}$ in order to conclude that $k'$ should not be at one of the top-$\min(k-1,L)$ positions. 
\end{remark}

\begin{remark}[Exploration not at the top]
With CM model, exploring at the top reduces the regret: it leads to less exploration, while the instantaneous regret remains unchanged.
However, this results does not hold with PBM (see Theorem 6 in \cite{Lagree2016} for details): while exploring at the top decreases the number of explorations, it also increases the regret per exploration; and the best trade-off depends on the values $\thetav$ and $\kappav$.

Therefore, as in \cite{Lagree2016}, we only explore through local changes in the recommendation. Note that these local changes are also more "user-friendly": as soon as the right leader has been identified, a sub-optimal item is always tried at the bottom of the recommendation, which is less surprising for users than a sub-optimal item displayed as the top recommendation.
\end{remark}

\begin{remark}[Upper-bound on the regret of UniRank assuming a weak order on items]
If the order on the best items is not total, the proof of  \cref{theo:unirank} may be adapted to get a $\OO\left(LK/\Delta\log T\right)$ bound. Indeed, under the strict weak order assumption, there exists a set
of optimal partitions, and therefore, any permutation compatible with a neighbor of any of these partitions may be recommended $\OO\left(1/\Delta\log T\right)$ times. In the worst case scenario, $K$ items are equivalent and strictly more attractive than the $L-K$ remaining items, and the set of the permutations compatible with a neighbor partition is composed of $K(L-K)$ permutations, which translates into a  $\OO\left(LK/\Delta\log T\right)$ regret bound.
Note that \cite{TopRank} proves a $\Omega\left(LK/\Delta\log T\right)$ lower-bound on the regret assuming that the best items have the same attractiveness which means that the upper-bound of \ouralgo{} for this specific setting is optimal.
\end{remark}

\forlater{(pour plus tard) ajouter une remarque sur $\tilde\Delta$ vs. on pense/"sait" que ça devrait être $\tilde\Delta^*$ qui vérifie notamment $\Delta=\tilde\delta^*\Delta^*$, mais ça demande une preuve compliquée que bras sous-optimaux sont tirés $\Omega(\log T)$ fois.
\\ 
En CM ça ne change rien, en PBM ça donnerait ...\\
Et mini-remarque : si on suppose $\Delta^*$, on voit l'intrêt de conditionnement $c_i\neq c_j$
}

\forlater{remarque: Unirank s'intéresse à la probabilité qu'un item soit meilleur qu'un autre, et pas à la perte en récompense lorsque l'un met deux items dnas le mauvais ordre. Du coup le même algo fonctionne pour toute fonction de récompense telle que le maximum soit atteint quand l'ordre est le bon, y compris par exemple le NDCG.}

\section{Experiments}\label{sec:exp}
In this section, we compare \ouralgo{} to TopRank \citep{TopRank}, PB-MHB \citep{PBMHB}, GRAB \citep{GRAB}, and CascadeKL-UCB \citep{Kveton2015a}. The experiments are conducted on the KDD Cup 2012 track 2 dataset, on the Yandex dataset \citep{Yandex}, and on a model with artificial parameters. We use the cumulative regret to evaluate the performance of each algorithm.


\subsection{Experimental Settings}\label{sec:exp_set}

In order to evaluate our algorithm, we design six experiments inspired by the ones conducted in \citep{TopRank}.
The standard metric used is the expected cumulative regret (see Equation \eqref{eq:regret}), denoted as \emph{regret}, which is the sum, over  $T$ consecutive recommendations, of the difference between the expected reward of the best answer and of the answer of a given ORS. The best algorithm is the one with the lowest regret.  
We use two click models for our experiments: the Position Based Model (PBM) and the Cascading Model (CM).
To play according to those models, we extract the parameters of the chosen model from the KDD Cup 2012 track 2 (KDD for short) database and the Yandex database \citep{Yandex}, and we experiment with a set of parameters (denoted \texttt{Simul}) chosen to highlight the $\OO\left(L/\Delta\log T\right)$ regret of \ouralgo{}: $L=10$, $K=5$, $\thetav=[0.1, 0.08, 0.06, 0.04, 0.02, 10^{-4}, 10^{-4}, 10^{-4}, 10^{-4}, 10^{-4}]$, and $\kappav=[1, 0.9, 0.83, 0.78, 0.75]$.

Yandex database comes from fully anonymized real-life logs of actions toward the Yandex search engine. It contains 703 million items displayed among 65 million search queries  and sharing 167 million hits (clicks). 
We consider the 10 most frequent queries in our experiments.
We use the GPL3 Pyclick library \citep{Chuklin2015} to infer the CM and PBM parameters of each query with the \emph{expectation maximization} algorithm.
Depending on the query, this leads to $\theta_i$ values ranging from 0.51 to 0.94, and $\kappa_i$ values ranging from 0.71 to 1.00  when considering PBM and $\theta_i$ values ranging from 0.03 to 0.50 for CM.

We also extract parameters from the KDD dataset. Due to the type of data contained in this dataset, we can only extract parameters for the PBM model. This dataset consists of session logs of \emph{soso.com}, a Tencent's search engine. It tracks clicks and displays of advertisements on a search engine result web-page, w.r.t. the user query. For each query, 3 positions are available for a various number of ads to display. Each of the 150M lines contains information about the search (UserId, QueryId\dots) and the ads displayed (AdId, Position, Click, Impression). We are looking for the best ads per query, namely the ones with a higher probability to be clicked. To follow previous works, instead of looking for the probability to be clicked per display, we target the probability to be clicked per session. This amounts to discarding the information \emph{Impression}.  
We also filter the logs to restrict the analysis to (query, ad) couples with enough information: for each query, ads are excluded if they were displayed less than 1,000 times at any of the 3 possible positions. Then, we filter queries that have less than 5 ads satisfying the previous condition. We end up with 8 queries and from 5 to 11 ads per query.
The overall process leads to $\theta_i$ values ranging from 0.004 to 0.149, and $\kappa_k$ values ranging from 0.10 to 1.00, depending on the query.

Then we simulate the users' interactions given these parameters as it is commonly done in bandits settings.
Similarly to \citep{TopRank}, we look at the results averaged on the queries, while displaying $K$ items among the $L$ most attractive ones selected among all items possible for each query. With Yandex dataset, $K=5$ and $L=10$, while with KDD dataset $K=3$ and $L$ varies from 5 to 11.
We run our experiments on an internal cluster to compute 20 independent sets of $10^7$ consecutive recommendations for each of the 10 most frequent Yandex queries and each of the 8 KDD queries. It leads respectively to 200 games per setting and  algorithm for Yandex and 160 games for KDD.
As TopRank requires the knowledge of the horizon $T$, we test the impact of this parameter by setting it to the right value ($10^7$), to a too high value ($10^{12}$), and to a too small value ($10^5$) with doubling trick.
To tune PB-MHB, we use the values recommended by \citep{PBMHB} for these datasets.

  \begin{figure*}[t]
    \vskip 0.2in 
    \centering%
    \begin{minipage}[t]{0.74\linewidth}
    \vspace{0pt}
    \centering
    \begin{subfigure}[b]{0.31\linewidth}
        \centering
        \includegraphics[width=\linewidth]{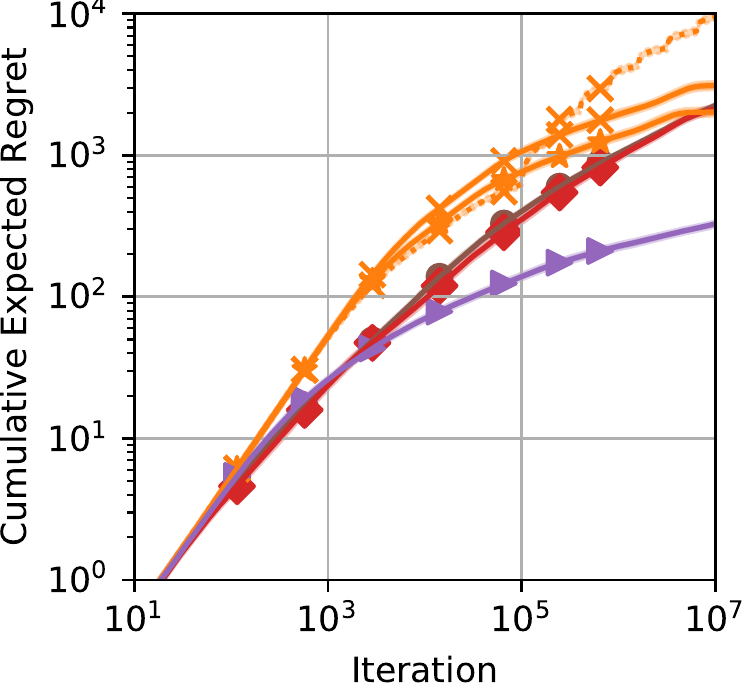}
        \caption{KDD PBM}
         \label{fig:PBM_KDD}
    \end{subfigure}%
    \hfill%
    \begin{subfigure}[b]{0.31\linewidth}
        \centering
        \includegraphics[width=\linewidth]{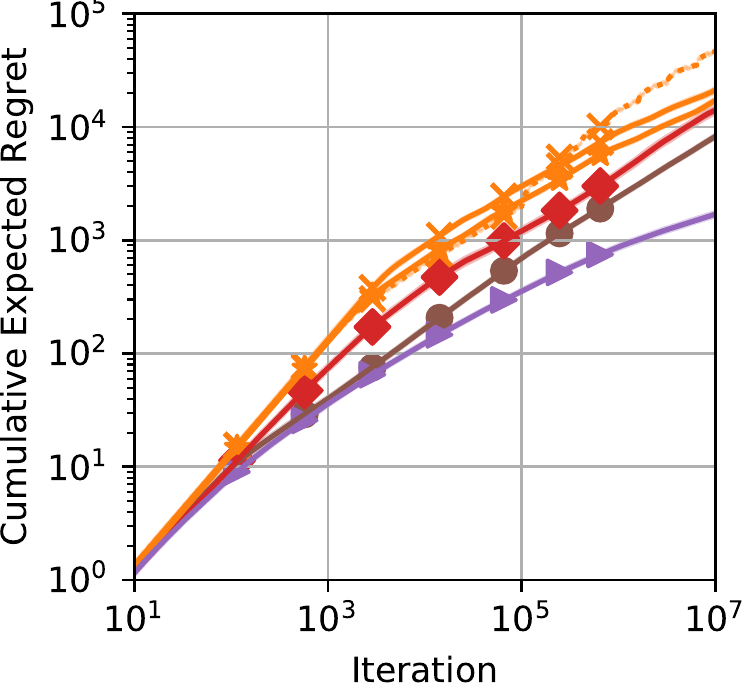}
        \caption{Yandex PBM}
         \label{fig:PBM_Yandex}
    \end{subfigure}%
    \hfill%
    \begin{subfigure}[b]{0.31\linewidth}
        \centering
        \includegraphics[width=\linewidth]{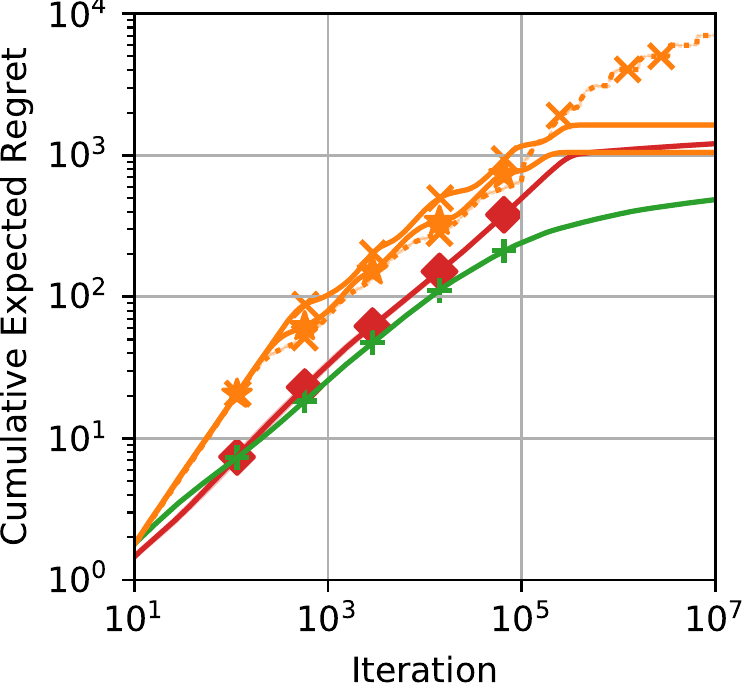}
        \caption{Yandex CM}
        \label{fig:CM_Yandex}
    \end{subfigure}
    \vskip 0.15in
    \begin{subfigure}[t]{0.31\linewidth}
        \centering
        \includegraphics[width=\linewidth]{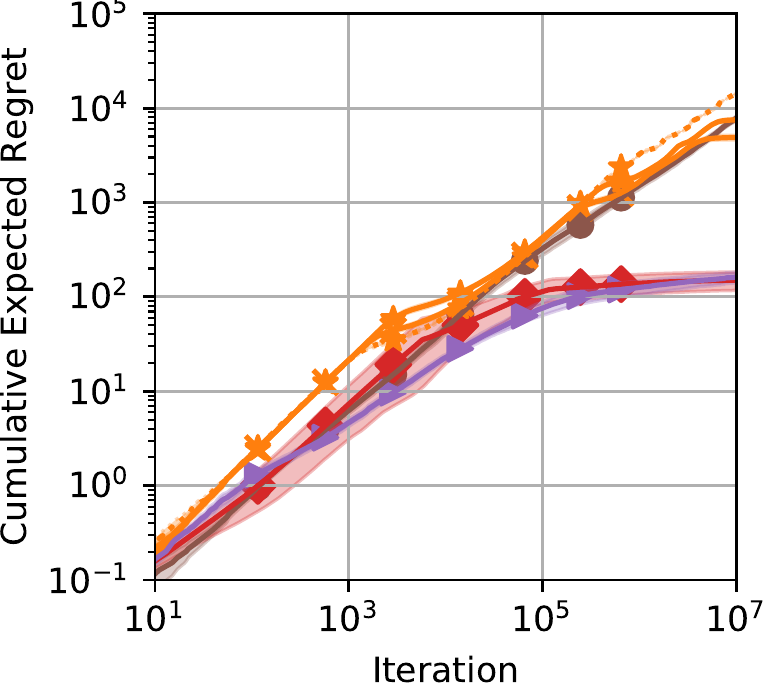}
        \caption{Yandex 8107157 PBM}
        \label{fig:PBM_Yandex9}
    \end{subfigure}%
     \hfill%
    \begin{subfigure}[t]{0.31\linewidth}
        \centering
        \includegraphics[width=\linewidth]{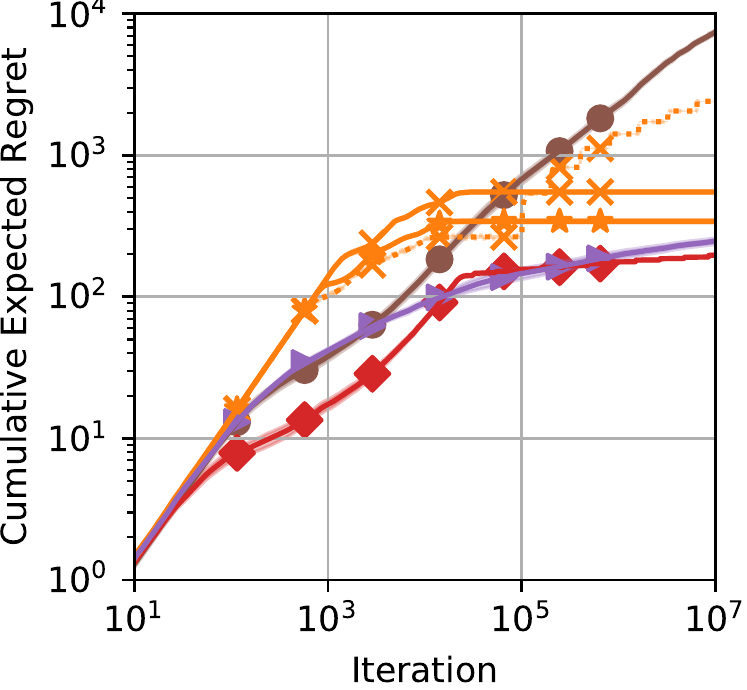}
        \caption{Simul PBM}
        \label{fig:PBM_simul}
    \end{subfigure}%
     \hfill%
    \begin{subfigure}[t]{0.31\linewidth}
        \centering
        \includegraphics[width=\linewidth]{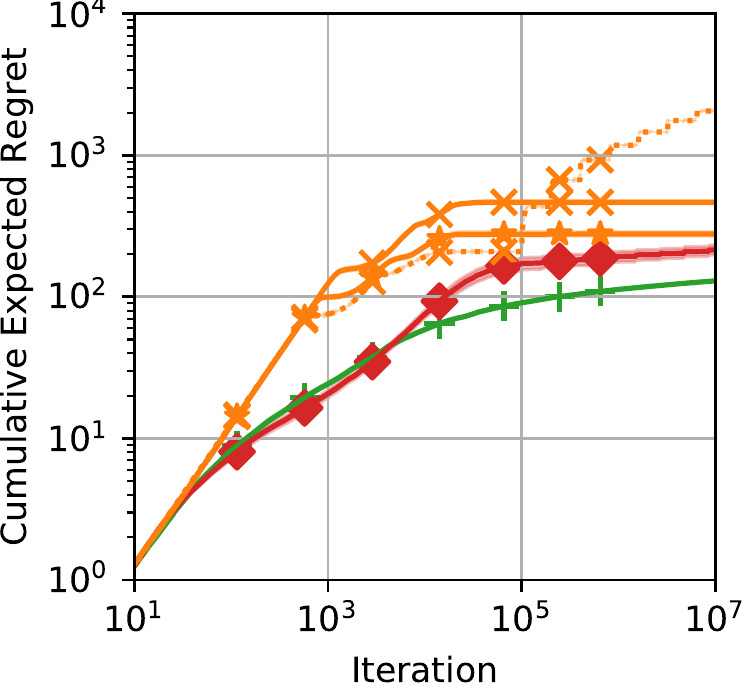}
        \caption{Simul CM}
        \label{fig:CM_simul}
    \end{subfigure}%
    \end{minipage}
    \quad%
    \fbox{
    \begin{minipage}[t]{0.21\linewidth}
    \vspace{0pt}
        \vskip 0.1in
        \text{\textbf{Generic Algorithms}}
        \includegraphics[height=4em]{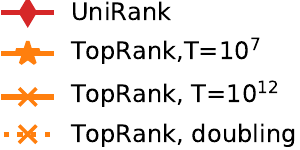}

        \vskip 0.15in
        \text{\textbf{Alg. Dedicated to PBM}}
        \includegraphics[height=2em]{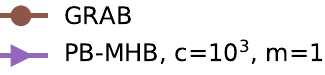}

        \vskip 0.15in
        \text{\textbf{Alg.  Dedicated to CM}}
        \includegraphics[height=1em]{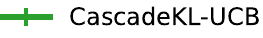}
    \end{minipage}}%

\caption{Cumulative regret w.r.t. iterations.
$K=5$ and $L=10$ for Yandex and Simul models (\subref{fig:PBM_Yandex},\subref{fig:CM_Yandex},\subref{fig:PBM_simul},\subref{fig:CM_simul});
$K=3$ and $L\in\{5,\dots,11\}$ for KDD model (\subref{fig:PBM_KDD});
$K=5$ and $L=6$ for Yandex 8107157 (\subref{fig:PBM_Yandex9}) which corresponds to the parameters of the query 8107157 of Yandex.
The plotted curves correspond to the average over 200, 160, or 20 independent sequences of recommendations (20 sequences per query). The (small) shaded areas depict the standard error of our regret estimates.}
\label{fig:Exp_Yandex}
\label{fig:Exp_KDD}
\vskip -0.2in
\end{figure*}

\commentrg{légende en PNG}

\begin{table}[t!]
 \caption{Average computation time (in ms) per recommendation. For each top 10 query of Yandex dataset, 20 runs are performed assuming CM model and $L=10$.}
  \label{tab:computation}
\vskip 0.15in
\begin{center}
  \begin{tabular}{cc}
    \toprule
    {Algorithm} &{Computation Time (ms)}\\
    \hline
    \ouralgo{} & $1.0\pm0.2$  \\
    TopRank    & $0.7\pm0.3$   \\
    PB-MHB     & $13.9 \pm 4.9$  \\
    GRAB     
    & $0.9 \pm 0.3$   \\
    CascadeKL-UCB
    & $0.9 \pm0.0$  \\
    \bottomrule
  \end{tabular}
\end{center}
 \vskip -0.1in
\end{table}

\subsection{Results}
Our results are shown in  \cref{fig:Exp_KDD}. 
As expected, CascadeKL-UCB (respectively PB-MHB) outperforms other algorithms in the CM (resp. PBM) model for which it is designed. 
However, PB-MHB is computationally expensive (see  \cref{tab:computation}) and lacks a theoretical analysis.
Surprisingly, although GRAB is designed for PBM model, it suffers a high regret when confronted to the query 8107157 of Yandex and to Simul with PBM model.

\forlater{
As expected, CascadeKL-UCB outperforms other algorithms in the CM model for which it is designed and suffers a linear regret (i.e. very high) in the PBM model. Surprisingly, although PB-MHB and GRAB are designed for PBM model, (i) PB-MHB is the best algorithm in all models and even outperforms CascadeKL-UCB in CM models, and (ii) GRAB is ranked second or third in all models (with \ouralgo{} ahead or behind, depending on the setting).
We conjecture that the good results of PB-MHB and GRAB in the CM model results from CM model being equivalent to a PBM model in the neighborhood of the optimal recommendation.
However, PB-MHB is computationally expensive (see  \cref{tab:computation}) and lacks a theoretical analysis.
Similarly, GRAB lacks a theoretical analysis on CM setting. 
}

Secondly, \ouralgo{} and TopRank enjoy a logarithmic regret in all settings and our algorithm \ouralgo{} outperforms TopRank for the models such that the $K$ best items do not have the same attractiveness $\theta_i$: query 8107157 of Yandex, simul PBM, and simul CM.
When confronted to other models, \ouralgo{} has a regret strictly smaller than TopRank before the iteration $t=10^6$, and smaller or equal to TopRank at the horizon.
Moreover, as already explained, TopRank is aware of the horizon $T$ and may stop (over)exploring early, as can be observed in the CM model after iteration $10^5$.
If TopRank targets a horizon $T=10^{12}$ or uses the doubling trick it suffers a higher regret than \ouralgo{}.

Regarding the computational complexity, as shown in  \cref{tab:computation}, PB-MHB is significantly slower with a computation time per recommendation ten times higher than any other algorithm. These other algorithms have a similar computation time of approximately 1 ms per recommendation.

Overall, as TopRank, \ouralgo{} is consistent over all settings, and require a reasonable computation time. Moreover, contrary to TopRank, (i) \ouralgo{} drastically decreases its regret by taking advantage of the differences of attractiveness between items, and (ii) \ouralgo{} does not require the knowledge of the horizon $T$.

\section{Conclusion} \label{sec:conclusion}

We have presented \ouralgo{}, a unimodal bandit algorithm for online ranking. The regret bound in $\OO\left(L/\Delta \log T\right)$
of our algorithm, is a direct consequence of the unimodality-like property of the bandit setting with respect to a graph where nodes are ordered partitions of items. Even though the proof is inspired by OSUB \citep{Combes2014}, the fact that \ouralgo{} handles partitions instead of recommendations, uses different estimators and builds upon an unusual exploration-exploitation strategy makes it original, and we believe that our theoretical analysis opens new perspectives for other semi-bandit settings.
Experiments against state-of-the-art learning algorithms show that our method is consistent in all settings, enjoys a smaller regret than TopRank and GRAB on specific settings, and has a much smaller computation time than PB-MHB.

While in industrial applications, contextual information is also used to build recommendations \citep{Li2019, Chen2019, Ermis2020, Gampa2021},
in this paper we restricted ourselves to independent arms to simplify the presentation of the approach. However, the integration of unimodal bandit algorithms working on parametric spaces \citep{Combes2020} should bridge the gap between both approaches.

\section*{Ethical Statement}
Regarding the societal impact of the proposed approach, it is worth mentioning that the approach aims at identifying and recommending the most popular items. Therefore, the approach may increase the monopoly effects: the most attractive items are displayed more often, so their reputation increases, and then they may become even more attractive\dots{} However bandit algorithms continuously explore and therefore continuously offer an opportunity to less popular items to increase their reputation.

\section*{Acknowledgements}
We thank the reviewers for their valuable comments towards clarification of the paper. 
This research was partially supported by
the Inria Project Lab “Hybrid Approaches for Interpretable AI” (HyAIAI)
and
the network on the foundations of trustworthy AI, integrating learning, optimisation, and reasoning (TAILOR) financed by the EU’s Horizon 2020 research and innovation program under agreement 952215.

\commentrg{Regarder les forlater}


\commentrg{ICML: References must include page numbers whenever possible and be as complete as possible. Place multiple citations in chronological order.
}

\bibliography{bib}
\bibliographystyle{icml2022}

\newpage
\appendix
\onecolumn

\section{Organisation of the Appendix}

The appendix is organized as follows. After listing most of the notations used in the paper in Appendix \ref{app:notations}, we prove Lemma \ref{lem:pbm_and_cm_are_unimodal} in Appendix \ref{app:pbm_and_cm_are_unimodal}.
Then we prove some technical lemmas in Appendix \ref{app:technical_lemmas}, which are required by the proof of Theorem \ref{theo:unirank} in Appendix \ref{app:regret_of_unirank}. Finally, we discuss the regret upper-bound of \ouralgo{} for some specific settings in Appendix \ref{app:vs_others}.


\section{Notations}\label{app:notations}

\begin{table}[htbp]
    \centering
\begin{sc}
\scriptsize
    \begin{longtable}{ll}
        \toprule
        \textbf{Symbol} &\textbf{Meaning}\\
        \midrule
        \endhead
        
        \hline \multicolumn{2}{r}{{Continued on next page}}
        \endfoot
        
        \bottomrule
        \endlastfoot

        T& Time horizon\\
        $t$& iteration \\
        L& number of items \\
        $i$& index of an item  \\ 
        K& number of positions in a recommendation\\
        $k$ & index of a position \\
        $[n]$ & set of integers $\{1,\dots,n\}$\\
        $\perm_K^L$ & set of permutations of K distinct items among L\\
        $\thetav$& vectors of probabilities of click\\
        $\theta_i$& probability of click on item $i$ \\
        $\kappav$& vectors of probabilities of view \\
        $\kappa_k$& probability of view at position $k$ \\
        $\Ac$ & set of bandit arms\\
        $\av$ & an arm in  $\Ac$ \\
        $\av(t)$ & the arm chosen at iteration $t$ \\
        $a_k$ & item displayed at position k in the recommendation $\av$ \\
        $\av^*$ & best arm \\
        $\rho$ & function from $\perm_K^L\times[K]$ to $[0,1]$ giving the probability of click   \\
        $\rho(\av,k)$ & probability of click on the item displayed at position $k$ when recommending $\av$ \\
        $\cv(t)$& clicks vector at iteration $t$ \\
        $c_i(t)$& clicks on item i at iteration $t$ \\
        $r(t)$ & reward collected at iteration $t$, $r(t)=\sum_{i=1}^L c_i(t)$\\
        $\mu_\av$& expectation of $r(t)$ while recommending $\av$, $\mu_\av=\EE[r(t)\mid\av(t)=\av]$ \\
        $\mu^*$& highest expected reward, $\mu^*=\max_{\av\in \perm_K^L}\mu_\av$\\
        $\Delta$& generic reward gap between one of the sub-optimal arms and one of the best arms\\
        $\Delta_c$& reward gap while exchanging items ${\min(c-1,K)}$ and $c$ in the optimal recommendation, 
        \\
        $\tilde\delta_{i,j}$& smallest probability for $c_i(t)$ to be different from $c_j(t)$\\
        & while both items are in the same subset of the chosen partition $\Pm(t)$ \\
        $\tilde\delta_k^*$& smallest probability for $c_{\min(k-1, K)}(t)$ to be different from $c_k$, while both items are in the same \\
        & subset of the chosen partition $\Pm(t)$ and $\Pm(t)$ is in the neighborhood of the optimal partition\\
        $\tilde{\Delta}_{i,j}$& smallest (respectively highest) expected difference of click between items $i$ and $j$ if $i\succ j$ (resp. $j \succ i$)\\
        & while both items are in the same subset of the chosen partition $\Pm(t)$\\
        $R(T)$ & cumulative (pseudo-)regret, $R(T) = T\mu^* - \EE\left[\sum_{t = 1}^T \mu_{\av(t)}\right]$\\
        $\succ$& strict weak order \\
        $(i,j) \circ \av$ & permutation swapping items i and j in recommendation $\av$ \\
        $\Pm$& ordered partition of items representing a subset of recommendations, $\Pm = \left(P_1,\dots, P_d\right)$\\ 
        $P_c$& $c^{th}$ part of $\Pm$ such as $\bigcup_{c=1}^dP_c = [L]$, and $P_c\cap P_{c'}$ is empty when $c\neq c'$ \\
         $\Ac\left(\Pm\right)$& set of recommendations $\av$ agreeing with $\Pm$\\
        $\tilde{\Pm}(t)$&best partition at iteration $t$ given the previous choices and feedbacks (called leader)\\
        $\Pm^* $& partition such that any permutation $\av$ in $\Ac\left(\Pm^*\right)$ is compatible with the strict weak order on items.\\
        $\Gc$ & graph carrying a partial order on the partitions of items\\
        $\Nc(\tilde{\Pm})$&Neighborhood in $\Gc$ of the partition $\Pm$, 
        $\Nc(\tilde{\Pm})\defeq\left\{\left(\tilde{P}_1(t), \dots,\tilde{P}_{c-1}(t), \tilde{P}_c(t) \cup \tilde{P}_{c+1}(t),\tilde{P}_{c+2}(t),\dots \tilde{P}_{\tilde{d}}(t)\right) : c \in [\tilde{d}-2]\right\}$\\&\hfill$\cup
\left\{\left(\tilde{P}_1(t), \dots,\tilde{P}_{\tilde{d}-1}(t)\cup\{j\}, \tilde{P}_{\tilde{d}-1}(t)\setminus\{j\}, \tilde{P}_{\tilde{d}}(t)\right) : j \in \tilde{P}_{\tilde{d}}(t)\right\}.$
\\
        $t_{i,j}(t)$& number of iterations at which items $i$ and $j$ have been gathered in the same subset of items $P_c(s)$,\\
        &$t_{i,j}(t)\defeq\sum_{s = 1}^{t-1}\ind\left\{\exists c, (i,j)\in P_c(s)^2\right\}$ \\
        $T_{i,j}(t)$ & number of iterations at which items $i$ and $j$ have been gathered in the same subset of items $P_c(s)$\\
        &and lead to a different click value, $T_{i,j}(t) = \sum_{s = 1}^{t-1} \ind\left\{\exists c, (i,j)\in P_c(s)^2\right\}\ind\{c_i(s) \neq c_j(s)\}$\\
        $\tilde{t}_{\tilde\Pm}(t)$ & number of time a permutation $\tilde\Pm$ as been the leader, $\tilde{t}_{\tilde\Pm}(t) \defeq\sum_{s = 1}^{t-1}\ind\left\{\tilde\Pm(s)=\tilde\Pm\right\}$ \\
         $\tilde\delta_{i,j}(\av) $& probability of difference, $
\tilde\delta_{i,j}(\av) = \PP_{\av'\sim \Uc(\{\av, (i,j)\circ\av\}) }\left[c_i \neq c_j\right]$\\
         $\tilde\Delta_{i,j}(\av) $& expected click difference, $
\tilde\Delta_{i,j}(\av) = \EE_{\av'\sim \Uc(\{\av, (i,j)\circ\av\}) }\left[c_i-c_j \mid c_i \neq c_j\right]$\\
        $\hat{s}_{i,j}(t)$ & \ouralgo{}'s main statistic to infer that $i\succ j$, $\hat{s}_{i,j}(t)\defeq \frac{1}{T_{i,j}(t)}\sum_{s = 1}^{t-1} \ind\left\{\exists c, (i,j)\in P_c(s)^2\right\}(c_i(s) - c_j(s))$\\
        $\bar{\bar s}_{j,i}(t)$ & Kullback-Leibler based optimistic estimator, $\bar{\bar s}_{j,i}(t)\defeq 2*f\left(\frac{1+\hat{s}_{i,j}(t)}{2}, T_{i,j}(t), \tilde{t}_{\tilde\Pm}(t)\right)-1$\\
        $f$ & Kullback-Leibler index function, $f(\hat\mu, T, t) \defeq \inf \{\mu\in[0, \hat\mu]: T\times\KL(\hat\mu, \mu) \leq \log (t) + 3 \log(\log (t))\},$\\
        $\KL(p,q)$ & Kullback-Leibler divergence from a Bernoulli distribution of mean $p$\\
        &to a Bernoulli distribution of mean $q$,  $\KL(p,q) = p \log \left(\frac{p}{q}\right) + (1-p)\log \left(\frac{1 - p}{1 - q}\right)$ \\
        $\Uc(S)$ & uniform distribution on the set $S$\\
        $c$ & (in PB-MHB) parameter controlling size of the step in the Metropolis Hasting inference \\
\end{longtable}
\end{sc}
    \caption{Summary of the notations.}
    \label{tab:notationsapp}
\end{table}

Table \ref{tab:notationsapp} summarizes the notations used throughout the paper and the appendix. Below are additional notations necessary for the proofs.

\begin{definition}[Specific notations to count events and observations]
The proofs are based on the concentration
of the statistic $\hat{s}_{i,j}(t)$ which is the average over $T_{i,j}(t)$ observations.
The number $T_{i,j}(t)$ itself is a sum: the sum of the random variables
$\ind\{c_i(s) \neq c_j(s)\} \mid \exists c, (i,j)\in P_c(s)^2$, where $s$ is in $[t]$. To discuss the concentration of this sum, for any iteration $t$ in $[T]$, we denote $t_{i,j}(t)\defeq\sum_{s = 1}^{t-1}\ind\left\{\exists c, (i,j)\in P_c(s)^2\right\}$ the number of iterations at which the random variable is observed. 
\end{definition}

\begin{definition}[Recommended subset]
Let  $(L, K, \rho)$ be an online learning to rank problem, $\Pm$ be an ordered partition of $[L]$ in $d$ subsets, and $c\in[d]$ the index of one of these subsets. The subset $P_c$ is \emph{recommended} (denoted $\recommended(P_c)$) if the recommendations compatible with $\Pm$ include some items from $P_c$. More specifically, the subset $P_c$ is \emph{recommended} if $|\bigcup_{\ell\in[c-1]}P_\ell| < K$.
\end{definition}

\begin{definition}[Expectations on clicks]
let $i$ and $j$ be two different items.

We denote 
$$\tilde{\delta}_{i,j} \defeq
\min_{\Pm: \exists c, (i,j)\in P_c^2 \land \recommended(P_c)}
\PP_{\av(t) \sim \Uc(\Ac\left(\Pm\right)) }\left[c_i(t) \neq c_j(t)\right]$$
the smallest probability for $c_i(t)$ to be different from $c_j(t)$ while both items are in the same subset of the chosen partition $\Pm(t)$ (and may potentially be clicked upon). If we assume $1 \succ 2  \succ \dots \succ L$, we also denote
$$\tilde{\delta}_{i}^* \defeq
\min_{\substack{\Pm\in\Nc(\left(\{1\},\dots,\{K\}, \{{K+1},\dots,L\}\right)): \exists c, (\min(i-1,K),i)\in P_c^2}}
\PP_{\av(t) \sim \Uc(\Ac\left(\Pm\right)) }\left[c_{\min(i-1,K)}(t) \neq c_i(t)\right]$$
the smallest probability for $c_{\min(i-1,K)}(t)$ to be different from $c_i(t)$ while both items $\min(i-1,K)$ and $i$ are in the same subset of the chosen partition $\Pm(t)$ (and may potentially be clicked upon), and $\Pm(t)$ is in the neighborhood of the optimal partition $\Pm^* = \left(\{1\},\dots,\{K\}, \{{K+1},\dots,L\}\right)$.

If $i\succ j$, we denote 
$$\tilde{\Delta}_{i,j} \defeq
\min_{\Pm: \exists c, (i,j)\in P_c^2 \land \recommended(P_c)}
\EE_{\av(t) \sim \Uc(\Ac\left(\Pm\right)) }\left[c_i(t)-c_j(t) \mid c_i(t) \neq c_j(t)\right]
= \min_{\av\in\perm_K^L:\{i,j\}\cap\av([K])\neq\varnothing}\tilde\Delta_{i,j}(\av),$$
the smallest expected difference of clicks between items $i$ and $j$ while both items are in the same subset of the chosen partition $\Pm(t)$ (and may potentially be clicked upon). 

Symmetrically, if $j\succ i$, we denote
$$\tilde{\Delta}_{i,j} \defeq
\max_{\Pm: \exists c, (i,j)\in P_c^2 \land \recommended(P_c)}
\EE_{\av(t) \sim \Uc(\Ac\left(\Pm\right)) }\left[c_i(t)-c_j(t) \mid c_i(t) \neq c_j(t)\right]
= \max_{\av\in\perm_K^L:\{i,j\}\cap\av([K])\neq\varnothing}\tilde\Delta_{i,j}(\av),$$
the greatest expected difference of clicks between items $i$ and $j$ while both items are in the same subset of the chosen partition $\Pm(t)$ (and may potentially be clicked upon). 

Lemma \ref{lem:min_s} in Appendix \ref{app:min_s} ensures the proper definition of these notations under Assumptions  \ref{asp:strict_weak_order}, \ref{asp:acyclicity}, and \ref{asp:identifiability}, and states that
$\tilde{\delta}_{i,j}=\tilde{\delta}_{j,i}>0$ and $\tilde{\Delta}_{i,j}=-\tilde{\Delta}_{j,i}>0$ if $i\succ j$.
\end{definition}

\begin{definition}[Reward gap]
Let $(L, K, \rho)$ be an OLR problem satisfying Assumption \ref{asp:acyclicity} and such that the order on items is a total order. Without loss of generality, let us assume that $1\succ 2 \succ \dots \succ L$.. Denoting $\Pm^* = \left(\{1\},\dots,\{K\}, \{{K+1},\dots,L\}\right)$ the optimal partition associated to this order and taking $c\geqslant 2$, the \emph{reward gap} of item $c$ is 
\begin{align*}
\Delta_c
&\defeq\rho\left(\av^*,\min(c-1,K)\right) + \rho\left(\av^*,c\right)\\
&\quad- \rho\left(({\min(c-1,K)},{c})\circ\av^*,\min(c-1,K)\right) - \rho\left(({\min(c-1,K)},{c})\circ\av^*,c\right)
\end{align*}

Note that for $c\leqslant K$,
$
\Delta_c\defeq\rho(\av^*,c-1) + \rho(\av^*,c) - \rho(({c-1},{c})\circ\av^*,c-1) - \rho(({c-1},{c})\circ\av^*,c),
$
and for $c\geqslant K+1$,
$
\Delta_c=\rho(\av^*,K) - \rho((K,c)\circ\av^*,K).
$
\end{definition}

\section{Algorithm for the Elicitation of the leader partition $\tilde\Pm(t)$}\label{app:algo_elicitation}

\begin{algorithm}[t!]
\caption{Elicitation of the leader partition $\tilde\Pm(t)$}\label{alg:leader_elicitation}
\begin{algorithmic}[1]
\REQUIRE number of items $L$, number of positions $K$,
iteration index $t$, statistics ${\hat s}_{i,j}(t)$
    \STATE $\tilde{d} \gets 1$; $R \gets [L]$; $n\gets L$
    \REPEAT
        \STATE \textbf{for each} $i\in R$, $S_i \gets |\left\{j \in R : {\hat s}_{i,j}(t) > 0 \right\}|$
        \STATE sort items in $R$ by $S_i$: $S_{i_1} > S_{i_2} > \dots > S_{i_n}$
        \STATE $\ell\gets\min\left\{\ell\in[n]: \forall k<\ell, \forall k'\geqslant\ell: {\hat s}_{i_k,i_{k'}}(t) > 0 \right\}$
        \STATE $\displaystyle B \gets \left\{i_1,\dots,i_{\ell-1}\right\}$;
         $\displaystyle \tilde{P}_{\tilde{d}}(t) \gets B$
        \STATE $\tilde{d} \gets \tilde{d}+1$; $R\gets R\setminus B$; $n\gets|R|$
    \UNTIL{$\left|\bigcup_{\tilde{c}=1}^{\tilde{d}}\tilde{P}_{\tilde{c}}(t)\right| \geqslant K$}
    \STATE $\tilde{d} \gets \tilde{d}+1$ ; $\displaystyle \tilde{P}_{\tilde{d}}(t) \gets R$
    \STATE \textbf{return} $\tilde\Pm(t)$
\end{algorithmic}
\end{algorithm}

\forlater{to be added: DCM}

\section{Proof of Lemma \ref{lem:pbm_and_cm_are_unimodal} (PBM and CM Fulfills Assumptions  \ref{asp:strict_weak_order}, \ref{asp:acyclicity}, and \ref{asp:identifiability}) }\label{app:pbm_and_cm_are_unimodal}

For both CM and PBM click models, we note $\theta_i$ the click probability of item $i$. For PBM we have $\kappa_k$ the  probability that a user see the position $k$. 

\begin{proof}
Let us begin with some preliminary remarks. 

First, with PBM model, the positions are ranked by decreasing observation probability, meaning that $\kappa_{a_1}\geqslant\kappa_{a_2}\geqslant\dots\geqslant\kappa_{a_K}$.

Secondly, by definition, $\rho(k,\av) > 0$ for any position $k$ and recommendation $\av$, which implies that:
\begin{itemize}
    \item $\min_i \theta_i > 0$ and $\max_i \theta_i < 1$  in CM model;
    \item $\kappa_K > 0$ in PBM model.
\end{itemize}

Let us now prove that Assumptions \ref{asp:strict_weak_order}, \ref{asp:acyclicity} and \ref{asp:identifiability} are fulfilled by PBM and CM click models with the strict weak order $\succ$ defined by $i \succ j \iff \theta_i > \theta_j$. 

By definition of $\succ$, \cref{asp:strict_weak_order} is fulfilled taking the the preferential attachment function $g: i \mapsto \theta_i$, and \cref{asp:strict_weak_order}$^*$ is fulfilled as soon as $\theta_i \neq \theta_j$ for any item $i$ in top-$K$ items and any item $j\neq i$.

For Assumption \ref{asp:acyclicity},
we have to prove that having $\av$ compatible with $\succ$ is optimal, meaning $\mu_\av = \mu^*$.

Let $\av$ be a permutation compatible with $\succ$.

In the case of CM, $\mu_{\av}= 1-\sum_{k=1}^{K}(1-\theta_{a_k})$. In order to maximize $\mu_{\av}$, one has to select the $K$ higher values of $\thetav$. As $\av$ is compatible with $\succ$, which is defined based on values $\theta_i$, it satisfies this property. Hence, CM fulfills Assumption \ref{asp:acyclicity}.

For PBM, $\mu_{\av}=\sum_{k=1}^{K}\theta_{a_k}\kappa_k$. As the series $(\kappa_k)_{k\in[K]}$ is non-increasing,  $\mu_\av$ is maximized if $(\theta_k)_{k\in[K]}$ is also non-increasing and if $\theta_K \geqslant \max_{k\geqslant K+1}\theta_k$. These properties are ensured by the fact that $\av$ is compatible with $\succ$ and that $\succ$ is defined based on values $\theta_i$. Hence, PBM fulfills Assumption \ref{asp:acyclicity}.

We now prove that CM and PBM fulfill Assumption \ref{asp:identifiability}. Let $i$ and $j$ be two distinct items such that $i\succ j$ and $\av\in\perm_K^L$ be a recommendation such that at least one of both items is displayed.

First, $\EE_{\av'\sim \Uc(\{\av, (i,j)\circ\av\}) }\left[c_i(t)\neq c_j(t) \mid \av(t)=\av'\right]$ is non-null with PBM model as $c_i(t)$ and $c_j(t)$ are independent and as at least one of the four variables $c_i(t) \mid \av(t)=\av$, $c_i(t) \mid \av(t)=(i,j)\circ\av$, $c_j(t) \mid \av(t)=\av$, $c_j(t) \mid \av(t)=(i,j)\circ\av$ has an expectation which is non-zero and strictly smaller than 1 (due to $\kappa_K>0$ and $\theta_i>\theta_j$).

Similarly, $\EE_{\av'\sim \Uc(\{\av, (i,j)\circ\av\}) }\left[c_i(t)\neq c_j(t) \mid \av(t)=\av'\right]$ is non-null with CM model as at most one of both items can be clicked at each iteration and the shown item has non-zero probability to be clicked (by definition of $\rho$).

Then, we consider $\tilde\Delta_{i,j}(\av)$ as

$$
\tilde\Delta_{i,j}(\av) = \frac{\PP_{\av'\sim \Uc(\{\av, (i,j)\circ\av\}) }(c_i=1,c_j=0)-\PP_{\av'\sim \Uc(\{\av, (i,j)\circ\av\}) }(c_i=0,c_j=1)}{\PP_{\av'\sim \Uc(\{\av, (i,j)\circ\av\}) }(c_i=1,c_j=0)+\PP_{\av'\sim \Uc(\{\av, (i,j)\circ\av\}) }(c_i=0,c_j=1)}
$$

We want to control the sign of $\tilde\Delta_{i,j}(\av)$, which is also the sign of its numerator, as its denominator (noted $D_{\tilde\Delta_{i,j}(\av)}$) is non-negative.

The recommendation $\av'$ is drawn uniformly in $\{\av, (i,j)\circ\av\}$ thus  $$\PP_{\av'\sim \Uc(\{\av, (i,j)\circ\av\}) }(c_i=1,c_j=0)=\frac{1}{2}\PP_{\av}(c_i=1,c_j=0) +\frac{1}{2}\PP_{(i,j)\circ\av }(c_i=1,c_j=0).
$$

When considering a CM click model, we have $\PP_{\av}(c_i=1,c_j=0)= \prod_{p=1}^{k-1}(1-\theta_{a_p})\theta_i$ and  $\PP_{\av}(c_i=0,c_j=1)= \prod_{p=1}^{l-1}(1-\theta_{a_p})\theta_j$ when i and j  $\in \av$.

In that case, we have: 

$$
\tilde\Delta_{i,j}(\av) = \frac{
\frac{1}{2}\prod_{p=1}^{k-1}(1-\theta_{a_p})\theta_i +\frac{1}{2}\prod_{p=1}^{l-1}(1-\theta_{a_p})\theta_i
-\left(\frac{1}{2}\prod_{p=1}^{l-1}(1-\theta_{a_p})\theta_j +\frac{1}{2}\prod_{p=1}^{k-1}(1-\theta_{a_p})\theta_j\right)
}{D_{\tilde\Delta_{i,j}(\av)}}
$$

which can be simplified in: 
$$
\tilde\Delta_{i,j}(\av) = \frac{\frac{1}{2}\left(\prod_{p=1}^{k-1}(1-\theta_{a_p})+\prod_{p=1}^{l-1}(1-\theta_{a_p})\right)(\theta_i-\theta_j) }{D_{\tilde\Delta_{i,j}(\av)}}.
$$

Since $\max_i \theta_i < 1 $,  $\prod_{p=1}^{k-1}(1-\theta_{a_p})+\prod_{p=1}^{l-1}(1-\theta_{a_p})>0 $, thus the sign of $\tilde\Delta_{i,j}(\av)$ is the sign of $(\theta_i-\theta_j)$ and  $\tilde\Delta_{i,j}(\av)>0 \iff \theta_i > \theta_j \iff i \succ j $.

Now if $i\notin \av $ then $\PP_{\av}(c_i=1,c_j=0)= 0$ as the position is not seen. We have: 

$$
\tilde\Delta_{i,j}(\av) = \frac{\frac{1}{2}(\prod_{p=1}^{l-1}(1-\theta_{a_p}))(\theta_i-\theta_j) }{D_{\tilde\Delta_{i,j}(\av)}}
$$
which leads to the same conclusion as the previous case.
By symmetry, we have the same conclusion with j $\notin \av $.

Now with a PBM click model, we have $\PP_{\av}(c_i=1,c_j=0)= \kappa_k\theta_i (1-\kappa_l\theta_j)$ as $c_i=1$ and $c_j=0$ are independant events.

Thus, we have: 

$$
\tilde\Delta_{i,j}(\av) = \frac{\frac{1}{2}\kappa_k\theta_i (1-\kappa_l\theta_j) +\frac{1}{2}\kappa_l\theta_i (1-\kappa_k\theta_j)
-\left(\frac{1}{2}\kappa_l\theta_j (1-\kappa_k\theta_i) +\frac{1}{2}\kappa_k\theta_j (1-\kappa_l\theta_i)\right) }{D_{\tilde\Delta_{i,j}(\av)}}
$$

which can be simplified in: 
$$
\tilde\Delta_{i,j}(\av) = \frac{\frac{1}{2}(\kappa_k+\kappa_l)(\theta_i-\theta_j) }{D_{\tilde\Delta_{i,j}(\av)}}
$$

As $\kappa_k$ or $\kappa_l$ is positive if $i$ or $j$ is presented, similarly to the CM case we have $\tilde\Delta_{i,j}(\av)>0 \iff \theta_i > \theta_j \iff i \succ j $. 

This proof can be extended to $i$ or $j$ $\notin \av$ by taking $\kappa_k = 0 $ when $k>K$.

We can conclude that both CM and PBM fulfills Assumption \ref{asp:identifiability}.
\end{proof}

\forlater{
\section{One Iteration of \ouralgo{}, a Concrete Example}\label{app:example_iteration}

\writerg{
Let us demonstrate the behaviour of \ouralgo{} for one iteration. We consider the setting with $L=7$ items to be displayed at $K=4$ positions. Let us assume $1 \succ 2 \succ 3 \succ 4 \succ 5 \succ 6 \succ 7$ which is unknown from the algorithm.

We consider an iteration $t$ such that the matrix $\hat{\Sm} \defeq\left[\hat{s}_{i,j}(t)\right]_{(i,j)\in [L]^2}$ and the corresponding slight-optimistic estimate matrix $\tilde{\Sm} \defeq\left[\tilde{s}_{i,j}(t)\right]_{(i,j)\in [L]^2}$ are
$$
\hat{\Sm} = \left[
\begin{array}{ccccccc}
    \nullCell & + & + & - & + & + & + \\
    - & \nullCell & + & + & - & + & + \\
    - & - & \nullCell & + & + & + & - \\
    + & - & - & \nullCell & + & + & + \\
    - & + & - & - & \nullCell & + & + \\
    - & - & - & - & - & \nullCell & - \\
    - & - & + & - & - & + & \nullCell
\end{array}
\right]
\text{ and }
\tilde{\Sm} = \left[
\begin{array}{ccccccc}
    \nullCell & - & + & - & + & + & + \\
    - & \nullCell & + & + & - & + & + \\
    - & - & \nullCell & + & + & + & - \\
    - & - & - & \nullCell & - & + & + \\
    - & - & - & - & \nullCell & + & + \\
    - & - & - & - & - & \nullCell & - \\
    - & - & - & - & - & + & \nullCell
\end{array}
\right],
$$
where the $+$ and $-$ symbols indicate the sign of $\hat{s}_{i,j}(t)$.

Note that, due to the order on items, the expectation of $\hat{\Sm}$ is a matrix with non-negative values in the upper-triangular part, and non-positive values in the lower-triangular part. Being only one realisation, $\hat{\Sm}$ does not fulfill this pattern but a close one.

Note also that $\hat{\Sm}$ is symmetric as $\hat{s}_{i,j}(t) = -\hat{s}_{j,i}(t)$ by definition of $\hat{s}_{i,j}(t)$. This isn't true for  $\tilde{\Sm}$, however $\tilde{\Sm} \leqslant \hat{\Sm}$ by definition.

\paragraph{Leader-Partition Elicitation}
The Leader-partition $\tilde{\Pm}(t)$ is build recursively from sub-parts of $\tilde{\Sm}$.

First, only the columns 1 and 2 of $\tilde{\Sm}$ are composed of negative values. So $\tilde{P}_1(t)=\{1,2\}$.

Secondly, we restrict ourselves to remaining items, meaning we consider the matrix 
$$
\tilde{\Sm}_{3:L,3:L} = \left[
\begin{array}{ccccc}
    \nullCell & + & + & + & - \\
    - & \nullCell & - & + & + \\
    - & - & \nullCell & + & + \\
    - & - & - & \nullCell & - \\
    - & - & - & + & \nullCell
\end{array}
\right],
$$
for which, only the first column (corresponding to item 3) is composed of negative values. So $\tilde{P}_2(t)=\{3\}$.

Thirdly, we restrict ourselves to remaining items, meaning we consider the matrix 
$$
\tilde{\Sm}_{4:L,4:L} = \left[
\begin{array}{cccc}
    \nullCell & - & + & + \\
    - & \nullCell & + & + \\
    - & - & \nullCell & - \\
    - & - & + & \nullCell
\end{array}
\right],
$$
for which, only columns 1 and 2 (corresponding to items 4 and 5) is composed of negative values. So $\tilde{P}_3(t)=\{4,5\}$.

As the three first subsets of $\tilde{\Pm}(t)$ contains $5\geqslant K$ items, the remaining items are put in the last subset $\tilde{P}_4(t)=\{6,7\}$.

Overall, $\tilde{d}=4$ and $\tilde{\Pm}(t) = \left(\{1,2\}, \{3\}, \{4,5\}, \{6,7\}\right)$.

\paragraph{Optimistic Partition Elicitation (Example 1)}
The "played" partition $\Pm(t)$ is chosen in the neighborhood of $\tilde\Pm(t)$, meaning in the set of partitions obtained by merging consecutive subsets of $\tilde\Pm(t)$ and picking one item in $\tilde{P}_{\tilde d}(t)$. Therefore, the neighborhood of $\Pm(t)$ is
$$
\Nc\left(\tilde\Pm(t)\right) =
\left.\begin{cases}
\left(\{1,2\}, \{3\}, \{4,5\}\right),\\
\left(\{1,2, 3\}, \{4,5\}\right),\\
\left(\{1,2, 3\}, \{4,5,6\}\right),\\
\left(\{1,2, 3\}, \{4,5,7\}\right),\\
\left(\{1,2\}, \{3,4,5\}\right),\\
\left(\{1,2\}, \{3\}, \{4,5, 6\}\right),\\
\left(\{1,2\}, \{3\}, \{4,5,7\}\right)
\end{cases}\right\}.
$$

The merging process is done iteratively, considering subsets from the first to the last one and looking at the KL-indices indicating whether each item in a subset $\tilde{P}_c(t)$ are clearly more attractive than the items in the next subset $\tilde{P}_{c+1}(t)$.

Let us assume that the KL-indices matrix $\ubar{\Sm} \defeq\left[\bar{\bar s}_{j,i}(t)\right]_{(i,j)\in [L]^2}$ is 
$$
\ubar{\Sm} = \left[
\begin{array}{ccccccc}
    \nullCell & - & + & - & + & + & + \\
    - & \nullCell & + & + & - & + & + \\
    - & - & \nullCell & - & + & + & - \\
    - & - & - & \nullCell & - & + & - \\
    - & - & - & - & \nullCell & - & + \\
    - & - & - & - & - & \nullCell & - \\
    - & - & - & - & - & + & \nullCell
\end{array}
\right].
$$

Therefore, we first look at the submatrix $
\ubar{\Sm}_{1:2,3} = \left[
\begin{array}{c}
    +\\
    +
\end{array}
\right],
$
which contains only positive values, meaning that items 1 and 2 should be considered more attractive than item 3, and therefore subsets $\tilde{P}_1(t)=\{1,2\}$ and $\tilde{P}_2(t)=\{3\}$ are not merged.

We then look at  the submatrix $
\ubar{\Sm}_{3,4:5} = \left[
\begin{array}{cc}
    -&+
\end{array}
\right],
$
which contains a non-positive value. Therefore subsets $\tilde{P}_2(t)=\{3\}$ and $\tilde{P}_3(t)=\{4,5\}$ are merged.

Finally, as the subset $\tilde{P}_3(t)$ as been merged, we do consider adding an item from subset $\tilde{P}_4(t)$. The "played" subset is $\Pm(t) = \left(\{1,2\}, \{3,4,5\}\right)$ and the recommendation is drawn at random from the set of permutations 
$$
\Ac\left(\Pm(t)\right) =
\left.\begin{cases}
\left(1,2,3,4,5\right),
\left(1,2,3,5,4\right),
\left(1,2,4,3,5\right),\\
\left(1,2,4,5,3\right),
\left(1,2,5,3,4\right),
\left(1,2,5,4,3\right),\\
\left(2,1,3,4,5\right),
\left(2,1,3,5,4\right),
\left(2,1,4,3,5\right),\\
\left(2,1,4,5,3\right),
\left(2,1,5,3,4\right),
\left(2,1,5,4,3\right)
\end{cases}\right\}.
$$

\paragraph{Optimistic Partition Elicitation (Example 1)}
Let us now assume another KL-indices matrix $\ubar{\Sm}$: 
$$
\ubar{\Sm} = \left[
\begin{array}{ccccccc}
    \nullCell & - & + & - & + & + & + \\
    - & \nullCell & - & + & - & + & + \\
    - & - & \nullCell & - & + & + & - \\
    - & - & - & \nullCell & - & + & - \\
    - & - & - & - & \nullCell & - - & + \\
    - & - & - & - & - & \nullCell & - \\
    - & - & - & - & - & + & \nullCell
\end{array}
\right],
$$
where the $- -$ symbol means that $\bar{\bar s}_{5,6}(t) < \bar{\bar s}_{4,7}(t)$.

Therefore, we first look at the submatrix $
\ubar{\Sm}_{1:2,3} = \left[
\begin{array}{c}
    +\\
    -
\end{array}
\right],
$
which a non-positive value. Therefore the subsets $\tilde{P}_1(t)=\{1,2\}$ and $\tilde{P}_2(t)=\{3\}$ are merged.

Then, as the subset $\tilde{P}_2(t)$ is already merged we look at the next subset: $\tilde{P}_3(t)$. The subsets from $\tilde{P}_1(t)$ to $\tilde{P}_3(t)$ contain $5\geqslant K$ items so we do not consider merging $\tilde{P}_3(t)$ and $\tilde{P}_4(t)$, but adding to $\tilde{P}_3(t)$ and item from $\tilde{P}_4(t)$. We look at  the submatrix $
\ubar{\Sm}_{4:5, 6:7} = \left[
\begin{array}{cc}
    +&-\\
    - -&+
\end{array}
\right],
$
which contains a non-positive value. Therefore we add to $\tilde{P}_3(t)$ the item with the smallest value: 6.

Finally, the "played" subset is $\Pm(t) = \left(\{1,2,3\}, \{4,5,6\}\right)$ and the recommendation is drawn at random from the set of 36 permutations 
$$
\Ac\left(\Pm(t)\right) =
\left.\begin{cases}
\left(1,2,3,4,5,6\right),
\left(1,2,3,4,6,5\right),
\dots,
\left(1,2,3,6,5,4\right),\\
\left(1,3,2,4,5,6\right),
\left(1,3,2,4,6,5\right),
\dots,
\left(1,3,2,6,5,4\right),\\
\dots\\
\left(3,2,1,4,5,6\right),
\left(3,2,1,4,6,5\right),
\dots,
\left(3,2,1,6,5,4\right)
\end{cases}\right\}.
$$
}
}
\section{Technical Lemmas Required by the Proof of Theorem \ref{theo:unirank}}\label{app:technical_lemmas}

In this section, we gather technical Lemmas required to prove the regret upper-bound of \ouralgo{}. These lemmas regard
the concentration away from zero of the statistic $\hat{s}_{i,j}(t)$ (Appendices \ref{app:min_s} and \ref{app:concentration}), and the sufficient optimism brought by $\bar{\bar s}_{j,i}(t)$ (Appendix \ref{app:sufficient_optimism}).

\forlater{ !!! On n'en a plus besoin maintenant que le découpage est "trivial".
\subsection{Proof of Lemma \ref{lem:unimodality} (Pseudo-Unimodality Assuming a Total Order on Top-$K$ Items) }\label{app:unimodality}

\begin{proof}
To ease the notations, we take the following order on items:  $1 \succ 2 \succ \dots \succ K \succ [L]\setminus[K]$.  Therefore, $\Pm^* = \left(\{1\}, \dots,\{K\},[L]\setminus[K]\right)$. 
Therefore, if $\tilde\Pm \neq \Pm^*$ 
\begin{itemize}
    \item either  $\exists c \in [\tilde{d}]$, such that $|P_c|>1$ and  $i^*\succ \argmax_{j\in P_c\setminus\{i^*\}} g(j)$, where $i^*=\argmax_{i\in P_c} g(i)$;
    \item or $\exists c \in [\tilde{d}-1]$, $\exists (i,j)\in\tilde{P}_c\times\tilde{P}_{c+1}$, such that $j \succ i$.
\end{itemize}

Let us show that this second alternative is divided into the two last outputs of Lemma \ref{lem:unimodality}. Let $c\in[K]$ be the smallest index such that $\tilde{P}_c = \{i\}$ and there exists $j\in\tilde{P}_{c+1}$ such that $j \succ i$. Either $c>1$, and therefore $\tilde{P}_{c-1} = \{i'\}$ and $i' \succ i$, or $c=1$.
\end{proof}
}

\subsection{Minimum Expected Click Difference}\label{app:min_s}

Assumption \ref{asp:identifiability} builds upon $\tilde\Delta_{i,j}(\av)$ which measures the difference of attractiveness between $i$ and $j$ while all other items are at fixed positions. In the theoretical analysis of \ouralgo{}, we handle situations where other items may also change in position thanks to the following Lemma.


\begin{lemma}[Minimum expected click difference]\label{lem:min_s}
Let $(L, K, \rho)$ be an OLR problem satisfying Assumptions \ref{asp:acyclicity} and \ref{asp:identifiability} with $\succ$ the order on items, and let $i$ and $j$ be two items such that $i\succ j$. Then, for any partition of items $\Pm$, if there exists $c$ such that $(i,j)\in P_c^2$ and
$\EE_{\av(t) \sim \Uc(\Ac\left(\Pm\right)) }\left[c_i(t) \neq c_j(t)\right] \neq 0,$
then 
$\EE_{\av(t) \sim \Uc(\Ac\left(\Pm\right)) }\left[c_i(t)-c_j(t) \mid c_i(t) \neq c_j(t)\right] > 0$
and therefore
\begin{align*}
    \tilde{\delta}_{i,j} &> 0
    &&{ and }
    &\tilde{\Delta}_{i,j} &> 0.
\end{align*}

Symmetrically, if $j\succ i$, for any partition of items $\Pm$, if there exists $c$ such that $(i,j)\in P_c^2$ and
$\EE_{\av(t) \sim \Uc(\Ac\left(\Pm\right)) }\left[c_i(t) \neq c_j(t)\right] \neq 0,$
then 
$\EE_{\av(t) \sim \Uc(\Ac\left(\Pm\right)) }\left[c_i(t)-c_j(t) \mid c_i(t) \neq c_j(t)\right] < 0$
and therefore
\begin{align*}
    \tilde{\delta}_{i,j} &> 0
    &&{ and }
    &\tilde{\Delta}_{i,j} &< 0.
\end{align*}
\end{lemma}

\begin{proof}
The proof consists in writing
$\EE_{\av(t) \sim \Uc(\Ac\left(\Pm\right)) }\left[c_i(t) \neq c_j(t)\right] \neq 0$
two times as a sum other $\av(t) \in \Uc(\Ac\left(\Pm\right))$, and in reindexing one of both sums by $(i,j)\circ\av(t) \in \Uc(\Ac\left(\Pm\right))$. Then, adding the terms of both sums we get a sum of terms $\tilde\Delta_{i,j}(\av)$ which by assumption \ref{asp:identifiability} are positive. Hence this sum is positive, which concludes the proof. 
\end{proof}


\subsection{Upper-bound on the Number of High Deviations for Variables with Lower-Bounded Mean}\label{app:concentration}

The Proof of Theorem  \ref{theo:unirank} requires the control of the expected number of high deviations of the statistic $\hat{s}_{i,j}(t)$. We control this expectation through Lemma \ref{lem:dev_s} which derives from the application of Lemmas  \ref{lem:concentrationOne} and \ref{lem:concentration} to $\hat{s}_{i,j}(t)$ and $\hat{T}_{i,j}(t)$. Hereafter, we express and prove the three lemmas. Note that Lemmas \ref{lem:concentrationOne} and \ref{lem:concentration} are extensions of  Lemmas $4.3$ and $B.1$ of \cite{Combes2014} to a setting where the handled statistic is a mixture of variables following different laws of bounded expectation.

\begin{lemma}[Concentration bound with lower-bounded mean]\label{lem:concentrationOne}
Let $(X^a_t)_{t\geqslant 1}$ with $a \in \mathcal{R} $, be  $|\mathcal{R}|<\infty$ independent sequences of independent random variables bounded  in $[0,B]$ defined on a probability space $(\Omega,\mathcal{F},\PP)$.
Let $\mathcal{F}_t$ be an increasing sequence of $\sigma-$fields of $\mathcal{F}$ such that for each t, $\sigma((X^{a}_1)_{a\in\mathcal{R}}, \dots,(X^{a}_t)_{a\in\mathcal{R}})\subset{\mathcal{F}_t}$
and for $s > t$ and $a$ a recommendation, $X^a_s$ is independent from $\mathcal{F}_t$. Consider $|\mathcal{R}|$ previsible sequences $(\epsilon^{a}_t)_{t\geq1}$ of Bernoulli variables (for all $t>0$, $\epsilon^a_t$ is $\mathcal{F}_{t-1}-mesurable$) such that for all $t>0$, $\sum_i \epsilon^{a}_t \in \{0,1\}$. Let $\delta > 0$ and for every $t \in \{1, \dots, n\}$ let 
$$S(t)=\sum^t_{s=1} \sum_i \epsilon^i_s(X^i_s-\EE[X^i_s]), \qquad  T(t)=\sum^t_{s=1} \sum_i \epsilon^i_s, \qquad \hat{\mu}(t) =\frac{S(t)}{N(t)}.$$

Define $\phi\in\{t_0,\dots,T+1\}$ a $\Fc$-stopping time such that either $T(\phi)\geqslant s$ or $\phi=T+1$.

Then $$\PP\left(S(\phi) \geqslant T(\phi)\delta, \phi\leqslant T\right) \leqslant \exp(-\frac{2n\delta^2}{B^2}).$$
\end{lemma}

\begin{proof}
Let $\lambda>0$, and define $G_t=\exp(\lambda(S(t)-\delta T(t)))\ind\{t\leqslant T\}$. We have that:
\begin{align*}
\PP(S(\phi) \geqslant T(\phi)\delta, \phi\leqslant T)
& = \PP(\exp(\lambda(S(\phi)-\delta T(\phi))\ind\{\phi\leqslant T\} \geqslant 1)
\\&=\PP(g_\phi \geqslant 1)
\\&\leqslant \EE[G_\phi].
\end{align*}

Next we provide an upper bound for $\EE[G_\phi]$. We define the following quantities:

$$Y_s^i = \varepsilon_s^i(\lambda(X_s^i-\EE[X_s^i])-\lambda^2B^2/8)$$

$$\tilde{G}_t=\exp\left(\sum_{s=1}^t\sum_iY_s^i\right)\ind\{t\leqslant T\}.$$

Taking $\lambda=4\delta/B^2$, $G_t$ can be written:
$$G_t=\tilde{G}_t\exp(-T(t)(\lambda\delta-\lambda^2B^2/8)
=\tilde{G}_t\exp(-2T(t)\delta^2/B^2).
$$

As $T(t)\geqslant n$ if $\phi\leqslant T$ we can upper bound $G_\phi$ by:

$$G_\phi = \tilde{G}_\phi\exp(-2T(\phi)\delta^2/B^2)
\leqslant \tilde{G}_\phi\exp(-2n\delta^2/B^2).$$

It is noted that the above inequality holds even when $\phi=T+1$, since $G_{T+1} = \tilde{G}_{T+1}=0$. Hence:
$$\EE[G_\phi] \leqslant \EE[\tilde{G}_\phi]\exp(-2n\delta^2/B^2)$$

We prove that $\left(\tilde{G}_t\right)_t$ is a super-martingale. We have that $\EE[\tilde{G}_{T+1}\mid\Fc_T] = 0 \leqslant \tilde{G}_T$. For $s\leqslant T-1$, since $B_{t+1}$ is $\Fc$ measurable:
$$
\EE[\tilde{G}_{t+1}\mid\Fc_t] = \tilde{G}_t((1-\sum_i\varepsilon_{t+1}^i)+\sum_i\varepsilon_{t+1}^i\EE[\exp(Y_{t+1}^i)]).
$$

As proven in (Hoeffding, 1963)[eq. 4.16] since $X_{t+1}^i\in[0,B]$:

$$\EE[\exp(\lambda(X_{t+1}^i-\EE[X_{t+1}^i])] \leqslant \exp(\lambda^2B^2/8),$$
so $\EE[\exp(Y_{t+1}^i)]\leqslant 1$ and $\left(\tilde{G}_t\right)_t$ is a super-martingale: $\EE[\tilde{G}_{t+1}\mid\Fc_t] \leqslant \tilde{G}_t$. Since $\phi\leqslant T+1$ almost surely, and  $\left(\tilde{G}_t\right)_t$ is a supermartingale, Doob's optional stopping theorem yields: $\EE[\tilde{G}_\phi] \leqslant \EE[\tilde{G}_0] = 1$, and so

\begin{align*}
\PP(S(\phi)\geqslant T(\phi)\delta, \phi\leqslant T)
&\leqslant \EE[G_\phi]
\\&\leqslant \EE[\tilde{G}_\phi]\exp(-2n\delta^2/B^2)
\\&\leqslant\exp(-2n\delta^2/B^2),
\end{align*}
which concludes the proof
\end{proof}

\begin{lemma}[Expected number of large deviation with lower-bounded mean]\label{lem:concentration}
Let $(L, K, \rho)$ be an OLR problem, $\Fc_t$ the natural $\sigma$-algebra generated by the OLR problem, and $\Fc=(\Fc_t)_{t\in \ZZ}$ the corresponding filtration. We denote $O_t\defeq(\av(1),\cv(1), \dots,\av(t-1),\cv(t-1))$ the set of random values observed up to time $t-1$.
Let $Z_t\in[0,B]$ and $B_t\in\{0,1\}$ be two $\Fc_{t-1}$-measurable random variables, $\Lambda\subseteq\NN$ be a random set of instants, and $\varepsilon>0$.
For any $t\in\ZZ$, we denote $S(t)\defeq \sum_{s=0}^{t}B_sZ_s$ and $T(t)\defeq \sum_{s=0}^{t}B_s$.
If
for any $t>0$, $\EE\left[Z_t\mid O_t, B_t=1\right]\geqslant \delta$
and there exists a sequence of random sets $(\Lambda(n))_{n>0}$ such that (i) $\Lambda \subseteq \bigcup_{n>0}\Lambda(n)$, (ii) for all $n>0$ and all $t\in\Lambda(n)$, $T(t)\geqslant\varepsilon n$, (iii) $|\Lambda(n)| \leqslant 1$, and (iv) the event $t\in\Lambda(n)$ is $\Fc$-measurable. Then

$$\EE\left[\sum_{t\geq1} \ind\{t \in \Lambda:  S(t) < \frac{\delta}{2}T(t) \}\right]
 \leq \frac{2B^2}{\epsilon\delta^2}$$
\end{lemma}

\begin{proof}
Let $T\in\NN$. For all $n\in\NN$, $|\Lambda(n)| \leqslant 1$, we define $\Phi_n$ as $T+1$ if $\Lambda(n)\cap [T]$ is empty and $\{\Phi_n\}=\Lambda(n)$ otherwise. Since $\Lambda \subseteq \bigcup_{n>0}\Lambda(n)$, we have

$$
\sum_{t=1}^T\ind\left\{t\in\Lambda:  S(t) < \frac{\delta}{2}T(t)\right\}
\leqslant \sum_{n\geqslant 1}\ind\left\{S(\Phi_n) < \frac{\delta}{2}T(\Phi_n), \Phi_n\leqslant T\right\}.
$$

Taking expectations,
$$
\EE\left[\sum_{t=1}^T\ind\left\{t\in\Lambda:  S(t) < \frac{\delta}{2}T(t)\right\}\right]
\leqslant \sum_{n\geqslant 1}\PP\left[S(\Phi_n) < \frac{\delta}{2}T(\Phi_n), \Phi_n\leqslant T\right]
$$

For any $t\in\NN$, denote $S'(t)\defeq \sum_{s=0}^{t}B_s(Z_s-\EE\left[Z_s\mid 0_s, B_s=1\right])$. As for any $s\in\NN$, $\EE\left[Z_s\mid 0_s, B_s=1\right]> \delta$, $S'(t)<S(t)-T(t)\delta$. Therefore, for any $n\in\NN$
$$\PP\left[S(\Phi_n) < \frac{\delta}{2}T(\Phi_n), \Phi_n\leqslant T\right]
\leqslant \PP\left[S'(\Phi_n) < -\frac{\delta}{2}T(\Phi_n), \Phi_n\leqslant T\right]$$
and
$$
\EE\left[\sum_{t=1}^T\ind\left\{t\in\Lambda:  S(t) < \frac{\delta}{2}T(t)\right\}\right]
\leqslant \sum_{n\geqslant 1}\PP\left[S'(\Phi_n) < -\frac{\delta}{2}T(\Phi_n), \Phi_n\leqslant T\right]
$$

By Lemma \ref{lem:concentrationOne},
since $\Phi_n$ is a stopping time upper bounded by $T+1$, and  $T(\Phi_n)\geqslant\varepsilon n$,

$$
\EE\left[\sum_{t=1}^T\ind\left\{t\in\Lambda:  S(t) < \frac{\delta}{2}T(t)\right\}\right]
\leqslant \sum_{n\geqslant 1}\exp\left(-\frac{\varepsilon n\delta^2}{2B^2}\right)
\leqslant \frac{2B^2}{\varepsilon\delta^2},
$$

where the last inequality drives from the $\sum_{n\geqslant 1}\exp\left(-nw\right)
\leqslant \int_0^{+\infty}\exp\left(-uw\right)du=\frac{1}{w}$.

This upper-bound is valid for any $T$, which concludes the proof.
\end{proof}


\begin{lemma}[Expected number of large deviation for our statistics]\label{lem:dev_s}
Let $(L, K, \rho)$ be an OLR problem satisfying Assumptions \ref{asp:acyclicity} and \ref{asp:identifiability} with $\succ$ the order on items, and let $i$ and $j$ be two items.
If there exists a sequence of random sets $(\Lambda(n))_{n>0}$ such that (i) $\Lambda \subseteq \bigcup_{n>0}\Lambda(n)$, (ii) for all $n>0$ and all $t\in\Lambda(n)$, $t_{i,j}(t+1)\geqslant\varepsilon n$, (iii) $|\Lambda(n)| \leqslant 1$, and (iv) the event $t\in\Lambda(n)$ is $\Fc$-measurable. Then, 
\begin{equation}\label{eq:bound_t}
\EE\left[\sum_{t\geq1} \ind\left\{t\in\Lambda, T_{i,j}(t) < \frac{\tilde{\delta}_{i,j}}{2} t_{i,j}(t)\right\}\right] = \OO(1)
\end{equation}
and 
$$\EE\left[\sum_{t\geq1} \ind\left\{t\in\Lambda, \frac{\hat{s}_{i,j}(t)}{\tilde{\Delta}_{i,j}} < \frac{1}{2} \right\}\right] = \OO(1),$$
meaning
\begin{align}
\label{eq:bound_s}
\EE\left[\sum_{t\geq1} \ind\left\{t\in\Lambda, \hat{s}_{i,j}(t) < \frac{\tilde{\Delta}_{i,j}}{2} \right\}\right] &= \OO(1)
&, \text{if }i\succ j;\\
\EE\left[\sum_{t\geq1} \ind\left\{t\in\Lambda, \hat{s}_{i,j}(t) > \frac{\tilde{\Delta}_{i,j}}{2} \right\}\right] &= \OO(1)
&, \text{if }j\succ i.
\end{align}
\end{lemma}

\begin{proof}
Let assume $i\succ j$. We first prove Claim \eqref{eq:bound_t} and then prove Claim \eqref{eq:bound_s} using Claim \eqref{eq:bound_t}.

For any $t\leqslant1$, we define both following $\Fc_{t-1}$-measurable random variables
\begin{align*}
    Z_t&\defeq\ind\left\{c_i(t) \neq c_j(t)\right\}
    &
    B_t&\defeq\ind\left\{\exists c, (i,j)\in P_c(t)^2\right\},
\end{align*} 
and we denote $O_t\defeq(\av(1),\cv(1), \dots,\av(t-1),\cv(t-1))$ the set of random values observed up to time $s-1$.
Note that $T_{i,j}(t+1)=\sum_{s = 1}^{t} B_sZ_s$, $t_{i,j}(t+1)=\sum_{s = 1}^{t} B_s$,
and $\EE\left[Z_t\mid 0_t, B_t=1\right]> \tilde\delta_{i,j}$ by Lemma \ref{lem:min_s}.

Therefore by Lemma \ref{lem:concentration}
$$
\EE\left[\sum_{t\geq1} \ind\{t \in \Lambda:  T_{i,j}(t+1) < \frac{\tilde\delta_{i,j}}{2}t_{i,j}(t+1) \}\right]
 \leqslant \frac{2}{\epsilon\tilde\delta_{i,j}^2},
 $$
meaning
$$
\EE\left[\sum_{t\geq1} \ind\{t \in \Lambda:  T_{i,j}(t) < \frac{\tilde\delta_{i,j}}{2}t_{i,j}(t) \}\right]
 \leqslant 1 + \frac{2}{\epsilon\tilde\delta_{i,j}^2} =\OO(1),
 $$
which corresponds to Claim  \eqref{eq:bound_t}.

Let now prove Claim  \eqref{eq:bound_s} using the following decomposition
\begin{align*}
\EE\left[ \sum_{t=1}^T\ind\left\{t\in\Lambda, \hat{s}_{i,j}(t) < \frac{\tilde{\Delta}_{i,j}}{2}\right\}\right]
    & \leqslant \EE\left[ \sum_{t=1}^T\ind\left\{t\in\Lambda, \hat{s}_{i,j}(t) < \frac{\tilde{\Delta}_{i,j}}{2}
    T_{i,j}(t) < \frac{\tilde\delta_{i,j}}{2}t_{i,j}(t)\right\}\right]
    \\ & \quad
    +\EE\left[ \sum_{t=1}^T\ind\left\{t\in\Lambda, \hat{s}_{i,j}(t) < \frac{\tilde{\Delta}_{i,j}}{2},
    T_{i,j}(t) \geqslant \frac{\tilde\delta_{i,j}}{2}t_{i,j}(t)\right\}\right],
\end{align*}
Where the first right-hand side term is smaller than $\EE\left[\sum_{t\geq1} \ind\left\{t\in\Lambda, T_{i,j}(t) < \frac{\tilde{\delta}_{i,j}}{2} t_{i,j}(t)\right\}\right]$ and therefore is a $\OO(1)$. We control the second term by applying again Lemma \ref{lem:concentration}.

For any $t\leqslant1$, we define both following $\Fc_{t-1}$-measurable random variables
\begin{align*}
    Z_t&\defeq c_i(t) - c_j(t)
    &
    B_t&\defeq\ind\left\{\exists c, (i,j)\in P_c(t)^2, c_i(t) \neq c_j(t)\right\},
\end{align*} 
Note that $Z_t\in[-1,1]$, $\hat{s}_{i,j}(t+1)T_{i,j}(t+1)=\sum_{s = 1}^{t} B_sZ_s$, $T_{i,j}(t+1)=\sum_{s = 1}^{t} B_s$,
and $\EE\left[Z_t\mid 0_t, B_t=1\right]> \tilde\Delta_{i,j}$ by Lemma \ref{lem:min_s} as $i \succ j$.

We also define $A\defeq \Lambda\cap\left\{t\in\NN: T_{i,j}(t) \geqslant \frac{\tilde\delta_{i,j}}{2}t_{i,j}(t)\right\}$ and for any $n\in \NN$, $A(n)\defeq \Lambda(n)\cap\left\{t\in\NN: T_{i,j}(t) \geqslant \frac{\tilde\delta_{i,j}}{2}t_{i,j}(t)\right\}$.
Then, (i) as $\Lambda \subseteq \bigcup_{n>0}\Lambda(n)$, $A \subseteq \bigcup_{n>0}A(n)$, (ii) for all $n>0$ and all $t\in A(n)$, $T_{i,j}(t) \geqslant \frac{\tilde\delta_{i,j}}{2}t_{i,j}(t)\geqslant\frac{\tilde\delta_{i,j}}{2}\varepsilon n$, (iii) $|A(n)|\leqslant|\Lambda(n)| \leqslant 1$, and (iv) the event $t\in A(n)$ is $\Fc$-measurable. Therefore by Lemma \ref{lem:concentration}
$$
\EE\left[\sum_{t\geq1} \ind\{t \in A:  \hat{s}_{i,j}(t+1)T_{i,j}(t+1) < \frac{\tilde\Delta_{i,j}}{2}T_{i,j}(t+1) \}\right]
 \leqslant \frac{8}{\tilde\delta_{i,j}\varepsilon\tilde\Delta_{i,j}^2},
 $$
meaning
$$
\EE\left[\sum_{t\geq1} \ind\left\{\substack{t\in\Lambda, \hat{s}_{i,j}(t) < \frac{\tilde{\Delta}_{i,j}}{2},\\
    T_{i,j}(t) \geqslant \frac{\tilde\delta_{i,j}}{2}t_{i,j}(t)} \right\}\right]
 \leqslant 1 + \frac{8}{\tilde\delta_{i,j}\varepsilon\tilde\Delta_{i,j}^2} =\OO(1).
 $$

Overall,
$\EE\left[ \sum_{t=1}^T\ind\left\{\substack{t\in\Lambda, \hat{s}_{i,j}(t) < \frac{\tilde{\Delta}_{i,j}}{2}}\right\}\right] = \OO(1) + \OO(1) =\OO(1)$
which corresponds to Claim  \eqref{eq:bound_s}.

Other claims are proved symmetrically.
\end{proof}

\subsection{Upper-Bound on the Number of Lower-Estimations of an Optimistic Estimator}\label{app:sufficient_optimism}

This section presents two results aiming at upper-bounding the number of iterations at which $\tilde\Delta_{j,i}$ is lower-estimated by $\bar{\bar s}_{j,i}(t)$ if $j\succ i$. These new results are extensions of Lemma 9 and Theorem 10 of \cite{Garivier2011} to a setting where the handled statistic is a mixture of variables following different laws of bounded expectation.

\begin{lemma} \label{lem:9GravierCappe}
Let X be a random variable taking value in $[0,1]$ and let $\mu \leq \EE[X]$. then for all $\lambda < 0 $,
$$\EE[\exp(\lambda X)] \leq 1- \mu + \mu \exp(\lambda), $$

\end{lemma}

\begin{proof}
The function f : $[0,1] \xrightarrow{\RR}$ defined by $f(x) = \exp(\lambda x) - x(\exp(\lambda)-1)-1$ is convex and such that $f(0) = f(1)= 0$, hence $f(x) \leq 0 $ for all $ x \in [0,1]$. Consequently, 
$$\EE[\exp(\lambda X)] \leq  \EE [X (\exp(\lambda)-1)+1] = \EE[X](\exp(\lambda)-1) +1  $$

As $\lambda < 0 $ and $\mu \leq \EE[X]$, we have $\EE[X](\exp(\lambda)-1) \leq \mu(\exp(\lambda)-1) $ and 
$$\EE[\exp(\lambda X)] \leq \mu(\exp(\lambda)-1)+ 1 $$

\end{proof}

\begin{lemma}\label{theo:10GravierCappe}
Let $(X^a_t)_{t\geqslant 1}$ with $a \in \mathcal{R} $, be  $|\mathcal{R}|<\infty$ independent sequences of independent random variables bounded  in $[0,1]$ defined on a probability space $(\Omega,\mathcal{F},\PP)$ with common expectations $\mu^a = \EE[X^a_t]$ of minimal value $\mu = \min_{a\in\mathcal{R}}\mu^a$.
Let $\mathcal{F}_t$ be an increasing sequence of $\sigma-$fields of $\mathcal{F}$ such that for each t, $\sigma((X^{a}_1)_{a\in\mathcal{R}}, \dots,(X^{a}_t)_{a\in\mathcal{R}})\subset{\mathcal{F}_t}$
and for $s > t$ and $a$ a recommendation, $X^a_s$ is independent from $\mathcal{F}_t$. Consider $|\mathcal{R}|$ previsible sequences $(\epsilon^{a}_t)_{t\geq1}$ of Bernoulli variables (for all $t>0$, $\epsilon^a_t$ is $\mathcal{F}_{t-1}-mesurable$) such that for all $t>0$, $\sum_i \epsilon^{a}_t \in \{0,1\}$. Let $\delta > 0$ and for every $t$ let 
$$S(t)=\sum^t_{s=1} \sum_i \epsilon^i_sX^i_s, \qquad  N(t)=\sum^t_{s=1} \sum_i \epsilon^i_s, \qquad \hat{\mu}(t) =\frac{S(t)}{N(t)} $$
$$u(t) = \max \{ q >  \hat{\mu}(t) :N(t)d(\hat{\mu}(t),q) \leq \delta \}$$
Then 
$$\PP(u(t) < \mu) \leq e\lceil\delta \log(t)\rceil \exp(-\delta) $$
\end{lemma}

\begin{proof}

For every $\lambda < 0$, by Lemma \ref{lem:9GravierCappe}, it holds that $\log (\EE[\exp (\lambda X^a_1)]) \leq \log (1- \mu +\mu\exp(\lambda)) = \phi_\mu(\lambda)$ for all $a$.
Let $W_0^\lambda =1$ and for $t \geq 1$, 
$$W_t^{\lambda} = \exp(\lambda S(t) -N(t)\phi_\mu(\lambda)) $$

$(W_t^{\lambda})_{t \geq} 0$ is a super-martingale relative to $(\mathcal{F}_t)_{t \geq 0}$. In fact,

$$\EE[\exp(\lambda\{S(t+1)-S(t)\})|\mathcal{F}_t] = \EE[\exp(\lambda \sum_i \epsilon^i_{t+1} X^i_{t+1})|\mathcal{F}_t]$$

As $(X^i_t)_t$  are independent sequences, we can rewrite :
$$
\EE[\exp(\lambda\{S(t+1)-S(t)\})|\mathcal{F}_t]  = \prod_i\EE[\exp(\lambda \epsilon^i_{t+1} X^i_{t+1})|\mathcal{F}_t]
 =  \prod_i\exp(\epsilon^i_{t+1} \log(\EE [\exp(\lambda X^i_{t+1})|\mathcal{F}_t]))$$
 $$
 =  \exp( \sum_i \epsilon^i_{t+1} \log(\EE [\exp(\lambda X^i_1)|\mathcal{F}_t]))  \leq  \exp( \sum_i \epsilon^i_{t+1} \phi_\mu(\lambda)) 
 = \exp ( \{ N(t+1)- N(t) \} \phi_\mu(\lambda))
$$

which can be rewritten as

$$\EE[\exp(\lambda S(t+1) - N(t+1)\phi_\mu(\lambda))|\mathcal{F}_t] \leq \exp (\lambda S(t) - N(t)\phi_\mu(\lambda))$$

The rest of the proof follows \cite{Garivier2011}. Using the "peeling trick": the interval $\{1,\dots,t\}$ of possible values for $N(t)$ is divided into slices $\{t_{k-1}+1,\dots,t_k\}$ of geometrically increasing size. Each slice is treated independently. We assume that $\delta > 1$ and we construct the slicing as follow : $t_0 = 0$ and for $k \in \NN^*$, $t_k = \lfloor(1+\eta)^k\rfloor$, with $\eta = 1/(\delta-1)$.
Let $D = \lceil\frac{\log t}{\log 1+\eta} \rceil$ be the first interval such that $t_D \geq t$ and $A_k$ the event $\{t_{k-1} \leq N(t) \leq t_k\} \cap \{u(t) < \mu\}$ . We have :
$$\PP(u(t) < \mu ) \leq \PP(\bigcup^D_{k=1}A_k)\leq \sum^D_{k=1} \PP(A_k)$$ 
Note that by definition of $u(t)$, we have $u(t)< \mu$ if and only if $\hat\mu(t) < \mu$ and $N(t)d(\hat\mu(t),\mu) > \delta$.
Let s be the smallest integer such that $\delta/(s+1) \leq d(0,\mu)$.
If $N(t) \leq s$, then
$$N(t)d(\hat\mu,\mu) \leq sd(\hat\mu,\mu)
\underset{\text{as } \hat\mu \leq \mu}{\leq}sd(0,\mu)
\underset{\text{by definition of } s}{<} \delta. $$
Thus, we can't have $\hat\mu < \mu $ and $N(t)d(\hat\mu,\mu) > \delta $ and $\PP(u(t)< \mu)= 0$ . We have for all $k$ such that $t_k \leq s$, $\PP(A_k) = 0 $ and we have $ u(t) > \mu$ when $N(t) \in \{t_{k-1}+1,\dots, t_k\}$ and $t_k \leq s$.

Now lets see how $u(t)$ can be upper bounded by $\mu$ when $N(t)> s$.
For $k$ such that $t_k \geq s$, we note $\tilde{t}_{k-1} = \max\{t_{k-1},s\}$ and we take  $z<\mu $ such as $d(z,\mu)= \delta/(1+\eta)^k$ and $x \in ]0,\mu[$ such that $d(x,\mu)=\delta/N(t)$. We define $\lambda(x) = \log(x(1-\mu)) -log(\mu(1-x)) < 0$ so that we can rewrite $d(x,\mu)$ as $d(x,\mu) = \lambda(x)x-\phi_\mu(\lambda(x)) $.


\begin{itemize}
    \item with $N(t)> \tilde{t}_{k-1},$ we have $ d(z,\mu) = \frac{\delta}{(1+\mu)^k} \geq \frac{\delta}{(1+\mu)N(t)}$
    \item with $N(t) \leq t_k$, we have $d(\hat\mu(t), \mu) > \frac{\delta}{N(t)} > \frac{\delta}{(1+\eta)^k} = d(z,\mu)$. As $\hat\mu < \mu$, we have $\hat\mu(t) \leq z$
\end{itemize}

Hence on the event $\{ \tilde{t}_{k-1}<N(t) \leq t_k \} \cap \{\hat \mu (t) < \mu \} \cap \{d(\hat\mu(t), \mu)\} $ it holds that $\lambda(z)\hat\mu(t)-\phi_\mu(\lambda(z)) \geq \lambda(z)z-\phi_\mu(\lambda(z)) = d(z,\mu) \geq \frac{\delta}{(1+\eta)N(t)}  $

It leads to :

\begin{align*}
    \{\tilde{t}_{k-1}< N(t) \leq t_k\} \cap \{u(t) < \mu\} \subset{\{\lambda(z)\hat\mu(t)-\phi_\mu(\lambda(z)) \geq \frac{\delta}{(1+\eta)N(t)} \}} \\
    \subset{\{\lambda(z)S(t)-N(t)\phi_\mu(\lambda(z)) \geq \frac{\delta}{(1+\eta)}\}}\\
    \subset{\{W_n^\lambda(z) > \exp \left( \frac{\delta}{(1+\eta)}\right)\}}
\end{align*}

As $(W_t^\lambda)_{t\geq 0}$ is a supermartingale, $\EE[W^{\lambda(z)}_n] \leq \EE[W^{\lambda(z)}_n] = 1$, and the Markov inequality yields :
$$\PP(\{\tilde{t}_{k-1}< N(t) \leq t_k\} \cap \{u(t) < \mu\}) \leq  \PP\left( W_n^\lambda(z) > \exp \left( \frac{\delta}{(1+\eta)}\right)\right) \leq \exp \left( -\frac{\delta}{(1+\eta)}\right) $$

As $\eta = 1/(\delta-1)$, $D =\lceil\frac{\log n}{\log 1+\eta} \rceil $ and $\log(1+1/(\delta-1)) \geq 1/\delta$, we obtain : 

$$\PP(u(t) < \mu) \leq \left\lceil\frac{\log n}{\log \left(1+\frac{1}{\delta-1}\right)} \right\rceil \exp(-\delta+1)\leq e\lceil \delta\log(t)\rceil \exp(-\delta) $$

\end{proof}

\forlater{
\begin{lemma}[Upper-Bound on the Number of Lower-Estimations of $\bar{\bar s}_{j,i}(t)$]\label{lemma:pessimistic_bound_is_good}
Under the hypotheses of Theorem \ref{theo:unirank}, for any pair of items $(i,j)$ such that $j\succ i$, and any partition $\tilde\Pm$

$$\EE\left[\left|\left\{t \in [T]: \tilde\Pm(t)=\tilde\Pm, \frac{\bar{\bar s}_{j,i}(t)+1}{2} \geqslant \frac{\tilde{\Delta}_{i,j}+1}{2} \right\}\right|\right] = O(\log(\log(T))).$$
\end{lemma}

\begin{proof}

For any ordered partition $\Pm$ and permutation $\av$, let consider the random variables
$$
X^{\Pm,\av}_t \defeq c_i(s)- c_j(s)+\ind\left\{c_i(s)=c_j(s)\right\}\tilde\Delta_{i,j}(\Pm) \mid \av(s)=\av
$$
and
$$
\varepsilon^{\Pm,\av}_s\defeq\ind\left\{\Pm(s)=\Pm\right\}\ind\left\{\av(s)=\av\right\}\ind\left\{\exists c, (i,j) \in P_c^2\right\}\ind\left\{c_i(s)\neq c_j(s)\right\}.
$$

Let $t$ be in $[T]$ and rewrite $\hat{s}_{i,j}(t)$ as
$$
\frac{1}{\sum_{s:\tilde\Pm(s)=\tilde\Pm}\sum_{\Pm,\av} \varepsilon^{\Pm,\av}_s}
\sum_{s:\tilde\Pm(s)=\tilde\Pm} \sum_{\Pm, \av} \varepsilon^{\Pm,\av}_s X^{\Pm,\av}_s.
$$

$X^{\Pm,\av}_s \in [-1,1]$, $\varepsilon^{\Pm,\av}_s \in \{0,1\}$, and the expectation of the variables $X^{\Pm,\av}_s$ are $\tilde\Delta_{i,j}(\av) \leqslant \tilde\Delta_{i,j}$ as $j\succ i$. Therefore, 
\begin{align*}
\PP(\frac{\bar{\bar s}_{j,i}(t)+1}{2}
\leqslant \frac{\tilde{\Delta}_{j,i}+1}{2})
&\leqslant e\lceil (\log \tilde{t}_{\tilde\Pm(t)}(t) + 3\log\log \tilde{t}_{\tilde\Pm(t)}(t))\log(\tilde{t}_{\tilde\Pm(t)}(t))\rceil \exp(-\log \tilde{t}_{\tilde\Pm(t)}(t) - 3\log\log \tilde{t}_{\tilde\Pm(t)}(t))
\\&
\leqslant\frac{e\lceil (\log \tilde{t}_{\tilde\Pm(t)}(t) + 3\log\log \tilde{t}_{\tilde\Pm(t)}(t))\log(\tilde{t}_{\tilde\Pm(t)}(t))\rceil}{\tilde{t}_{\tilde\Pm(t)}(t)\log^3 \tilde{t}_{\tilde\Pm(t)}(t)}.
\end{align*}

Hence,
\begin{align*}
\EE\left[\left|\left\{t \in [T]: \tilde\Pm(t)=\tilde\Pm, \frac{\bar{\bar s}_{j,i}(t)+1}{2} \leqslant \frac{\tilde{\Delta}_{i,j}+1}{2} \right\}\right|\right]
&\leqslant \sum_{s:\tilde\Pm(s)=\tilde\Pm} 
\end{align*}

\commentrg{Ça coince ici, à revoir}

\end{proof}
}

\section{Proof of Theorem \ref{theo:unirank} (Upper-Bound on the Regret of \ouralgo{} Assuming a Total Order on Items) }\label{app:regret_of_unirank}

Before proving the regret upper-bound of \ouralgo{}, we prove Lemmas \ref{theo:good_leader} and \ref{theo:leader_not_optimal} which are respectively bounding the exploration when the leader is the optimal one, and the number of iterations at which the leader is sub-optimal. Finally, the regret upper-bound of \ouralgo{} is given in Appendix \ref{app:regret_of_unirank_end}.

\subsection{Upper-Bound on the Number of Sub-Optimal Merges of \ouralgo{} when the Leader is the Optimal Partition}\label{app:good_leader}

\begin{lemma}[Upper-bound on the number of sub-optimal merges of \ouralgo{} when the leader is the optimal partition]\label{theo:good_leader}
Under the hypotheses of Theorem \ref{theo:unirank}, for any position $c\in\{2,\dots, L\}$ \ouralgo{} fulfills  
\begin{equation*}
\EE\left[ \sum_{t=1}^T\ind\left\{\substack{\tilde\Pm(t)=\Pm^*,\\\exists c', P_{c'}(t)=\{{\min(c-1, K)}, {c}\}}\right\}\right]
    \leqslant \frac{16}{\tilde\delta_c^*\tilde\Delta_{\min(c-1, K),c}^2}\log T + \OO\left(\log\log T\right).
\end{equation*}
\end{lemma}

\begin{proof}
Let $c\in\{2,\dots,L\}$ be a position, and denote $i$ (respectively $j$) the item ${\min(c-1, K)}$ (resp. $c$).
We aim at upper-bounding the number of iterations such that the leader $\tilde{\Pm}(t)$ is the optimal partition $\Pm^*$, and either the subsets $\Pm_{c-1}^* = \{i\}$ and $\Pm_c^* = \{j\}$ are merged in the chosen partition $\Pm(t)$, or $j\in \Pm_{K+1}^*(t)$ is added to the subset $\Pm_{K}^*= \{i\}$ in the chosen partition $\Pm(t)$.
Both situations require $\bar{\bar s}_{j,i}(t)$ to be positive.

Let decompose this number of iterations:
\begin{align*}
\EE\left[ \sum_{t=1}^T\ind\left\{\substack{\tilde\Pm(t)=\Pm^*,\\\exists c', P_{c'}(t)=\{{\min(c-1, K)}, {c}\}}\right\}\right]
    & \leqslant \EE\left[ \sum_{t=1}^T\ind\left\{\substack{\tilde\Pm(t)=\Pm^*,~
    \exists c', P_{c'}(t)=\{i, j\},\\
    \bar{\bar s}_{j,i}(t)\geqslant 0}\right\}\right]
\\
    & \leqslant\EE\left[ \sum_{t=1}^T\ind\left\{\substack{\tilde\Pm(t)=\Pm^*,~
    \exists c', P_{c'}(t)=\{i, j\},\\
    \hat{s}_{j,i}(t) > \frac{\tilde\Delta_{j,i}}{2}}\right\}\right]
\\
    & \quad +\EE\left[ \sum_{t=1}^T\ind\left\{\substack{\tilde\Pm(t)=\Pm^*,~
    \exists c', P_{c'}(t)=\{i, j\},\\
    T_j^*(t) < \frac{\tilde\delta_j^*}{2}t_j^*(t)}\right\}\right]
\\
    & \quad +\EE\left[ \sum_{t=1}^T\ind\left\{\substack{\tilde\Pm(t)=\Pm^*,~
    \exists c', P_{c'}(t)=\{i, j\},\\
    T_j^*(t) \geqslant \frac{\tilde\delta_j^*}{2}t_j^*(t),~
    \hat{s}_{j,i}(t) \leqslant  \frac{\tilde\Delta_{j,i}}{2},\\
    \bar{\bar s}_{j,i}(t)\geqslant0}\right\}\right]
    ,
\end{align*}
where $t_j^*(t) \defeq \sum_{s = 1}^{t-1} \ind\left\{\tilde\Pm(t)=\Pm^*\right\}\ind\left\{\exists c, (i,j)\in P_c(s)^2\right\}$,\\
and $T_j^*(t) \defeq \sum_{s = 1}^{t-1} \ind\left\{\tilde\Pm(t)=\Pm^*\right\}\ind\left\{\exists c, (i,j)\in P_c(s)^2\right\}\ind\{c_i(s) \neq c_j(s)\}.$

Let bound the first term in the right-hand side.

Denote $\Lambda=\left\{t: \tilde\Pm(t)=\Pm^*, ~\exists c', P_{c'}(t)=\{i, j\}\right\}$ the set of iterations at which $\tilde\Pm(t)=\Pm^*$ and both items $i$ and $j$ are gathered in a subset of $\Pm(t)$. We decompose that set as $\Lambda \subseteq \bigcup_{s\in\NN}\Lambda(s)$, with $\Lambda(s) \defeq \{t\in \Lambda: t_{i,j}(t) = s\}$. $|\Lambda(s)| \leqslant 1$ as $t_{i,j}(t)$ increases for each $t\in \Lambda$.  Note that for each $s\in \NN$ and $n \in \Lambda(s)$, $t_{i,j}(n) \geqslant t_{i,j}(n)= s$.

Note also that with the current hypothesis on the order, $i\succ j$, hence by Lemma \ref{lem:dev_s},
$$
\EE\left[\sum_{t\geq1} \ind\left\{t\in \Lambda, \hat{s}_{j,i}(t) > \frac{\tilde\Delta_{j,i}}{2} \right\}\right] = \OO(1).
$$

The second term  is bounded similarly with the same set $\Lambda$ but with a different decomposition: $\Lambda \subseteq \bigcup_{s\in\NN}\Lambda(s)$, with $\Lambda(s) \defeq \{t\in \Lambda: t_j^*(t) = s\}$. $|\Lambda(s)| \leqslant 1$ as $t_j^*(t)$ increases for each $t\in \Lambda$.  Note that for each $s\in \NN$ and $n \in \Lambda(s)$, $t_j^*(n) \geqslant t_j^*(n)= s$.

Therefore, the same proof as the one used in Lemma \ref{lem:dev_s} gives
$$
\EE\left[\sum_{t\geq1} \ind\left\{t\in \Lambda, T_j^*(t) < \frac{\tilde{\delta}_j^*}{2} t_j^*(t)\right\}\right] = \OO(1)
$$

It remains to upper-bound the third term.

Let note $C \defeq \left\{t\in [T]:
    \tilde\Pm(t)=\Pm^*,~
    \exists c', P_{c'}(t)=\{i, j\},~
    T_j^*(t) \geqslant  \frac{\tilde\delta_j^*}{2}t_j^*(t),~
    \hat{s}_{j,i}(t) \leqslant  \frac{\tilde\Delta_{j,i}}{2},~
    \bar{\bar s}_{j,i}(t)\geqslant0\right\}$.

Let $t\in C$.


By Pinsker's inequality and as $\bar{\bar s}_{j,i}(t)\geqslant0$,
\begin{align*}
    \frac{1}{2}
    & \leqslant \frac{\bar{\bar s}_{j,i}(t)+1}{2}
    \\
    &\leqslant \frac{\hat{s}_{j,i}(t)+1}{2} + \sqrt{\frac{\log(\tilde{t}_{\Pm^*}(t))+3\log(\log(\tilde{t}_{\Pm^*}(t)))}{2T_{i,j}(t)}}
    \\
    &\leqslant \frac{\tilde\Delta_{j,i}}{4} + \frac{1}{2} + \sqrt{\frac{\log(\tilde{t}_{\Pm^*}(t)))+3\log(\log(\tilde{t}_{\Pm^*}(t))))}{2T_{i,j}(t)}}
.
\end{align*}
Hence, $T_{i,j}(t) \leqslant  \frac{8\log(\tilde{t}_{\Pm^*}(t)))+24\log(\log(\tilde{t}_{\Pm^*}(t))))}{\tilde\Delta_{i,j}^2}$ as $\tilde\Delta_{i,j}=-\tilde\Delta_{j,i}>0$ given Lemma \ref{lem:min_s}. Then, by definition of $C$ and as
(i) $\tilde{t}_{\Pm^*}(t) \leqslant t\leqslant T$,
(ii) $T_j^*(t)\leqslant T_{i,j}(t)$,
and (iii) $\tilde\delta_j^*\geqslant\tilde\delta_{i,j}>0$ given Lemma \ref{lem:min_s}, $t_j^*(t) \leqslant \frac{2T_j^*(t)}{\tilde\delta_j^*} \leqslant \frac{2T_{i,j}(t)}{\tilde\delta_j^*} \leqslant  \frac{16\log(T)+48\log(\log(T))}{\tilde\delta_j^*\tilde\Delta_{i,j}^2}.$

Therefore, $C\subseteq \left\{t\in [T]:
    \tilde\Pm(t)=\Pm^*,~
    \exists c', P_{c'}(t)=\{i, j\},~
    t_j^*(t) \leqslant  \frac{16\log(T)+48\log(\log(T))}{\tilde\delta_j^*\tilde\Delta_{i,j}^2}
    \right\}$, and
\begin{align*}
    \EE\left[ \sum_{t=1}^T\ind\left\{\substack{\tilde\Pm(t)=\Pm^*,~
    \exists c', P_{c'}(t)=\{i, j\},\\
    T_j^*(t) \geqslant \frac{\tilde\delta_j^*}{2}t_j^*(t),~
    \hat{s}_{j,i}(t) \leqslant  \frac{\tilde\Delta_{j,i}}{2},\\
    \bar{\bar s}_{j,i}(t)\geqslant0}\right\}\right]
    &= \EE\left[|C|\right]
    \\
    &\leqslant \EE\left[\left|\left\{\substack{t\in [T]:
    ~\tilde\Pm(t)=\Pm^*,
    ~\exists c', P_{c'}(t)=\{i, j\},\\
    t_j^*(t) \leqslant  \frac{16\log(T)+48\log(\log(T))}{\tilde\delta_j^*\tilde\Delta_{i,j}^2}
    }\right\}\right|\right]
    \\
    &\leqslant \frac{16\log(T)+48\log(\log(T))}{\tilde\delta_j^*\tilde\Delta_{i,j}^2},
\end{align*}
which concludes the proof.

\end{proof}

\subsection{Upper-Bound on the Expected Number of Iterations at which the Leader is not the Optimal Partition}\label{app:leader_not_optimal}

\begin{lemma}[Upper-bound on the expected number of iterations at which the leader is not the optimal partition]\label{theo:leader_not_optimal}
Under the hypotheses of Theorem \ref{theo:unirank}, \ouralgo{} fulfills  
\begin{equation*}
\EE\left[\sum_{t=1}^T\ind\{\tilde\Pm(t)\neq\Pm^*\}\right]
    = \OO\left(\log\log T\right).
\end{equation*}
\end{lemma}

\begin{proof}
Let $\tilde\Pm \neq \Pm^*$ be an ordered partition of items of size $d$, and let upper-bound the expected number of iterations at which $\tilde\Pm(t)=\tilde\Pm$ by $\OO\left(\log\log T\right)$.
As there is a finite number of partitions, this will conclude the proof.

In this proof, for any couple of items $(i,j)$ we denote $\tilde{t}_{i,j}(t)\defeq\sum_{s = 1}^{t-1}\ind\left\{\tilde\Pm(t)=\tilde\Pm, \exists c, (i,j)\in P_c(s)^2\right\}$ the number of iterations at which both items have been gathered in the same subset of $\Pm(s)$ while the leader was $\tilde\Pm$. For each partition $\Pm$ in the neighborhood $\Nc\left(\tilde\Pm\right)$, we also denote $t_\Pm(t)\defeq\sum_{s = 1}^{t-1}\ind\left\{\tilde\Pm(t)=\tilde\Pm, \Pm(t)=\Pm\right\}$ the number of iterations at which $\Pm$ has been chosen while the leader was $\tilde\Pm$.

\commentrg[inline]{à corriger. La propriété correcte est : 
\begin{itemize}
    \item either there exists $c \in [K]$ and $(i,j)\in\tilde{P}_c^2$ such that $i\succ j$;
    \item or there exist $c \in [\tilde{d}-1]$, $i\in\tilde{P}_{c}$, and $j\in\tilde{P}_{c+1}$ such that $j \succ i$.
\end{itemize}
}

The proof depends on the difference between $\tilde\Pm$ and $\Pm^*$. By Lemma \ref{lem:unimodality}, 
\begin{itemize}
    \item either  $\exists c \in [\tilde{d}]$, such that $|P_c|>1$ and  $i^*\succ \argmax_{j\in P_c\setminus\{i^*\}} g(j)$, where $i^*=\argmax_{i\in P_c} g(i)$;
    \item or $\exists c \in [\tilde{d}-1]$, $\exists (i,j)\in\tilde{P}_c\times\tilde{P}_{c+1}$, such that $j \succ i$.
\end{itemize}
We first upper-bound the expected number of iterations at which $\tilde\Pm(t)=\tilde\Pm$ under the first condition, and then prove a similar upper-bound under the second condition.

\paragraph{Assume that there exists $c \in [\tilde{d}]$, such that $|P_c|>1$ and  $i\succ \argmax_{j\in P_c\setminus\{i^*\}} g(j)$, where $i=\argmax_{i\in P_c} g(i)$.}
Let $t$ be an iteration such that $\tilde\Pm(t)=\tilde\Pm$.
By \cref{asp:identifiability} and by design of the algorithm, if for each item $j\in \tilde{P}_c\setminus\{i\}$, the sign of $\hat{s}_{i,j}(t)$ would be the same as the sign of $\tilde\Delta_{i,j}>0$, then $i$ would be alone in $\tilde{P}_c(t)$. So $\hat{s}_{i,j}(t)\leqslant0$ for at least one item $j\in \tilde{P}_c\setminus\{i\}$. Let control the number of iteration at which this is true by considering the following decomposition:

$$\left\{t: \tilde\Pm(t)=\tilde\Pm\right\}
\subseteq \bigcup_{j\in\tilde{P}_c\setminus\{i\}} A_{i,j} \cup B_{i,j}\cup C_{i,j},
$$
where
$$
A_{i,j} \defeq \left\{t: \tilde\Pm(t)=\tilde\Pm, T_{i,j}(t) < \frac{\tilde\delta_{i,j}}{2}t_{i,j}(t)  \right\},
$$
$$
B_{i,j} \defeq \left\{t: \tilde\Pm(t)=\tilde\Pm, \frac{\hat{s}_{i,j}(t)}{\tilde\Delta_{i,j}} < \frac{1}{2}  \right\},
$$
and 
$$
C_{i,j}\defeq
\left\{t: \tilde\Pm(t)=\tilde\Pm,
T_{i,j}(t) \geqslant \frac{\tilde\delta_{i,j}}{2}t_{i,j}(t),
\frac{\hat{s}_{i,j}(t)}{\tilde\Delta_{i,j}} \geqslant \frac{1}{2},
\hat{s}_{i,j}(t) \leqslant 0,
\right\}.
$$

Let  $j$ be an item in $\tilde{P}_c\setminus\{i\}$, and let first upper-bound the expected size of $A_{i,j}$ and $B_{i,j}$, and then the expected size of $C_{i,j}$.

Note that at each iteration such that $\tilde\Pm(t)=\tilde\Pm$, $i$ and $j$ are in the same subset of the partition $\Pm(t)$, therefore $\tilde{t}_{i,j}(t) = \tilde{t}_{\tilde\Pm}(t)$.

Denote $\Lambda=\left\{t: \tilde\Pm(t)=\tilde\Pm\right\}$ the set of iterations at which $\tilde\Pm(t)=\tilde\Pm$, and decompose that set as $\Lambda \subseteq \bigcup_{s\in\NN}\Lambda(s)$, with $\Lambda(s) \defeq \{t\in \Lambda: \tilde{t}_{\tilde\Pm}(t) = s\}$. $|\Lambda(s)| \leqslant 1$ as $\tilde{t}_{\tilde\Pm}(t)$ increases for each $t\in \Lambda$.  Note that for each $s\in \NN$ and $n \in \Lambda(s)$, $t_{i,j}(n) \geqslant \tilde{t}_{i,j}(n) =\tilde{t}_{\tilde\Pm}(t)= s$.

Then by Lemma \ref{lem:dev_s}
$$
\EE\left[|A_{i,j}|\right]
=\EE\left[\sum_{t\geq1} \ind\left\{t\in \Lambda, T_{i,j}(t) < \frac{\tilde{\delta}_{i,j}}{2} t_{i,j}(t)\right\}\right] = \OO(1)
$$
and 
$$\EE\left[|B_{i,j}|\right]
=
\EE\left[\sum_{t\geq1} \ind\left\{t\in \Lambda, \frac{\hat{s}_{i,j}(t)}{\tilde\Delta_{i,j}} < \frac{1}{2} \right\}\right] = \OO(1).
$$

Let now upper-bound the expected size of $C_{i,j}$.

As $i\succ j$, $\tilde{\Delta}_{i,j}>0$.

Let $t\in C_{i,j}$. As $\hat{s}_{i,j}(t)\leqslant0$, $t \leqslant T$, and $\tilde{t}_{\tilde\Pm}(t) =\tilde{t}_{i,j}(t) \leqslant t_{i,j}(t) \leqslant \frac{2}{\tilde{\delta}_{i,j}} T_{i,j}(t)$, 
\begin{align*}
    0 &\geqslant \hat{s}_{i,j}(t)
    \geqslant \frac{\tilde{\Delta}_{i,j}}{2}
    >0,
\end{align*}
which is absurd. Hence, $C_{i,j} = \varnothing$, and $\EE\left[|C_{i,j}|\right]=0$.

Overall, if there exists $c \in [\tilde{d}]$, such that $|P_c|>1$ and  $i\succ \argmax_{j\in P_c\setminus\{i\}} g(j)$, where $i=\argmax_{i\in P_c} g(i)$,
\begin{align*}
\EE\left[\ind\{\tilde\Pm(t)=\tilde\Pm\}\right]
&\leqslant
\sum_{j\in\tilde{P}_c\setminus\{i\}} \EE\left[|A_{i,j}|\right] + \EE\left[|B_{i,j}|\right] + \EE\left[|C_{i,j}|\right]
\\
&=\OO(1) + \OO(1) + 0
\\
&=\OO(1)
\end{align*}

\paragraph{Assume that there exists $c\in [\tilde{d}-1]$, and $ (i,j)\in\tilde{P}_c\times\tilde{P}_{c+1}$, such that $j \succ i$.}
By design of \ouralgo{}, each neighbor of $\tilde\Pm$ takes one of both forms:

\begin{enumerate}
    \item $\left(\tilde{P}_1, \dots,\tilde{P}_{c-1}, \tilde{P}_c \cup \tilde{P}_{c+1},\tilde{P}_{c+2},\dots \tilde{P}_{\tilde{d}}\right)$,
    \item $\left(\tilde{P}_1, \dots,\tilde{P}_{\tilde{d}-2},\tilde{P}_{\tilde{d}-1}\cup\{j\}, \tilde{P}_{\tilde{d}}\setminus\{j\}\right)$.
\end{enumerate}
Let $\Pm$ be such neighbor. In the first scenario we denote $i(\Pm)\defeq\argmin_{i\in P_c}g(i)$ and $j(\Pm)\defeq\argmax_{j\in P_{c+1}}g(j)$. In the second scenario we denote $i(\Pm)\defeq\argmin_{i\in P_{\tilde{d}-1}}g(i)$, and $j(\Pm)$ the item $j$.
Finally, we denote $\Nc^+$ the set of neighbors $\Pm$ of $\tilde\Pm$ such that $j(\Pm) \succ i(\Pm)$, and $\Nc^-$ its complement $\{\tilde\Pm\}\cup\Nc(\tilde\Pm)\setminus \Nc^+$. 

It is also worth noting that with current hypothesis on $\tilde{\Pm}$,
\begin{itemize}
    \item $ \left|\Nc^+\right| + \left|\Nc^-\right| = \left|\Nc\left(\tilde{\Pm}\right)\right| + 1 \leqslant L$;
    \item $\Nc^+$ is non-empty (due to current assumption on $\tilde\Pm$);
    \item for each partition $\Pm\in\Nc(\tilde{\Pm})$,  $t_\Pm(t) = \tilde{t}_{i(\Pm),j(\Pm)}(t)$;
    \item by design of the algorithm, at each iteration $t$ such that $\tilde{\Pm}(t)=\tilde{\Pm}$,  $\hat{s}_{i(\Pm),j(\Pm)}(t) > 0$ for each partition $\Pm\in\Nc(\tilde{\Pm})$ as $i(\Pm)$ is in a subset before $j(\Pm)$ in $\tilde{\Pm}$.
\end{itemize}

To bound $\EE\left[\ind\{\tilde\Pm(t)=\tilde\Pm\}\right]$, we use the decomposition
$\{t \in [T]: \tilde\Pm(t)=\tilde\Pm\} = \cup_{\Pm^+\in\Nc^+}A_{\Pm^+} \cup B$ where
$$A_{\Pm^+} = \left\{t: \tilde\Pm(t)=\tilde\Pm, t_{\Pm^+}(t) \geqslant \varepsilon \tilde{t}_{\tilde\Pm}(t) \right\},$$
$$B = \left\{t: \tilde\Pm(t)=\tilde\Pm, \forall \Pm \in \Nc^+, t_{\Pm^+}(t) < \varepsilon \tilde{t}_{\tilde\Pm}(t) \right\},$$
$$\text{and } \varepsilon\defeq\frac{1}{\left|\Nc\left(\tilde{\Pm}\right)\right|+1}\geqslant\frac{1}{L}.$$
Hence, 
$$\EE\left[\ind\{\tilde\Pm(t)=\tilde\Pm\}\right]
\leqslant \sum_{\Pm \in \Nc^+}\EE\left[|A_{\Pm^+}|\right]
+ \EE\left[|B|\right].$$

\paragraph{Bound on $\EE\left[|A_{\Pm^+}|\right]$}
Let $\Pm \in \Nc^+$ be a permutation.

First, let's $t$ be in $A_{\Pm^+}$. Note that $\tilde{\Delta}_{i(\Pm^+),j(\Pm^+)}<0$, as $j(\Pm^+)\succ i(\Pm^+)$. Therefore, as $\hat{s}_{i(\Pm^+),j(\Pm^+)}(t) > 0$,  $\hat{s}_{i(\Pm^+), j(\Pm^+)}(t)>\frac{\tilde{\Delta}_{i(\Pm^+), j(\Pm^+)}}{2}$, and thus $E\left[|A_{\Pm^+}|\right]=
\EE\left[\sum_{t\geq1} \ind\left\{t\in A_{\Pm^+}, \hat{s}_{i(\Pm^+), j(\Pm^+)}(t) > \frac{\tilde{\Delta}_{i(\Pm^+), j(\Pm^+)}}{2} \right\}\right]$.

Secondly, let's decompose $A_{\Pm^+}$ as $A_{\Pm^+} \subseteq \bigcup_{s\in\NN}\Lambda(s)$, with $\Lambda(s) \defeq \{t\in A_{\Pm^+}: \tilde{t}_{\tilde\Pm}(t) = s\}$. $|\Lambda(s)| \leqslant 1$ as $\tilde{t}_{\tilde\Pm}(t)$ increases for each $t\in A_{\Pm^+}$.  Note that for each $s\in \NN$ and $n \in \Lambda(s)$, $t_{i(\Pm^+),j(\Pm^+)}(n) \geqslant \tilde{t}_{i(\Pm^+),j(\Pm^+)}(n) = t_{\Pm^+}(n) \geqslant \varepsilon\tilde{t}_{\tilde\Pm}(t)=\varepsilon s$.

Thus, as $j(\Pm^+)\succ i(\Pm^+)$, by Lemma \ref{lem:dev_s}
$$\EE\left[\sum_{t\geq1} \ind\left\{t\in A_{\Pm^+}, \hat{s}_{i(\Pm^+), j(\Pm^+)}(t) > \frac{\tilde{\Delta}_{i(\Pm^+), j(\Pm^+)}}{2} \right\}\right] = \OO(1).$$

Overall, $E\left[|A_{\Pm^+}|\right]=
\EE\left[\sum_{t\geq1} \ind\left\{t\in A_{\Pm^+}, \hat{s}_{i(\Pm^+), j(\Pm^+)}(t) > \frac{\tilde{\Delta}_{i(\Pm^+), j(\Pm^+)}}{2} \right\}\right]
= \OO\left(1\right).$

\paragraph{Bound on $\EE\left[|B|\right]$}
We first split $B$ in two parts: $B=B^{t_0} \cup B_{t_0}^T$, where $B^{t_0}\defeq \{t\in B: \tilde{t}_{\tilde\Pm}(t)\leqslant t_0\}$, $B_{t_0}^T\defeq \{t\in B: \tilde{t}_{\tilde\Pm}(t) > t_0\}$, and $t_0$ is chosen as small as possible to satisfy a constraint required later on in the proof.
Namely, $t_0 =
\max_{\Pm^- \in \Nc(\tilde\Pm)\setminus \Nc^+}
\inf \left\{s:
\sqrt{\frac{\log(s)+3\log(\log(s))}{\tilde{\delta}_{j(\Pm^-),i(\Pm^-)}(\varepsilon s-1)}}
 < \frac{\tilde{\Delta}_{i(\Pm^-),j(\Pm^-)}}{8}\right\}
$, with $t_0=0$ if $\Nc(\tilde\Pm)\setminus \Nc^+$ is empty. Note that $t_0$ only depends on $\tilde{\delta}_{j(\Pm^-),i(\Pm^-)}$ and  $\tilde{\Delta}_{i(\Pm^-),j(\Pm^-)}$ for $\Pm^-\in \Nc(\tilde\Pm)\setminus \Nc^+$.

We also define
\begin{itemize}
    \item $D_{\Pm^-} \defeq \left\{t \in [T]:  \tilde\Pm(t)=\tilde\Pm, \Pm(t)=\Pm^-, T_{j(\Pm^-),i(\Pm^-)}(t) < \frac{\tilde{\delta}_{j(\Pm^-),i(\Pm^-)}}{2} t_{j(\Pm^-),i(\Pm^-)}(t)\right\},$ for each $\Pm^-\in \Nc(\tilde\Pm)\setminus \Nc^+$
    \item $E_{\Pm^-} \defeq \left\{t \in [T]:  \tilde\Pm(t)=\tilde\Pm, \Pm(t)=\Pm^-, \hat{s}_{i(\Pm^-),j(\Pm^-)}(t) < \frac{\tilde{\Delta}_{i(\Pm^-),j(\Pm^-)}}{2}\right\},$ for each $\Pm^-\in \Nc(\tilde\Pm)\setminus \Nc^+$
    \item $F_{\Pm^+} \defeq \left\{t \in [T]:  \tilde\Pm(t)=\tilde\Pm, \frac{\bar{\bar s}_{j(\Pm),i(\Pm)}(t)+1}{2} \leqslant \frac{\tilde{\Delta}_{j(\Pm),i(\Pm)}+1}{2} \right\}$ for each $\Pm^+\in\Nc^+$.
\end{itemize}

Let $t \in B_{t_0}^T$. We have
$$\tilde{t}_{\tilde\Pm}(t)  = \sum_{\Pm^+\in\Nc^+}t_{\Pm^+}(t) + \sum_{\Pm^-\in \Nc^-} t_{\Pm^-}(t),$$
and by definition of $B$, $t_{\Pm^+}(t) < \varepsilon \tilde{t}_{\tilde\Pm}(t)$ for each $\Pm^+\in\Nc^+$. So, there exists $\Pm^-\in \Nc^-$ such that $t_{\Pm^-}(t) > \varepsilon \tilde{t}_{\tilde\Pm}(t)$ (otherwise, $\tilde{t}_{\tilde\Pm}(t)  = \sum_{\Pm^+\in\Nc^+}t_{\Pm^+}(t) + \sum_{\Pm^-\in \Nc^-} t_{\Pm^-}(t) < (|\Nc^+| + |\Nc^-|)\varepsilon \tilde{t}_{\tilde\Pm}(t) = \tilde{t}_{\tilde\Pm}(t)$, which is absurd).

Let's elicit an iteration $\psi(t)$ with specific properties. We denote $s'$ the first iteration such that $t_{\Pm^-}(s')\geqslant \varepsilon\tilde{t}_{\tilde\Pm}(t)$. At this iteration,  $t_{\Pm^-}(s')=\lceil \varepsilon\tilde{t}_{\tilde\Pm}(t)\rceil$, and $t_{\Pm^-}(s') =t_{\Pm^-}(s'-1) +1$, meaning that $\tilde\Pm(s'-1)=\tilde\Pm$ and $\Pm(s'-1)=\Pm^-$, and $t_{\Pm^-}(s'-1) = \lceil \varepsilon\tilde{t}_{\tilde\Pm}(t)\rceil-1$. Therefore, the set  $\{s \in [t]: \tilde\Pm(s) = \tilde\Pm, \Pm(s)=\Pm^-, t_{\Pm^-}(s) = \lceil \varepsilon\tilde{t}_{\tilde\Pm}(t)\rceil-1\}$ is non-empty. We define $\psi(t)$ as the minimum on this set
$$ \psi(t) \defeq \min\{s \in [t]: \tilde\Pm(s) = \tilde\Pm, \Pm(s)=\Pm^-, t_{\Pm^-}(s) = \lceil \varepsilon\tilde{t}_{\tilde\Pm}(t)\rceil-1\}.$$

Let prove by contradiction that $\psi(t) \in \bigcup_{\Pm^-\in \Nc^-} (D_{\Pm^-} \cup E_{\Pm^-}) \cup \bigcup_{\Pm^+\in\Nc^+}F_{\Pm^+}$. Assume that $\psi(t) \notin \bigcup_{\Pm^-\in \Nc^-} (D_{\Pm^-} \cup E_{\Pm^-}) \cup \bigcup_{\Pm^+\in\Nc^+}F_{\Pm^+}$. The partition $\Pm^-$ is in $\Nc^-$, so either $\Pm^-=\tilde\Pm$ or $\Pm^-\in \Nc(\tilde\Pm)\setminus \Nc^+$.

The set $\Nc^+$ is non-empty, so there exists a partition $\Pm^+\in \Nc^+$. As $\psi(t) \notin \bigcup_{\Pm^+\in\Nc^+}F_{\Pm^+}$, $\frac{\bar{\bar s}_{j(\Pm),i(\Pm)}(\psi(t))+1}{2} > \frac{\tilde{\Delta}_{j(\Pm),i(\Pm)}+1}{2}$, where $\tilde{\Delta}_{j(\Pm),i(\Pm)}>0$ as $j(\Pm)\succ i(\Pm)$. Thus  $\bar{\bar s}_{j(\Pm),i(\Pm)}(\psi(t))>0$ and $\Pm(\psi(t))=\Pm^-$ cannot be $\tilde\Pm$ by design of \ouralgo{}. Therefore $\Pm^-\in \Nc(\tilde\Pm)\setminus \Nc^+$.

Thus, either $\Nc(\tilde\Pm)\setminus \Nc^+$ is empty and we get a contradiction, or $i(\Pm^-)$ and $j(\Pm^-)$ are properly defined, and, by design of \ouralgo{}, $\bar{\bar s}_{j(\Pm^-),i(\Pm^-)}(\psi(t))\geqslant0$.
Moreover, since $\tilde\Pm(\psi(t)) = \tilde\Pm^-$ and $\psi(t) \notin D_{\Pm^-} \cup E_{\Pm^-}$,
$ T_{j(\Pm^-),i(\Pm^-)}(\psi(t)) \geqslant \frac{\tilde{\delta}_{j(\Pm^-),i(\Pm^-)}}{2} t_{j(\Pm^-),i(\Pm^-)}(\psi(t))$
and $\hat{s}_{i(\Pm^-),j(\Pm^-)}(\psi(t)) \geqslant \frac{\tilde{\Delta}_{i(\Pm^-),j(\Pm^-)}}{2}$.

Therefore,
\begin{multline*}
T_{j(\Pm^-),i(\Pm^-)}(\psi(t))
\geqslant \frac{\tilde{\delta}_{j(\Pm^-),i(\Pm^-)}}{2} t_{j(\Pm^-),i(\Pm^-)}(\psi(t))
\geqslant \frac{\tilde{\delta}_{j(\Pm^-),i(\Pm^-)}}{2} \tilde{t}_{j(\Pm^-),i(\Pm^-)}(\psi(t))\\
= \frac{\tilde{\delta}_{j(\Pm^-),i(\Pm^-)}}{2} t_{\Pm^-}(\psi(t))
= \frac{\tilde{\delta}_{j(\Pm^-),i(\Pm^-)}}{2} (\lceil \varepsilon\tilde{t}_{\tilde\Pm}(t)\rceil-1)
\geqslant \frac{\tilde{\delta}_{j(\Pm^-),i(\Pm^-)}}{2} (\varepsilon \tilde{t}_{\tilde\Pm}(t)-1)
\end{multline*}
and by Pinsker's inequality and the fact that $\psi(t)\leqslant t$ and $\tilde{t}_{\tilde\Pm}(s)$ is non-decreasing in $s$, and $\tilde{t}_{\tilde\Pm}(t) > t_0$,
\begin{align*}
\frac{1}{2} \geqslant \frac{-\bar{\bar s}_{j(\Pm^-),i(\Pm^-)}(\psi(t))+1}{2}
&\geqslant \frac{-\hat{s}_{j(\Pm^-),i(\Pm^-)}(\psi(t))+1}{2} - \sqrt{\frac{\log(\tilde{t}_{\tilde\Pm}(\psi(t)))+3\log(\log(\tilde{t}_{\tilde\Pm}(\psi(t))))}{2T_{j(\Pm^-),i(\Pm^-)}(\psi(t))}}
\\
&= \frac{\hat{s}_{i(\Pm^-),j(\Pm^-)}(\psi(t))+1}{2} - \sqrt{\frac{\log(\tilde{t}_{\tilde\Pm}(\psi(t)))+3\log(\log(\tilde{t}_{\tilde\Pm}(\psi(t))))}{2T_{j(\Pm^-),i(\Pm^-)}(\psi(t))}}
\\
&\geqslant \frac{1}{2} + \frac{\tilde{\Delta}_{i(\Pm^-),j(\Pm^-)}}{4} - \sqrt{\frac{\log(\tilde{t}_{\tilde\Pm}(t))+3\log(\log(\tilde{t}_{\tilde\Pm}(t)))}{\tilde{\delta}_{j(\Pm^-),i(\Pm^-)}(\varepsilon\tilde{t}_{\tilde\Pm}(t)-1)}}
\\
&\geqslant \frac{1}{2} + \frac{\tilde{\Delta}_{i(\Pm^-),j(\Pm^-)}}{4} - \frac{\tilde{\Delta}_{i(\Pm^-),j(\Pm^-)}}{8}
\\
&= \frac{1}{2} + \frac{\tilde{\Delta}_{i(\Pm^-),j(\Pm^-)}}{8}
\end{align*}
which contradicts the fact that $\tilde{\Delta}_{i(\Pm^-),j(\Pm^-)} > 0$.

Overall, we always get a contradiction, so, for any $t\in B_{t_0}^T$,  $\psi(t) \in \bigcup_{\Pm^-\in \Nc^-} (D_{\Pm^-} \cup E_{\Pm^-}) \cup \bigcup_{\Pm^+\in\Nc^+}F_{\Pm^+}$.

Hence, $ B_{t_0}^T \subseteq  \bigcup_{n\in \bigcup_{\Pm^-\in \Nc^-} (D_{\Pm^-} \cup E_{\Pm^-}) \cup \bigcup_{\Pm^+\in\Nc^+}F_{\Pm^+}} B_{t_0}^T\cap\left\{t\in [T]: \psi(t)=n \right\}$. Let $n$ be in $\bigcup_{\Pm^-\in \Nc^-} (D_{\Pm^-} \cup E_{\Pm^-}) \cup \bigcup_{\Pm^+\in\Nc^+}F_{\Pm^+}$. For any $t$ in $B_{t_0}^T\cap\left\{t\in [T]: \psi(t)=n \right\}$, there exists a partition  $\Pm^-\in\Nc^-$ such that  $t_{\Pm^-}(n) = \lceil\varepsilon \tilde{t}_{\tilde\Pm}(t)-1\rceil$, and $t_{\Pm^-}(n+1)=t_{\Pm^-}(n)+1$. So $|B_{t_0}^T\cap\left\{t\in [T]: \psi(t)=n \right\}| \leqslant L$ and
$$\EE\left[|B|\right]
\leqslant t_0 + \EE\left[|B_{t_0}^T|\right]
\leqslant t_0 + |\Nc^-|(\EE\left[|D|\right]+\EE\left[|E|\right]+\EE\left[|F|\right]).$$

It remains to upper-bound $\EE\left[|D|\right]$, $\EE\left[|E|\right]$, and $\EE\left[|F|\right]$ to conclude the proof.

\paragraph{Bound on $\EE\left[|D_{\Pm^-}|\right]$ and $\EE\left[|E_{\Pm^-}|\right]$}
Let $\Pm^-\in \Nc(\tilde\Pm)\setminus \Nc^+$
The upper-bound on $\EE\left[|D_{\Pm^-}|\right]$  and $\EE\left[|E_{\Pm^-}|\right]$ are obtained through Lemma \ref{lem:dev_s}.
Let $\Lambda\defeq\left\{t \in [T]: \tilde\Pm(t)=\tilde\Pm, \Pm(t)=\Pm^-\right\}$ and let use the decomposition
$\Lambda \subseteq \bigcup_{s\in \NN} \Lambda(s)$, where $\Lambda(s) \defeq \{t\in \Lambda: t_{i(\Pm^-),j(\Pm^-)}(t)=s\}$. $|\Lambda(s)| \leqslant 1$ as  $t_{i(\Pm^-),j(\Pm^-)}(t)$ increases for each $t\in \Lambda$. Note that for each $s\in \NN$ and $n\in \Lambda(s)$, $t_{i(\Pm^-),j(\Pm^-)}(n) \geqslant t_{i(\Pm^-),j(\Pm^-)}(n)=s$. Then, by Lemma \ref{lem:dev_s}, as $i(\Pm^-)\succ j(\Pm^-)$
$$
    \EE\left[|D_{\Pm^-}|\right]
    = \EE\left[\sum_{t=1}^T \ind\{t \in \Lambda: T_{i(\Pm^-),j(\Pm^-)}(t) < \frac{\tilde{\delta}_{i(\Pm^-),j(\Pm^-)}}{2} t_{i(\Pm^-),j(\Pm^-)}(t)\}\right] = \OO(1)
$$
and
$$
    \EE\left[|E_{\Pm^-}|\right]
    = \EE\left[\sum_{t=1}^T \ind\{t \in \Lambda: \hat{s}_{i(\Pm^-),j(\Pm^-)}(t) < \frac{\tilde{\Delta}_{i(\Pm^-),j(\Pm^-)}}{2}\}\right] = \OO(1).
$$

\paragraph{Bound on $\EE\left[|F_{\Pm^+}|\right]$}
By Lemma \ref{theo:10GravierCappe}, for each partition $\Pm^+\in\Nc^+$, $\EE\left[|F_{\Pm^+}|\right] = O(\log(\log(T)))$.

Overall $\EE\left[\ind\{\tilde\Pm(t)=\tilde\Pm\}\right] \leqslant
|\Nc^+|\OO(1)
+ t_0
+ |\Nc^-|\left((|\Nc^-|-1)\OO(1) + (|\Nc^-|-1)\OO(1) + |\Nc^+|\OO(\log\log T)\right)
= \OO\left(\log\log T\right)$, which concludes the proof.
\end{proof}

\subsection{Final Step of the Proof of Theorem \ref{theo:unirank} (Upper-Bound on the Regret of \ouralgo{} Assuming a Total Order on Items) }\label{app:regret_of_unirank_end}

The proof of Theorem \ref{theo:unirank} from Lemmas \ref{theo:good_leader} and \ref{theo:leader_not_optimal} is mainly based on an appropriate decomposition of the regret.

\begin{proof}[Proof of Theorem \ref{theo:unirank}]
The upper-bound on the expected number of iterations at which \ouralgo{} explores while the leader is the optimal partition is given by Lemma \ref{theo:good_leader}.

The upper-bound on the expected number of iterations at which the leader is not the optimal partition is given by Lemma \ref{theo:leader_not_optimal}.

Let now consider the impact of these upper-bounds on the regret of \ouralgo{}.

Let remind that $P_c^* = \{c\}$ for $c\in[K]$, $d^*=K+1$, and $P_{K+1}^* = [L]\setminus[K]$.
Therefore, $\mu^*=\mu_{a^*}=\sum_{k=1}^K \rho(\av^*,k)$, where $\av^*\defeq\left(1, 2, \dots, K\right)$.

Let first upper-bound the regret suffered at iteration $t$ while the the leader is the optimal partition:
\begin{align*}
R^*_t &= \mu^* -\EE_{\av(t)}\left[ \mu_{\av(t)}\mid\tilde\Pm(t)=\Pm^*\right]
\\
    &= \sum_{k=1}^K\rho(\av^*,k) -\EE_{\av(t)}\left[\rho(\av(t),k)\mid\tilde\Pm(t)=\Pm^*\right]
\\
    &= \sum_{k=1}^{K} \PP\left(a_k(t)={k}\mid\tilde\Pm(t)=\Pm^*\right) \left(\rho(\av^*,k)-\EE_{\av(t)}\left[\rho(\av(t),k)\mid a_k(t)={k}, \tilde\Pm(t)=\Pm^*\right]\right)
    \\&\quad
    + \sum_{k=2}^{K} \PP\left(a_{k-1}(t)={k}\mid\tilde\Pm(t)=\Pm^*\right) \left(\rho(\av^*,k-1)-\EE_{\av(t)}\left[\rho(\av(t),k-1)\mid a_{k-1}(t)={k}, \tilde\Pm(t)=\Pm^*\right]\right)
    \\&\quad
    + \sum_{k=2}^{K} \PP\left(a_k(t)={k-1}\mid\tilde\Pm(t)=\Pm^*\right) \left(\rho(\av^*,k)-\EE_{\av(t)}\left[\rho(\av(t),k)\mid a_k(t)={k-1}, \tilde\Pm(t)=\Pm^*\right]\right)
    \\&\quad
    + \sum_{\ell=K+1}^L \PP\left(a_K(t)={\ell}\mid\tilde\Pm(t)=\Pm^*\right) \left(\rho(\av^*,k)-\EE_{\av(t)}\left[\rho(\av(t),K)\mid a_K(t)={\ell}, \tilde\Pm(t)=\Pm^*\right]\right)
\end{align*}

Let's focus on the first right hand-side term. As the probability of click at position $k$ only depends on the set of items in positions $1$ to $k-1$, and as under the condition $a_k(t)={k} \wedge \tilde\Pm(t)=\Pm^*$,  $\av(t)$ and $\av^*$ have the same set of items in positions $1$ to $k-1$, $\rho(\av^*,k)=\EE_{\av(t)}\left[\rho(\av(t),k)\mid a_k(t)={k}, \tilde\Pm(t)=\Pm^*\right]$. Hence that term is equal to 0.

Let now take a look at the second term. By design of \ouralgo{}, as $a_{k-1}(t)={k} \wedge \tilde\Pm(t)=\Pm^*$, there exists $c'$ such that $P_c(t)=\{{k-1}, k\}$, and
\begin{align*}
\PP\left(a_{k-1}(t)={k}\mid\tilde\Pm(t)=\Pm^*\right)
    &= \PP\left(a_{k-1}(t)={k}, P_c(t)=\{{k-1}, k\}\mid\tilde\Pm(t)=\Pm^*\right)\\
    &= \frac{1}{2}\PP\left(P_c(t)=\{{k-1}, k\}\mid\tilde\Pm(t)=\Pm^*\right).
\end{align*}

Similarly, the third term corresponds to the existence of $c'$ such that $P_c(t)=\{{k-1}, k\}$, and
$$
\PP\left(a_{k}(t)={k-1}\mid\tilde\Pm(t)=\Pm^*\right)= \frac{1}{2}\PP\left(P_c(t)=\{{k-1}, k\}\mid\tilde\Pm(t)=\Pm^*\right).
$$

By summing both terms, we have to handle
\begin{multline*}
\frac{1}{2}\PP\left(P_c(t)=\{{k-1}, k\}\mid\tilde\Pm(t)=\Pm^*\right)\cdot\\
\left(\rho(\av^*,k-1) + \rho(\av^*,k) - \EE_{\av(t)}\left[\rho(\av(t),k-1) + \rho(\av(t),k)\mid a_{k-1}(t)={k}, a_{k}(t)={k-1}, \tilde\Pm(t)=\Pm^*\right]\right),
\end{multline*}
which is equal to
$\frac{1}{2}\PP\left(P_c(t)=\{{k-1}, k\}\mid\tilde\Pm(t)=\Pm^*\right)\Delta_k,$ where
$$
\Delta_k\defeq\rho(\av^*,k-1) + \rho(\av^*,k) - \rho(({k-1},{k})\circ\av^*,k-1) - \rho(({k-1},{k})\circ\av^*,k),
$$
as the probability of click at any position $k'$ only depends on the set of items in positions $1$ to $k'-1$.

Finally, following the same argumentation, the last term is equal to
$\frac{1}{2}\PP\left(P_c(t)=\{K, \ell\}\mid\tilde\Pm(t)=\Pm^*\right)\Delta_\ell,$ where
$
\Delta_\ell\defeq\rho(\av^*,K) - \rho((K,\ell)\circ\av^*,K).
$

Overall
\begin{align*}
R^*_t
    &= \sum_{k=2}^{K} \frac{1}{2}\PP\left(P_c(t)=\{{k-1}, k\}\mid\tilde\Pm(t)=\Pm^*\right)\Delta_k
    \\&\quad
    + \sum_{\ell=K+1}^L \frac{1}{2}\PP\left(P_c(t)=\{K, \ell\}\mid\tilde\Pm(t)=\Pm^*\right)\Delta_\ell
\\
    &=\sum_{k=2}^{L}\frac{1}{2}\PP\left(P_c(t)=\{{\min(k-1,K)}, k\}\mid\tilde\Pm(t)=\Pm^*\right)\Delta_k.
\end{align*}

Let finally upper-bound the overall regret.
\begin{align*}
R(T) &= \sum_{t = 1}^T \mu^* -\EE_{\av(t)}\left[\mu_{\av(t)}\right]
\\
    &= \sum_{t = 1}^T\PP\left(\tilde\Pm(t)\neq\Pm^*\right)\left(\mu^* -\EE_{\av(t)}\left[ \mu_{\av(t)}\mid\tilde\Pm(t)\neq\Pm^*\right]\right)
    \\&\quad
    + \sum_{t = 1}^T\PP\left(\tilde\Pm(t)=\Pm^*\right)\left(\mu^* -\EE_{\av(t)}\left[ \mu_{\av(t)}\mid\tilde\Pm(t)=\Pm^*\right]\right)
\\
    &\leqslant \sum_{t = 1}^T\PP\left(\tilde\Pm(t)\neq\Pm^*\right)K
    \\&\quad
    + \sum_{t = 1}^T\PP\left(\tilde\Pm(t)=\Pm^*\right)\sum_{k=2}^{L}\frac{1}{2}\PP\left(P_c(t)=\{{\min(k-1,K)}, k\}\mid\tilde\Pm(t)=\Pm^*\right)\Delta_k
\\
    &\leqslant \OO\left(\log\log T\right)
    \\&\quad
    + \sum_{t = 1}^T\sum_{k=2}^{L}\frac{1}{2}\PP\left(\tilde\Pm(t)=\Pm^*, P_c(t)=\{{\min(k-1,K)}, k, \}\right)\Delta_k
\\
    &= \OO\left(\log\log T\right)
    \\&\quad
    + \sum_{k=2}^{L}\frac{\Delta_k}{2}\sum_{t = 1}^T\PP\left(\tilde\Pm(t)=\Pm^*, P_c(t)=\{{\min(k-1,K)}, k, \}\right)
\\
    &\leqslant \OO\left(\log\log T\right)
    \\&\quad
    + \sum_{k=2}^{L}\frac{\Delta_k}{2}\left(\frac{16}{\tilde\delta_k^*\tilde\Delta_{k}^2}\log T + \OO\left(\log\log T\right)\right)
\\
    &=\sum_{k=2}^{L}\frac{8\Delta_k}{\tilde\delta_k^*\tilde\Delta_{k}^2}\log T + \OO\left(\log\log T\right)
\\
    &=\OO\left(\frac{L}{\Delta}\log T\right)
    ,
\end{align*}
where for any index $k\geqslant2$
\begin{align*}
    \tilde\Delta_k &\defeq \tilde\Delta_{{\min(k-1, K)},{k}}
    &and&
    &\Delta&\defeq\min_{k\in\{2, \dots, K\}} \frac{\tilde\delta_k^*\tilde\Delta_{k}^2}{8\Delta_k},
\end{align*}
which concludes the proof.

\end{proof}

\section{\ouralgo{}'s Theoretical Results While Facing State-of-the-Art Click Models}\label{app:vs_others}

Here, we prove Corollaries \ref{theo:unirank_CM} and \ref{theo:unirank_PBM} and then discuss the relationship between our upper-bounds and the known lower bounds.

\subsection{Proof of Corollary \ref{theo:unirank_CM} (Upper-Bound on the Regret of \ouralgo{} when Facing CM$^*$ Click Model) }\label{app:regret_of_unirank_vs_CM}

Corollary \ref{theo:unirank_CM_detailed} is a more precise version of Corollary \ref{theo:unirank_CM}. Its proof consists in identifying the gaps $\tilde\delta_k^*$, $\tilde\Delta_k$, and $\Delta_k$, where $k$ is the index of an item.

\begin{corollary}[Facing CM$^*$ click model]\label{theo:unirank_CM_detailed}
Under the hypotheses of Theorem \ref{theo:unirank}, if the user follows CM with  probability $\theta_i$ to click on item $i$ when it is observed, then for any index $k\geqslant 2$,
\begin{align*}
    \tilde\delta_k^* &=\left(\theta_{k-1}+\theta_{k}-\theta_{k-1}\theta_{k}\right)\prod_{\ell=1}^{k-2}\left(1-\theta_{{\ell}}\right)
    && \text{if }k\leqslant K,\\
    \tilde\delta_k^* &=\frac{1}{2}\left(\theta_{K}+\theta_{k}\right)\prod_{\ell=1}^{K-1}\left(1-\theta_{{\ell}}\right)
    && \text{if }k\geqslant K+1,\\
    \tilde\Delta_k &\geqslant\frac{\theta_{{\min(K,k-1)}}-\theta_{k}}{\theta_{{\min(K,k-1)}}+\theta_{k}},\\
    \Delta_k &= 0
    && \text{if }k\leqslant K,\\
    \Delta_k &=\left(\theta_{K}-\theta_{k}\right)\prod_{\ell=1}^{K-1}\left(1-\theta_{{\ell}}\right)
    && \text{if }k\geqslant K+1.
\end{align*}
Hence, \ouralgo{} fulfills
\begin{align*}
R(T)
&\leqslant
\sum_{k=K+1}^L16\frac{\theta_K+\theta_k}{\theta_K-\theta_k}\log T
+ \OO\left(\log\log T\right)
\\&
=\OO\left((L-K)\frac{\theta_K+\theta_{K+1}}{\theta_K-\theta_{K+1}}\log T\right).
\end{align*}
\end{corollary}

\begin{proof}[Proof of Corollary \ref{theo:unirank_CM_detailed}]
Values $\tilde\delta_k^*$ and $\Delta_k$ derive from a straightforward computation given CM model.

Let us prove the lower-bound on $\tilde\Delta_k$. Let $i$ and $j$ be two items such that $i\neq j$. Let $\av$ be a recommendation such that $\PP(c_i(t)\neq c_j(t)\mid \av(t)=\av)> 0$.

Without loss of generality, assume $i$ appears in $\av$ in position $k$, and if $j$ appears in $\av$, it is in a position $\ell>k$. Then
\begin{align*}
\tilde\Delta_{i,j}(\av)
    &=\frac{A\frac{1+B}{2}(\theta_i-\theta_j)}{A\frac{1+B}{2}(\theta_i+\theta_j)-AB\theta_i\theta_j}
    \geqslant\frac{\theta_i-\theta_j}{\theta_i+\theta_j},
\end{align*}
with $A\defeq\prod_{c=1}^{k-1}\left(1-\theta_{a_{c}}\right)$
and $B\defeq\prod_{c=k+1}^{\ell-1}\left(1-\theta_{a_{c}}\right)$ if $j$ appears in $\av$ and 0 otherwise.

Hence the lower-bounding values for $\tilde\Delta_k$, by noting that the term $A$ is lower-bounded by $\prod_{\ell=1}^{K-1}\left(1-\theta_{{\ell}}\right).$

Regarding the last formula in Lemma \ref{theo:unirank_CM_detailed}, it derives from the fact that
$\frac{\theta_{K}+\theta_{k}}{\theta_{K}-\theta_{k}}$
is maximized when $\theta_k$ is maximized, meaning $k=K+1$.
\end{proof}

\subsection{Proof of Corollary \ref{theo:unirank_PBM} (Upper-Bound on the Regret of \ouralgo{} when Facing PBM$^*$ Click Model) }\label{app:regret_of_unirank_vs_PBM}

Corollary \ref{theo:unirank_PBM_detailed} is a more precise version of  Corollary \ref{theo:unirank_PBM}. Its proof consists in identifying the gaps $\tilde\delta_k^*$, $\tilde\Delta_k$, and $\Delta_k$, where $k$ is the index of an item.

\begin{corollary}[Facing PBM$^*$ click model]\label{theo:unirank_PBM_detailed}
Under the hypotheses of Theorem \ref{theo:unirank}, if the user follows PBM with  the probability $\theta_i$ of clicking on item $i$ when it is observed and  the probability $\kappa_k$ of observing the position $k$, then for any index $k\geqslant 2$,
\begin{align*}
    \tilde\delta_k^* &=\frac{1}{2}\left(\theta_{{k-1}}+\theta_{k}\right)\left(\kappa_{{k-1}}+\kappa_{k}\right) - 2 \theta_{{k-1}}\theta_{k}\kappa_{{k-1}}\kappa_{k}
    && \text{if }k\leqslant K,\\
    \tilde\delta_k^* &=\frac{1}{2}\left(\theta_{K}+\theta_{k}\right)\kappa_K
    && \text{if }k\geqslant K+1,\\
    \tilde\Delta_k &\geqslant\frac{\theta_{{\min(K,k-1)}}-\theta_{k}}{\theta_{{\min(K,k-1)}}+\theta_{k}},\\
    \Delta_k &=\left(\theta_{{k-1}}-\theta_{k}\right)\left(\kappa_{{k-1}}-\kappa_{k}\right)
    && \text{if }k\leqslant K,\\
    \Delta_k &=\left(\theta_{K}-\theta_{k}\right)\kappa_K
    && \text{if }k\geqslant K+1.
\end{align*}

Hence, \ouralgo{} fulfills
\begin{align*}
R(T)
&\leqslant \sum_{k=2}^{K}\frac{8(\kappa_{k-1}-\kappa_k)(\theta_{{k-1}}+\theta_{k})^2}{\tilde\delta_k^*(\theta_{{k-1}}-\theta_{k})}\log T
+ \sum_{k=K+1}^L16\frac{\theta_{K}+\theta_{k}}{\theta_{K}-\theta_{k}}\log T
+ \OO\left(\log\log T\right)
\\&
= \OO\left(\frac{L}{\Delta}\log T\right),
\end{align*}
where
$\Delta \defeq \min \{
\min_{k \in\{2,\dots,K\}} \frac{\tilde\delta_k^*(\theta_{k-1}-\theta_k)}{(\kappa_{k-1}-\kappa_k)(\theta_{k-1}+\theta_k)^2},$
$
\min_{k \in\{K+1,\dots,L\}} \frac{\theta_K-\theta_k}{\theta_K+\theta_k}
\}$.
\end{corollary}


\begin{proof}[Proof of Corollary \ref{theo:unirank_PBM_detailed}]
Values $\tilde\delta_k^*$ and $\Delta_k$ derive from a straightforward computation given PBM model.

Let us prove the lower-bound on $\tilde\Delta_k$. Let $i$ and $j$ be two items such that $i\neq j$. Let $\av$ be a recommendation such that $\PP(c_i(t)\neq c_j(t)\mid \av(t)=\av)> 0$.

If both $i$ and $j$ appear in $\av$, denote $k<\ell$ these positions. Then
\begin{align*}
\tilde\Delta_{i,j}(\av)
    &=\frac{\frac{1}{2}(\kappa_k+\kappa_\ell)(\theta_i-\theta_j)}{\frac{1}{2}(\kappa_k+\kappa_\ell)(\theta_i+\theta_j)-2\kappa_k\kappa_\ell\theta_i\theta_j}
    \geqslant\frac{\theta_i-\theta_j}{\theta_i+\theta_j}.
\end{align*}

If only one of both items $i$ and $j$ appears in $\av$ then $\tilde\Delta_{i,j}(\av)=\frac{\theta_i-\theta_j}{\theta_i+\theta_j}$.

Hence for any index $k\geqslant 2$,
$\tilde\Delta_k \geqslant\frac{\theta_{{\min(K,k-1)}}-\theta_{k}}{\theta_{{\min(K,k-1)}}+\theta_{k}}.$
\end{proof}

\forlater{
A été en partie intégré dans la section principale de l'article. 
Discussion plus honnete
\\Mais pas la place. Et uniquement sur la défensive...
\\$\longrightarrow$ Pour plus tard (à intégrer avec "on élimine" $\log L$ et on retrouve $KL$ comme les autres dans le pire cas, et comparaison avec borne inf.)
\subsection{Comparison to Known Lower-Bounds}\label{app:lower-bound}

\cite{TopRank} prove an $\Omega(\sqrt{KLT})$ lower-bound for an online learning to rank setting similar to ours. However this proof builds upon an instance for which the attraction probability of items may take only two values, meaning that the order on items is not total. Our proof builds upon the total order on items to get a $\OO(L\log T)$ upper-bound on the regret. 

Furthermore, when looking only at PBM or CM setting, $\theta(L/\Delta\log T)$ lower and upper-bounds are usual. Typically, with CM click model, \cite{Kveton2015a} proved an  $\Omega((L-K)/\Delta\log T)$ lower-bound and proposed an algorithm with a $\OO((L-K)/\Delta\log T)$ regret. Similarly, \citep{Katariya2016} (respectively \cite{Lagree2016}) did the same for the \emph{dependent click model} (resp. the PBM click-model, assuming $\kappa_1,\dots,\kappa_K$ known).

\textbf{Discussion à intégrer}

From a theoretical point of view, UniRank's main limitation is the factor of the $\log\log T$ terms. With current version of the proof it scales in the size of the handled graph, which is of the order of the number of partitions of $[L]$. However, we believe it to be an artifact of the proof which inherits from OSUB's proof: we upper-bound the time spent on each sub-optimal node of the graph independently, while the algorithm builds upon the statistics $\hat{s}_{i,j}(t)$ which are shared between nodes. We believe a more refined proof would recover a factor polynomial in $L$ and $K$.

Another limitation of current analysis is the requirement for the order on items to be strict to obtain the factor $L$ in front of $\log T$. Without this assumption, we would recover a factor $LK$. Note that \citep{Sentenac2021}'s theoretical analysis faces the same kind of limitation, and note also that the lower-bound in \citep{TopRank} is based on an example which contradicts this assumption. In our point of view this remark is exhibiting more an interesting research perspective than a limitation: how the assumption "\emph{two items of the optimal set of items cannot be equivalent}" impact the difficulty of a bandit setting.

Finally, from a practical point of view, UniRank has a higher computational complexity than TopRank as it updates its partition (and consider merging some subsets) at each iteration, while TopRank updates its partition only when an edge is added to the leaned binary relation between items.
}

\forlater{
\section{Comparison UniRank}

  \begin{figure*}[!b]%
    \centering%
    \begin{subfigure}[b]{0.49\linewidth}
        \centering
    \includegraphics[width=\linewidth]{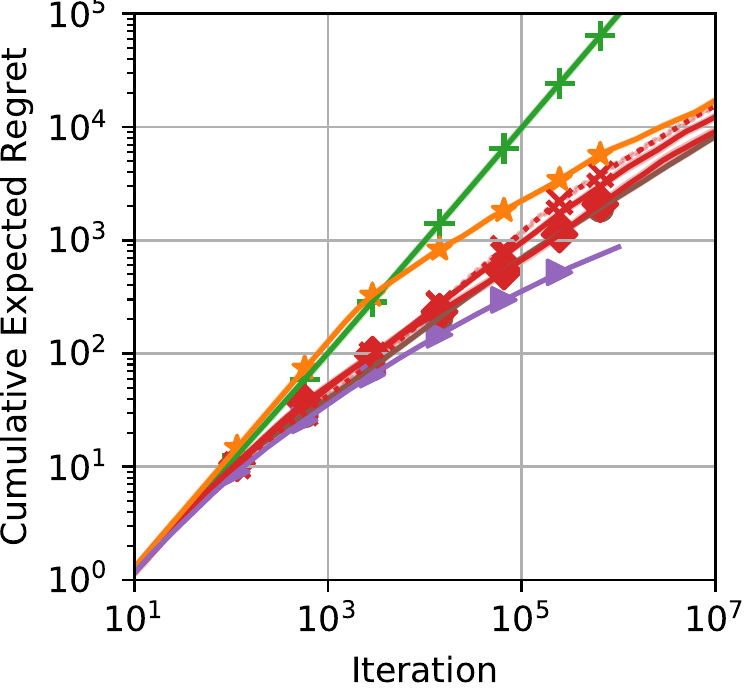}
        \caption{Yandex PBM}
        \label{fig:compareUni_Yandex_PBM}
    \end{subfigure}%
    \hfill%
    \begin{subfigure}[b]{0.49\linewidth}
        \centering
    \includegraphics[width=\linewidth]{Graph/CM_K_5_opponent_compareUni.pdf}
        \caption{Yandex CM}
        \label{fig:compareUni_Yandex_CM}
    \end{subfigure}%
    \hfill%
    
    \begin{subfigure}[b]{0.49\linewidth}
        \centering
    \includegraphics[width=\linewidth]{Graph/KDD_PBM_opponent_compareUni.pdf}
        \caption{KDD }
        \label{fig:compareUni_KDD}
    \end{subfigure}
    \hfill%
    \begin{subfigure}[b]{0.49\linewidth}
        \centering
    \includegraphics[width=\linewidth]{Graph/Neurips_legend_compareUni.pdf}
        \vspace{1cm}
    \end{subfigure}%
    
    \caption{Cumulative expected regret at T=$10^7$  on Yandex and KDD. The plotted surfaces correspond to the cumulative expected regret at iteration T=$10^7$ averaged over 20 independent sequences of recommendations per query (in total: 160 sequences for KDD and 200 sequences for each Yandex).Comparison of the different flavour of UniRank. \label{fig:compare_Uni}}
    \end{figure*}    
     
}


\forlater{La preuve triche: je ne tiens pas compte de l'initialisation, ie. de la période où $t_{i,j}(t)=0$ et du coup $\bar{\bar s}_{j(\Pm),i(\Pm)}(t)=-1$ en étant "forcé"}

\forlater{
\begin{algorithm}[H]
\caption{\ouralgo{}: \ouralgolong{}}
\begin{algorithmic}[1]
\REQUIRE number of items $L$, number of positions $K$
\FOR{$t =  1, 2, \dots$}
    \STATE
    \STATE\label{line:leadbeg}\COMMENT{leader-partition elicitation}
    \STATE $\tilde{d} \gets 0$; $R \gets [L]$
    \REPEAT
        \STATE $\displaystyle B \gets \left\{j \in R: \forall i \in R, {\hat s}_{i,j}(t) - \sqrt{\frac{\log \log t}{T_{i,j}(t)}} < 0 \right\}$
        \STATE \textbf{if} {$B \neq \varnothing$} \textbf{then} $\tilde{d} \gets \tilde{d}+1$ ; $\displaystyle \tilde{P}_{\tilde{d}}(t) \gets B$
        \textbf{end if}
    \UNTIL{$R = \varnothing$ \OR{} $B=\varnothing$ \OR{} $\left|\bigcup_{\tilde{c}=1}^{\tilde{d}}\tilde{P}_{\tilde{c}}(t)\right| \geqslant K$}
    \STATE \textbf{if} {$R \neq \varnothing$} \textbf{then} $\tilde{d} \gets \tilde{d}+1$ ; $\displaystyle \tilde{P}_{\tilde{d}}(t) \gets R$
    \textbf{end if}
    \label{line:leadend}
    \STATE
    \STATE\label{line:optbeg}\COMMENT{optimistic partition elicitation}
    \STATE $\tilde{c} \gets 1$;\qquad\qquad $d \gets 0$
    \WHILE{$\tilde{c} \leq \tilde{d}-2$}
        \STATE $d \gets d+1$
        \IF{$\min_{(i,j)\in \tilde{P}_{\tilde{c}}(t)\times \tilde{P}_{\tilde{c}+1}(t)} \bar{\bar s}_{j,i}(t) < 0$}
            \STATE\COMMENT{merge both subsets}
            \STATE  $\displaystyle P_d(t) \gets \tilde{P}_{\tilde{c}}(t) \cup \tilde{P}_{\tilde{c}+1}(t)$;\qquad\qquad$\tilde{c} \gets \tilde{c}+2$
        \ELSE
            \STATE\COMMENT{keep current subset untouched}
            \STATE  $\displaystyle P_d(t) \gets \tilde{P}_{\tilde{c}}(t)$;\qquad\qquad$\tilde{c} \gets \tilde{c}+1$
        \ENDIF
    \ENDWHILE
    \IF{$\tilde{c} = \tilde{d}-1$}
        \STATE $d \gets d+1$;\qquad\qquad $\displaystyle P_d(t) \gets \tilde{P}_{\tilde{d}-1}(t)$
        \IF{$\tilde{P}_{\tilde{d}}(t)\neq \varnothing$ \AND $\displaystyle\min_{(i,j)\in \tilde{P}_{\tilde{d} - 1}(t)\times \tilde{P}_{\tilde{d}}(t)} \bar{\bar s}_{j,i}(t) < 0$}
        \STATE\COMMENT{add best item from remaining ones}
            \STATE $\displaystyle P_d(t) \gets P_d(t) \cup \left\{\argmin_{j\in \tilde{P}_{\tilde{d}}(t)}\min_{i\in \tilde{P}_{\tilde{d} - 1}(t)} \bar{\bar s}_{j,i}(t)\right\}$
        \ENDIF
    \ENDIF
    \STATE \label{line:optend} $d(t) \gets d$
    \STATE
    \STATE\COMMENT{recommendation}
    \STATE \label{line:reco} choose $\av(t)$ uniformly at random in $\Ac\left(\Pm(t)\right)$  
    \STATE observe the clicks vector $\cv(t)$
\ENDFOR
\end{algorithmic}
\end{algorithm}
} 


\end{document}